\documentclass[twoside]{article}

\usepackage[accepted]{aistats2017}

\usepackage[T1]{fontenc}
\usepackage[utf8x]{inputenc}
\usepackage[english]{babel}

\usepackage[numbers]{natbib}

\usepackage{amsmath}
\usepackage{amssymb}
\usepackage{amsthm}
\usepackage{algorithm}
\usepackage{algorithmic}
\usepackage{float}
\usepackage{subfigure}
\newfloat{Subroutine}{thp}{lop}
\usepackage{eqparbox} \usepackage{placeins}
\usepackage{setspace}
\usepackage{xspace}

\usepackage[colorinlistoftodos, textwidth=20mm, disable]{todonotes}

\definecolor{blued}{RGB}{70,197,221}
\newcommand{\todod}[1]{\todo[color=blued,inline]{#1}}
\definecolor{citrine}{rgb}{0.89, 0.82, 0.04}
\newcommand{\todom}[1]{\todo[color=citrine]{\tiny#1}}

\newcommand{\tododout}[1]{\todo[color=blued]{\tiny#1}}

\usepackage{hyperref}
\usepackage{url}

\usepackage{xcolor} 
\newcommand{\rcolb}[1]{\textcolor{red!90}{\textbf{#1}}}

\newcommand{\bcolb}[1]{\textcolor{blue!90}{\textbf{#1}}}

\usepackage{wrapfig}
\usepackage{array}
\usepackage{enumitem}
\usepackage{mdframed}
\usepackage{bm}\usepackage{amsfonts}

\def\:#1{\protect \ifmmode {\mathbf{#1}} \else {\textbf{#1}} \fi}
\newcommand{\nystrom}{Nystr\"{o}m\xspace}
\newcommand{\sequentialalg}{\textsc{SQUEAK}\xspace}
\newcommand{\parallelalg}{\textsc{DISQUEAK}\xspace}

\newcommand{\inkestimate}{\textsc{INK-Estimate}\xspace}

\newcommand{\updatedictop}{\textsc{Dict-Update}\xspace}
\newcommand{\mergedictop}{\textsc{Dict-Merge}\xspace}
\newcommand{\shrinkop}{\textsc{Shrink}\xspace}
\newcommand{\expandop}{\textsc{Expand}\xspace}
\newcommand{\dictpool}{\mathcal{S}}
\newcommand{\coldict}{\mathcal{I}}
\newcommand{\X}{\mathcal{X}}

\newcommand{\bphi}{\boldsymbol{\phi}}
\newcommand{\CommaBin}{\mathbin{\raisebox{0.5ex}{,}}}
\newcommand*{\MyDef}{\mathrm{\tiny def}}
\newcommand*{\eqdefU}{\ensuremath{\mathop{\overset{\MyDef}{=}}}}\newcommand*{\eqdef}{\mathop{\overset{\MyDef}{\resizebox{\widthof{\eqdefU}}{\heightof{=}}{=}}}}
\newcommand{\concmat}{{\:\Psi}}
\newcommand{\concvec}{{\bm{\psi}}}
\newcommand{\nhl}{\nu}

\newcommand{\wt}[1]{\widetilde{#1}}
\newcommand{\wh}[1]{\widehat{#1}}

\newcommand{\wb}[1]{\overline{#1}}
\newcommand{\transp}{\mathsf{T}}

\DeclareMathOperator*{\Tr}{Tr}

\DeclareMathOperator*{\Diag}{Diag}

\newcommand{\norm}[2]{\left\Vert #1 \right\Vert_{#2}}
\newcommand{\normsmall}[1]{\Vert #1 \Vert}

\newcommand{\probability}{\mathbb{P}}

\DeclareMathOperator*{\expectedvalue}{\mathbb{E}}

\newcommand{\condbar}{\;\middle|\;}

\newcommand{\indfunc}{\mathbb{I}}
\renewcommand{\complement}{\mathsf{C}}
\renewcommand{\Re}{\mathbb{R}}

\newcommand{\Real}{\mathbb{R}}

\newcommand{\rkhs}{\mathcal{H}}

\newcommand{\hilbprod}[2]{\left\langle{#1},{#2}\right\rangle_\rkhs}
\newcommand{\kerfunc}{\mathcal{K}}
\newcommand{\kermatrix}{{\:K}}
\newcommand{\featkermatrix}{{\:\Phi}}

\newcommand{\deff}[1]{d_{\text{eff}}(#1)}
\newcommand{\adeff}[1]{\wt{d}_{\text{eff}}(#1)}
\newcommand{\pdeff}{d_{\text{eff}}(\gamma)}
\newcommand{\atau}{\wt{\tau}}

\newcommand{\akermatrix}{\:{\wt{K}}}

\newcommand{\selmatrix}{{\:S}}

\newcommand{\F}{\mathcal{F}}
\newcommand{\Fp}{\mathcal{F}}
\newcommand{\dataset}{\mathcal{D}}

\newcommand{\wrt}{w.r.t. }

\newcommand{\vareps}{\varepsilon}
\renewcommand{\epsilon}{\varepsilon}
\newcommand{\bigotime}{\mathcal{O}}

\newtheorem{theorem}{Theorem}
\newtheorem{corollary}{Corollary}
\newtheorem{definition}{Definition}
\newtheorem{lemma}{Lemma}
\newtheorem{proposition}{Proposition}

\begin{document}

\twocolumn[

\aistatstitle{Distributed Adaptive Sampling for Kernel Matrix Approximation}

\aistatsauthor{ Daniele Calandriello \And Alessandro Lazaric \And Michal Valko }
\aistatsaddress{SequeL team, INRIA Lille - Nord Europe} ]

\begin{abstract}
Most kernel-based methods, such as kernel or Gaussian process regression, kernel PCA, ICA, or $k$-means clustering,
do not scale to large datasets, because constructing and storing
the kernel matrix $\kermatrix_n$ requires at least $\bigotime(n^2)$ time and space for
$n$ samples.
Recent works~\cite{alaoui2014fast,musco2016provably} show that
sampling points with replacement according to their ridge leverage scores (RLS)
generates small dictionaries of relevant points with strong spectral approximation guarantees for 
$\kermatrix_n$. The drawback of RLS-based methods is that
computing exact RLS requires constructing and storing the whole kernel matrix.
In this paper, we introduce \sequentialalg, a new algorithm for
kernel approximation based on RLS sampling that \emph{sequentially} processes the dataset, storing a dictionary which creates accurate kernel matrix approximations with a number of points that only depends on the effective dimension $\deff{\gamma}$ of the dataset.
Moreover since all the RLS estimations are efficiently performed using only the small dictionary,
\sequentialalg is the first RLS sampling algorithm that never constructs the whole matrix $\kermatrix_n$, runs in linear time $\wt{\bigotime}(n\deff{\gamma}^3)$ w.r.t.~$n$,
and requires only a single pass over the dataset. We also propose a parallel and distributed version of \sequentialalg that \emph{linearly} scales across multiple machines, achieving similar accuracy in as little as $\wt{\bigotime}(\log(n)\deff{\gamma}^3)$ time.
\end{abstract}
\vspace{-0.1in}
\vspace{-0.1in}
\section{Introduction}\label{sec:intro}
\vspace{-0.1in}
One of the major limits of kernel ridge regression (KRR), kernel PCA~\cite{scholkopf1999kernel}, and other kernel methods is that for $n$ samples storing and
manipulating the kernel matrix $\kermatrix_n$ requires $\bigotime(n^2)$ space,
which becomes rapidly infeasible for even a relatively small $n$. 
For larger sizes (or streams) we cannot even afford to store 
or process the data on as single machine.
Many solutions focus
on how to scale kernel methods by reducing its space (and time) complexity without
compromising the prediction accuracy. A popular approach is to construct
low-rank approximations of the kernel matrix by randomly selecting a subset (dictionary) of
$m$ columns from $\kermatrix_n$, thus reducing the space complexity to
$\bigotime(nm)$. These methods, often referred to as \textit{Nystr\"{o}m
approximations}, mostly differ in the distribution used to sample the columns of
$\kermatrix_n$ and the construction of low-rank approximations. Both of these
choices significantly affect the accuracy of the resulting
approximation~\citep{rudi2015less}. \citet{bach2013sharp} showed that uniform
sampling preserves the prediction accuracy of KRR (up to $\varepsilon$) only
when the number of columns $m$ is proportional to the maximum degree of freedom
of the kernel matrix. This may require sampling $\bigotime(n)$ columns in
datasets with high coherence~\citep{gittens2013revisiting},
i.e., a kernel matrix with weakly correlated
columns. On the other hand, \citet{alaoui2014fast} showed that sampling
columns according to their ridge leverage scores (RLS) (i.e., a measure of the
influence of a point on the regression) produces an accurate \nystrom
approximation with only a number of columns~$m$ proportional to the average
degrees of freedom of the matrix, called \textit{effective dimension}.
Unfortunately, the complexity of computing RLS requires storing the whole kernel matrix, thus making this approach infeasible. However, \citet{alaoui2014fast}
proposed a fast method to compute a constant-factor approximation of the RLS
and showed that accuracy and space complexity are close to the case of sampling
with exact RLS at the cost of an extra dependency on the inverse of the minimal
eigenvalue of the kernel matrix. Unfortunately, the minimal eigenvalue can be arbitrarily small in many problems. \citet{calandriello2016analysis} addressed this
issue by processing the dataset \emph{incrementally} and updating estimates of
the ridge leverage scores, effective dimension, and \nystrom approximations
on-the-fly. Although the space complexity of the resulting algorithm
(\inkestimate) does not depend on the minimal eigenvalue anymore, it introduces
a dependency on the largest eigenvalue of $\kermatrix_n$, which in the worst
case can be as big as $n$, thus losing the advantage of the
method. In this paper we introduce an algorithm for SeQUEntial Approximation
of Kernel matrices (\sequentialalg), a new algorithm that builds on \inkestimate,
but uses \emph{unnormalized} RLS.
This improvement, together with a new analysis, opens the way to 
major improvements over current leverage sampling methods (see Sect.~\ref{sec:conclusions} for a comparison with existing methods)
closely matching the dictionary size achieved by exact RLS sampling.
First, unlike \inkestimate, \sequentialalg is simpler, does not need
to compute an estimate of the effective dimension for normalization,
and exploits a simpler, more accurate RLS estimator.
This new estimator only requires access to the points stored in the dictionary.
Since the size of the dictionary is much smaller than the $n$, \sequentialalg
needs to actually observe only a fraction of the kernel matrix $\kermatrix_n$,
resulting in a runtime linear in $n$.
Second, since our dictionary updates require only access to local data,
our algorithm allows for distributed processing where machines operating on
different dictionaries do not need to communicate with each other.
In particular, intermediate dictionaries can be extracted in parallel from small portions of the dataset and they can be later merged in a hierarchical way. Third, the sequential nature of \sequentialalg requires a more sophisticated
analysis that take into consideration the complex interactions and dependencies between successive resampling steps.
The analysis of \sequentialalg builds on a new martingale argument that could be of independent interest for similar online resampling schemes.
Moreover, our \sequentialalg can naturally incorporate new data 
without the need of recomputing the whole resparsification from scratch and therefore it can be applied in streaming settings.
We note there exist other ways 
to avoid the intricate dependencies with simpler analysis, for example by resampling 
\cite{musco2016provably}, but with negative algorithmic side effects: these methods 
need to pass through the dataset multiple times. \sequentialalg passes 
\emph{through the dataset only once}\footnote{Note that there is an important difference in whether the method passes through \emph{kernel matrix} only once or through the \emph{dataset} only once, in the former, the algorithm may still need access one data point up to $n$ times, thus making it unsuitable for the streaming setting and less practical for distributed computation.}
and is therefore the first provably accurate kernel approximation algorithm that can handle both \emph{streaming and distributed} settings.
 
\vspace{-0.1in}
\section{Background}\label{sec:setting}
\vspace{-0.1in}

In this section, we introduce the notation and basics of kernel approximation used through the paper.

\textbf{Notation.}
We use curly capital letters $\mathcal{A}$ for collections. We use upper-case
bold letters $\:A$ for matrices and operators, lower-case bold letters $\:a$
for vectors,
lower-case letters $a$ for scalars, with the exception of $f,g,$ and $h$ which
denote functions, and $[n] := \{1,\ldots,n\}$ for the set of
integers between 1 and $n$. We denote by $[\:A]_{ij}$ and $[\:a]_i$, the
$(i,j)$ element of a matrix and $i$th element of a vector respectively. We
denote by $\:I_n\in\Re^{n\times n}$, the identity matrix of dimension $n$ and by
$\Diag(\:a)\in\Re^{n\times n}$ the diagonal matrix with the vector
$\:a\in\Re^n$ on the diagonal. We use $\:e_{n,i} \in \Re^{n}$ to denote the
indicator vector for element $i$ of dimension~$n$. When the dimension of
$\:I$ and $\:e_{i}$ is clear from the context, we omit the $n$. We use $\:A
\succeq \:B$ to indicate that $\:A-\:B$ is a Positive Semi-Definite (PSD) operator.

\textbf{Kernel.}
We consider a positive definite kernel function
$\kerfunc: \X \times \X \rightarrow \Re$ and we denote with
$\rkhs$ its induced Reproducing Kernel Hilbert Space (RKHS),
and with $\varphi: \X \rightarrow \rkhs$ its corresponding feature map.
Using $\varphi$, and without loss of generality,
for the rest of the paper we will replace $\rkhs$ with a high dimensional
space $\Real^D$ where $D$ is large and potentially infinite.
With this notation, the kernel evaluated between to points can be expressed as
$\kerfunc(\:x, \:x') = \hilbprod{\kerfunc(\:x,\cdot)}{\kerfunc(\:x', \cdot)} = \hilbprod{\varphi(\:x)}{\varphi(\:x')} = \varphi(\:x)^\transp \varphi(\:x')$.
Given a dataset of points $\dataset=\{\:x_t\}_{t=1}^n$, we define the (empirical) kernel matrix
$\kermatrix_{t} \in \Re^{t\times t}$ as the
application of the kernel function on all pairs of input values (i.e.,
$[\kermatrix_{t}]_{ij} = k_{i,j} = \kerfunc(\:x_i, \:x_j)$ for any $i,j\in [t]$),
with $\:k_{t,i} = \kermatrix_t \:e_{t,i}$ as its $i$-th column.
We also define the feature vectors $\bphi_i = \varphi(\:x_i) \in \Real^D$ and after
introducing the matrix 
$    \featkermatrix_{t} = \left[ \bphi_1,\bphi_2, \dots, \bphi_t \right] \in \Real^{D \times t}$                                     we can rewrite the kernel matrix as $\kermatrix_{t} = \featkermatrix_{t}^\transp
\featkermatrix_{t}$.

\textbf{Kernel approximation by column sampling.}
One of the most popular strategies to have low space complexity approximations
of the kernel $\kermatrix_t$ is to randomly select a subset of its columns
(possibly reweighted) and use them to perform the specific kernel task at hand
(e.g., kernel regression). More precisely, we define a column dictionary as a
collection $\coldict_t = \{(i,w_i)\}_{i=1}^t$, where the first term denotes the
index of the column and $w_i$ its weight, which is set to zero for all columns
that are not retained. For the theoretical analysis, we conveniently keep the
dimension of any dictionary $\coldict_t$ to $t$, while in practice, we only
store the non-zero elements. In particular, we denote by $|\coldict_t|$ be the
size of the dictionary corresponding to the elements with non-zero weights
$w_i$. Associated with a column dictionary, there is a selection matrix
$\selmatrix_t~=~\Diag(\sqrt{w_1} \ldots \sqrt{w_t})\in\Re^{t\times t}$ such
that for any matrix $\:A_t\in\Re^{t \times t}$, $\:A_t \selmatrix_t$ returns a
$t\times t$ matrix where the columns selected by $\coldict_t$ are properly
reweighted and all other columns are set to 0. Despite the wide range of kernel
applications, it is possible to show that in most of
them, the quality of a dictionary can be measured in terms of how
well it approximates the projection associated to the kernel. 
In kernel regression, for instance, we use $\:K_t$ to construct the projection
(hat) matrix that projects the observed labels $\:y_t$  to ${\bf\widehat{y}}_t$. 
In particular, let $\:P_t = \:K_t \:K_t^+$ be the projection matrix
(where $\kermatrix_t^+$ indicates the pseudoinverse),
then ${\bf\widehat{y}}_t = \:P_t\:y_t$. If $\:K_t$
is full-rank, then $\:P_t = \:I_t$ is the identity matrix and,
we can reconstruct any target vector $\:y_t$ exactly.
On the other hand, the only sampling
scheme which guarantees to properly approximate a full rank $\:P_t$ requires all columns to
be represented in $\coldict_t$. In fact, all columns have the same
``importance'' and no low-space approximation is possible. Nonetheless,
kernel matrices are often either rank deficient or have extremely small
eigenvalues (exponentially decaying spectrum),
as a direct (and desired) consequence of embedding low dimensional
points $\:x_i$ into a high dimensional RKHS.
In this case, after soft thresholding the smaller eigenvalues to a given value $\gamma$,
$\:K_t$ can be effectively approximated using a small subset
of columns. This
is equivalent to approximating the $\gamma$-ridge projection matrix
\begin{align*}
\:P_t \eqdef (\kermatrix_t + \gamma\:I)^{-1/2}\kermatrix_t(\kermatrix_t + \gamma\:I)^{-1/2}.
\end{align*}
We say that a column dictionary is accurate if the following condition is satisfied.
\begin{definition}\label{def:eps-acc-dict}
A dictionary $\coldict_t = \{(i,w_i)\}_{i=1}^t$ and its associated selection matrix $\selmatrix_t \in \Real^{t \times t}$ are $\varepsilon$-accurate w.r.t.\@ a kernel matrix $\kermatrix_t = \kermatrix_t^{1/2}\kermatrix_t^{1/2}$
if \footnote{the matrix norm we use is the operator (induced) norm}
\begin{align}\label{cond:concentr-condition}
\| \:P_t - \wt{\:P}_t \| \leq \varepsilon,
\end{align}
where for a given $\gamma > 0$, the approximated projection matrix is defined as
\begin{align*}
\wt{\:P}_t \eqdef (\kermatrix_t + \gamma\:I_t)^{-\frac 1 2}\kermatrix_t^{1/2}\selmatrix_t \selmatrix_t^\transp\kermatrix_t^{1/2}(\kermatrix_t + \gamma\:I_t)^{-\frac 1 2}.
\end{align*}
\end{definition}

Notice that this definition of accuracy is purely theoretical, since $\wt{\:P}_t$ is never computed. Nonetheless, as illustrated in Sect.\@ \ref{sec:generalization}, $\varepsilon$-accurate dictionaries can be used to construct suitable kernel approximation in a wide range of problems.

\textbf{Ridge leverage scores sampling.} \citet{alaoui2014fast} showed that an $\varepsilon$-accurate dictionary can be obtained by sampling columns proportionally to their $\gamma$-ridge leverage scores (RLS) defined as follows.

\begin{definition}\label{def:exact-lev-scores}
Given a kernel matrix $\:K_t\in\Re^{t\times t}$, the $\gamma$-ridge leverage score (RLS) of column $i\in [t]$ is
\begin{align}\label{eq:exact-rls}
    \tau_{t,i} = \:e_{t,i}^\transp\kermatrix_t(\kermatrix_t + \gamma\:I_t)^{-1} \:e_{t,i},
\end{align}
Furthermore, the effective dimension $\deff{\gamma}_t$ of the kernel matrix $\kermatrix_t$ is defined as
\begin{align}\label{eq:exact-deff}
\deff{\gamma}_t = \sum_{i=1}^t \tau_{t,i}(\gamma) = \Tr\left(\kermatrix_t(\kermatrix_t + \gamma \:I_t)^{-1}\right).
\end{align}
\end{definition}

The RLS can be interpreted and derived in many ways, and they are well studied
\cite{cohen2017input, cohen2016online, woodruff2014sketching} in the linear setting (e.g. $\bphi_t = \:x_t$).
\citet{patel_oasis:_2015} used them
as a measure of incoherence to select important points,
but their deterministic algorithm provides guarantees only
when $\kermatrix_t$ is exactly low-rank.
Here we notice that
\begin{align*}
&\tau_{t,i} =\:e_{t,i}^\transp\:K_t(\kermatrix_t + \gamma\:I_t)^{-1} \:e_{t,i}
= \:e_{t,i}^\transp \:P_t \:e_{t,i},
\end{align*}
which means that they correspond to the diagonal elements of the $\:P_t$ itself. Intuitively, this correspond to selecting each column $i$ with probability $p_{t,i} = \tau_{t,i}$ will capture the most important columns to define $\:P_t$, thus minimizing the approximation error $\| \:P_t - \wt{\:P}_t \|$.
More formally, \citet{alaoui2014fast} state the following.
\begin{proposition}\label{prop:gamma.approx}
Let $\varepsilon\in[0,1]$ and $\coldict_n$ be the dictionary built with $m$ columns randomly selected proportionally to RLSs $\{\tau_{n,i}\}$ with weight $w_i = 1/(m \tau_{n,i})$. If $m = \bigotime(\frac{1}{\varepsilon^{2}}\pdeff_n\log(\frac{n}{\delta}))$, then w.p.\@ at least $1-\delta$, the corresponding dictionary is $\varepsilon$-accurate.
\end{proposition}

Unfortunately, computing exact RLS requires storing $\kermatrix_n$ and this is seldom possible in practice. In the next section, we introduce \sequentialalg, an RLS-based incremental algorithm able to preserve the same accuracy of Prop.~\ref{prop:gamma.approx} \emph{without} requiring to know the RLS in advance. We prove that it generates a dictionary only a constant factor larger than exact RLS sampling.

 \vfil
\vspace{-0.1in}
\section{Sequential RLS Sampling}\label{sec:sequential-alg}
\vspace{-0.1in}

\begin{algorithm}[t!]
    \begin{algorithmic}[1]
        \renewcommand\algorithmiccomment[1]{        \hfill\(\triangleright\){#1}}
        \renewcommand{\algorithmicrequire}{\textbf{Input:}}
        \renewcommand{\algorithmicensure}{\textbf{Output:}}
        \REQUIRE Dataset {$\dataset$}, parameters $\gamma, \varepsilon, \delta$
        \ENSURE $\coldict_n$
        \STATE Initialize $\mathcal{I}_0$ as empty, $\wb{q}$ (see Thm.~\ref{thm:sequential-alg-main})
        \FOR{$t = 1,\dots,n$}
            \STATE Read point $\:x_t$ from $\dataset$
            \STATE $\wb{\coldict} = \coldict_{t-1} \cup \{(t,\wt{p}_{t-1,t}\!=\!1,q_{t-1,t}\!=\!\wb{q})\}$ \COMMENT{\expandop}
            \STATE $\coldict_{t} = \updatedictop(\wb{\coldict})$ using Eq.~\ref{eq:rls-estimator}
        \ENDFOR
    \end{algorithmic}
       \caption{The \sequentialalg algorithm}
    \label{alg:sequentialalg}
\end{algorithm}
\begin{Subroutine}[t!]
       \begin{algorithmic}[1]
        \renewcommand\algorithmiccomment[1]{        \hfill \(\triangleright\){#1}}
        \renewcommand{\algorithmicrequire}{\textbf{Input:}}
        \renewcommand{\algorithmicensure}{\textbf{Output:}}
        \REQUIRE $\wb{\coldict}$
        \ENSURE $\coldict_{t}$
        \STATE Initialize $\coldict_{t} = \emptyset$
        \FORALL[\shrinkop]{$i \in \{1,\dots,t\}$}
        \IF{$q_{t-1,i} \neq 0$}
        \STATE Compute $\atau_{t,i}$ using $\wb{\coldict}$
        \STATE Set $\wt{p}_{t,i} = \min\{\atau_{t,i}, \wt{p}_{t-1,i}\}$
        \STATE Set $q_{t,i} \sim \mathcal{B}(\wt{p}_{t,i}/\wt{p}_{t-1,i}, q_{t-1,i})$
        \ELSE
        \STATE $\wt{p}_{t,i} = \wt{p}_{t-1,i}$ and $q_{t,i} = q_{t-1,i}$
        \ENDIF
        \ENDFOR
    \end{algorithmic}
    \caption{The \updatedictop algorithm}
    \label{alg:updatedictop}
\end{Subroutine}

In the previous section, we showed that sampling proportionally to the RLS
$\{\tau_{t,i}\}$ leads to a dictionary such that $\| \:P_t - \wt{\:P}_t \| \leq
\varepsilon$. Furthermore, since the RLS correspond to the diagonal entries of
$\:P_t$, an accurate approximation $\wt{\:P}_t$ may be used in turn to compute
accurate estimates of $\tau_{t,i}$. The  \sequentialalg algorithm
(Alg.~\ref{alg:sequentialalg}) builds on this intuition to sequentially process
the kernel matrix $\:K_n$ so that exact RLS computed on a small matrix ($\:K_t$
with $t\ll n$) are used to create an $\varepsilon$-accurate dictionary, which
is then used to estimate the RLS for bigger kernels, which are in turn used to
update the dictionary and so on. 
While \sequentialalg shares a similar structure with
\inkestimate~\citep{calandriello2016analysis}, the sampling probabilities are
computed from different estimates of the RLS $\tau_{t,i}$ and no renormalization by
an estimate of $\pdeff_t$ is needed. Before giving the details of the
algorithm, we redefine a dictionary as a collection $\coldict =
\{(i,\wt{p}_i, q_i)\}_i$, where $i$ is the index of the point $\:x_i$ stored
in the dictionary, $\wt{p}_i$ tracks the probability used to sample it, and $q_i$ is the
number of copies (multiplicity) of $i$. The weights are then computed as
$w_i~=~q_i/(\wb{q}\wt{p}_i)$, where $\wb{q}$ is an algorithmic parameter discussed
later. We use $\wt{p}_i$ to stress the fact that these probabilities will be
computed as approximations of the actual probabilities that should be used to
sample each point, i.e., their RLS $\tau_i$.

\sequentialalg receives as input a dataset $\dataset=\{\:x_t\}_{t=1}^n$ and processes it \emph{sequentially}. Starting with an empty dictionary $\coldict_0$, at each time step $t$, \sequentialalg receives a new point $\:x_t$.
Adding a new point $\:x_t$ to the kernel matrix can either decrease
the importance of points observed before (i.e., if they are correlated with the new point) or leave it unchanged (i.e., if their corresponding kernel columns are orthogonal) and thus for any $i\leq t$, the RLS evolves as follows.

\begin{lemma}\label{lem:monotone-decrease-prob}
    For any kernel matrix $\kermatrix_{t-1}$ at time $t-1$ and its extension
    $\kermatrix_{t}$ at time $t$, we have that the RLS
    are monotonically decreasing and the effective dimension is monotonically increasing,
\begin{align*}
\frac{1}{\tau_{t-1,i} + 1}\tau_{t-1,i} \leq \tau_{t,i} \leq \tau_{t-1,i},&&
\deff{\gamma}_{t} \geq \deff{\gamma}_{t-1}.
\end{align*}
\end{lemma}

The previous lemma also shows that the RLS cannot decrease too quickly  and since $\tau_{t-1,i}\leq 1$, they can at most halve when $\tau_{t-1,i}=1$.
After receiving the new point $\:x_t$, we need to update our dictionary $\coldict_{t-1}$
to reflect the changes of the $\tau_{t,i}$. We proceed in two phases.
During the \expandop phase, we directly add the new element
$\:x_{t}$ to $\coldict_{t-1}$ and obtain a
temporary dictionary $\wb{\coldict}$, where the new element $t$ is added with a
sampling probability $\wt{p}_{t-1,t}=1$ and a number of
copies $q_{t-1,t} = \wb{q}$, i.e., $\wb{\coldict} = \coldict_{t-1} \cup
\{(t,\wt{p}_{t-1,t} = 1,q_{t-1,t} = \wb{q})\}$. This increases our memory
usage, forcing us to update the dictionary using \updatedictop,
in order to decrease its size.
Given as input $\wb{\coldict}$, we use the following estimator to compute
the approximate RLS $\atau_{t,i}$,
    \begin{align}\label{eq:rls-estimator}
        \atau_{t,i} &= (1-\varepsilon)\bphi_i^\transp(\featkermatrix_t\wb{\:S}\wb{\:S}^\transp\featkermatrix_t^\transp + \gamma\:I)^{-1}\bphi_i\nonumber\\
        & =\tfrac{1-\varepsilon}{\gamma}(k_{i,i} - \:k_{t,i}^\transp\wb{\selmatrix}(\wb{\selmatrix}^\transp\kermatrix_t\wb{\selmatrix} + \gamma\:I_t)^{-1}\wb{\selmatrix}^\transp\:k_{t,i}),
    \end{align}
where $\varepsilon$ is the accuracy parameter, $\gamma$ is the regularization
and $\wb{\selmatrix}$ is the selection matrix associated to
$\wb{\coldict}$. This estimator follows naturally from a reformulation of the RLS.
In particular, if we consider $\bphi_i$, the RKHS representation of $\:x_i$,
the RLS $\tau_{t,i}$ can be formulated as $\tau_{t,i} =
\bphi_i^\transp(\featkermatrix_t\:I_t\featkermatrix_t^\transp +
\gamma\:I)^{-1}\bphi_i$, where we see that the importance of point $\:x_i$ is
quantified by how orthogonal (in the RKHS) it is w.r.t.\@ the other points.
Because we do not have access to all the columns
($\wb{\selmatrix}\wb{\selmatrix}^\transp \neq \:I_t$),
similarly to what \cite{cohen2016online} did for the special case $\bphi_i = \:x_i$,
we choose to use
$\atau_{t,i} \approx
\bphi^\transp(\featkermatrix_t\wb{\selmatrix}\wb{\selmatrix}^\transp\featkermatrix_t^\transp +
\gamma\:I)^{-1}\bphi_i$, and then we use the kernel trick
to derive a form that we can actually compute, resulting in Eq.~\ref{eq:rls-estimator}.
The approximate RLSs are then used to define the new sampling probabilities as
$\wt{p}_{t,i} = \min\{\atau_{t,i}, \wt{p}_{t-1,i}\}$.
For each element in~$\wb{\coldict}$, the \shrinkop step draws a sample from
the binomial $\mathcal{B}(\wt{p}_{t,i}/\wt{p}_{t-1,i}, q_{t-1,i})$, where the
minimum taken in the definition of $\wt{p}_{t,i}$ ensures that the binomial probability is well
defined (i.e., $\wt{p}_{t,i} \leq \wt{p}_{t-1,i}$). This resampling step
basically \textit{tracks} the changes in the RLS and constructs a new dictionary
$\coldict_{t}$, which is \textit{as if} it was created from scratch using all
the RLS up to time~$t$ (with high probability).
We 
see that the new element~$\:x_t$ is only added to
the dictionary with a large number of copies (from 0 to $\wb{q}$) if
its estimated relevance $\wt{p}_{t,t}$ is high,
and that over time elements originally in $\coldict_{t-1}$ are stochastically reduced to
reflect the reductions of the RLSs. The lower $\wt{p}_{t,i}$ w.r.t.\
$\wt{p}_{t-1,i}$, the lower the number of copies
$q_{t,i}$ w.r.t.\ $q_{t-1,i}$. If the probability $\wt{p}_{t,i}$
continues to decrease over time, then~$q_{t,i}$ may become zero, and the column
$i$ is completely dropped from the dictionary (by setting its weight to zero).
The approximate RLSs enjoy the following guarantees.

\begin{lemma}\label{lem:fast-rls}
    Given an $\varepsilon$-approximate dictionary $\coldict_{t-1}$
    of matrix $\kermatrix_{t-1}$, construct $\wb{\coldict}$ by adding element
    $(t, 1, \wb{q})$ to it, and compute the selection matrix $\wb{\selmatrix}$.
    Then for all $i$ in $\wb{\coldict}$ such that $q_{t-1,i} \neq 0$, the estimator in Eq.~\ref{eq:rls-estimator} is $\alpha$-accurate, i.e., it satisfies
                    $\tau_{t,i}/\alpha \leq \atau_{t,i} \leq \tau_{t,i}$,
        with $\alpha~=~(1+\varepsilon)/(1-\varepsilon)$. Moreover, given RLS $\tau_{t-1,i}$ and $\tau_{t,i}$, and two $\alpha$-accurate RLSs, $\atau_{t-1,i}$ and $\atau_{t,i}$,
    the quantity
    $\min\left\{\atau_{t,i},\; \atau_{t-1,i} \right\}$ is also an $\alpha$-accurate RLS.
\end{lemma}

This result is based on the property that whenever $\coldict_{t-1}$ is
$\varepsilon$-accurate for $\kermatrix_{t-1}$, the projection matrix $\:P_t$
can be approximated by $\wb{\:P}_t$ constructed using the temporary dictionary
$\wb{\coldict}$ and thus, the RLSs can be accurately estimated and used to
update $\coldict_{t-1}$ and obtain a new $\varepsilon$-accurate dictionary for
$\kermatrix_{t}$. Since $\wt{\tau}_{t,i}$ is used to sample the new dictionary
$\coldict_t$, we need each point to be sampled \textit{almost} as frequently as
with the true RLS $\tau_{t,i}$, which is guaranteed by the lower bound of
Lem.~\ref{lem:fast-rls}. Since RLSs are always smaller or equal than~1, this
could be trivially achieved by setting $\atau_{t,i}$ to~1. Nonetheless, this
would keep all columns in the dictionary. Consequently, we need to force the RLS estimate  to decrease as much as possible, so
that low probabilities allow reducing the space as much as possible.
This is obtained by the upper bound in Lem.~\ref{lem:fast-rls}, which
guarantees that the estimated RLS are always smaller than the exact RLS. As a
result, \shrinkop sequentially preserves the overall accuracy of the dictionary
and \textit{at the same time} keeps its size as small as possible, as shown
in the following theorem.

\begin{theorem}\label{thm:sequential-alg-main} Let $\varepsilon>0$ be the
accuracy parameter, $\gamma > 1$ the regularization, and $0 < \delta <1$ the probability of failure. Given an arbitrary dataset $\dataset$ in
input together with parameters $\varepsilon$, $\gamma$, and $\delta$, we run
\sequentialalg with
\begin{align*}
\wb{q} =  \frac{39\alpha\log\left(2n/\delta\right)}{\varepsilon^{2}}\CommaBin
\end{align*}
where  $\alpha = (1+\varepsilon)/(1-\varepsilon)$. Then, w.p.\@ at least $1-\delta$,
 \sequentialalg generates a sequence of random dictionaries $\{\coldict_t\}_{t=1}^n$ that are $\varepsilon$-accurate (Eq.~\ref{cond:concentr-condition}) w.r.t.\@ any of the intermediate kernels~$\kermatrix_t$, and the size of the dictionaries is bounded as $\max\limits_{t=1,\ldots,n} |\coldict_t| \leq 3\wb{q}\deff{\gamma}_n.$

As a consequence, on a successful run the overall complexity of \sequentialalg is bounded as
\begin{align*}
 \text{space complexity} &=  \Big(\max_{t=1,\ldots,n} |\coldict_t|\Big)^2 \leq \left(3\wb{q}\deff{\gamma}_n\right)^2,\\
 \text{time complexity} &=\bigotime\left(n \deff{\gamma}_n^3\wb{q}^3\right).
\end{align*}
\end{theorem}

We show later that Thm.~\ref{thm:sequential-alg-main} is special case of Thm.~\ref{thm:parallel-alg-main}
and give a sketch of the proof with the statement of Thm.~\ref{thm:parallel-alg-main}.
We postpone the discussion about this result and the comparison with previous
results to Sect.~\ref{sec:conclusions} and focus now on the space
and time complexity. Note that while the dictionaries $\coldict_t$ always
contain $t$ elements for notational convenience, \shrinkop actually \emph{never} 
updates the probabilities of the elements with $q_{t-1,i} = 0$. This feature is
particularly important, since at any step~$t$, it only requires to compute
approximate RLSs for the elements which are actually included in
$\coldict_{t-1}$ and the new point $\:x_t$ (i.e., the elements in
$\wb{\coldict}$) and thus it does not require recomputing the RLSs of points
$\:x_s$ ($s < t$) that have been dropped before!
This is why \sequentialalg  computes an $\varepsilon$-accurate dictionary
with a\emph{ single pass over the dataset}. Furthermore, the estimator in
Eq.~\ref{eq:rls-estimator} does not require computing the whole kernel column
$\:k_{t,i}$ of dimension $t$. In fact, the components of $\:k_{t,i}$,
corresponding to points which are no longer in $\wb{\coldict}$, are directly set
to zero when computing $\:k_{t,i}^\transp \wb{\selmatrix}$. As a result, for any new
point $\:x_t$ we need to evaluate $\kerfunc(\:x_s,\:x_t)$ only for the indices
$s$ in $\wb{\coldict}$.
Therefore, \sequentialalg never performs more than $n(3\wb{q}\deff{\gamma}_n)^2$
kernel evaluations, which means that it does not even need to observe
large portions of the kernel matrix.
Finally, the runtime is dominated by the $n$ matrix inversions used in Eq.~\ref{eq:rls-estimator}.
Therefore, the total runtime is $\bigotime(n\big(\max_{t=1,\ldots,n} |\coldict_t|\big)^3)
= \bigotime(n \deff{\gamma}_n^3\wb{q}^3)$. In the next section, we introduce
\parallelalg, which improves the runtime by independently constructing separate dictionaries
in parallel and then merging them recursively to construct a final $\varepsilon$-accurate
dictionary.

\vspace{-0.1in}
\section{Distributed RLS Sampling}\label{sec:parallel-alg}
\vspace{-0.1in}

\begin{algorithm}[t!]
        \begin{algorithmic}[1]
        \renewcommand\algorithmiccomment[1]{        \hfill \(\triangleright\)\eqparbox{COMMENT}{#1}}
        \renewcommand{\algorithmicrequire}{\textbf{Input:}}
        \renewcommand{\algorithmicensure}{\textbf{Output:}}
        \REQUIRE Dataset $\dataset$, parameters $\gamma, \varepsilon, \delta$
        \ENSURE $\coldict_{\dataset}$
        \STATE Partition $\dataset$ into disjoint sub-datasets $\dataset_i$
        \STATE Initialize $\coldict_{\dataset_i} = \{(j,\wt{p}_{0,i} = 1, q_{0,i} = \wb{q}) : j \in \dataset_i\}$
        \STATE Build set $\dictpool_{1} = \{\coldict_{\dataset_i}\}_{i=1}^k$
        \FOR{$h = 1,\dots,k-1$}
            \IF[\mergedictop]{$|\dictpool_h| > 1$}
                \STATE Pick two dictionaries $\coldict_{\dataset}, \coldict_{\dataset'}$ from $\dictpool_h$
                \STATE $\wb{\coldict} = \coldict_{\dataset} \cup \coldict_{\dataset'}$
                \STATE $\coldict_{\dataset,\dataset'} = \updatedictop(\wb{\coldict})$ using Eq.~\ref{eq:rls-estimator-merge}
                \STATE Place $\coldict_{\dataset,\dataset'}$ back into $\dictpool_{h+1}$
            \ELSE
                \STATE $\dictpool_{h+1} = \dictpool_h$
            \ENDIF
        \ENDFOR
        \STATE Return $\coldict_{\dataset}$, the last dictionary in $\dictpool_{k}$
    \end{algorithmic}
    \caption{The distributed \sequentialalg algorithm}
    \label{alg:distributedalg}
\end{algorithm}

In this section, we show that a minor change in the structure of \sequentialalg
allows us to parallelize and distribute the computation of the dictionary
$\coldict_n$ over multiple machines, thus reducing even further its time complexity. Beside the computational advantage, a distributed architecture is needed as soon as the input dimension $d$ and the number of points $n$ is so large that having the dataset on a single machine is impossible.
Furthermore, distributed processing can reduce contention \todom{i guess this is some distributed lingo?} on bottleneck data sources
such as databases or network connections.
DIstributed-\sequentialalg (\parallelalg, Alg.~\ref{alg:distributedalg}) partitions $\dataset$ over multiple machines and
the (small) dictionaries that are generated from different portions of the
dataset are integrated in a hierarchical way. The initial dataset is
partitioned over $k$ disjoint sub-datasets $\dataset_i$ with $i=1,\ldots,k$ and $k$
dictionaries $\coldict_{\dataset_i} = \{(j,\wt{p}_{0,i} = 1, q_{0,i} = \wb{q}) : j \in \dataset_i\}$
are initialized simply by placing all samples in $\dataset_i$ into $\coldict$
with weight 1 and multiplicity $\wb{q}$. Alternatively, if the datasets
$\dataset_i$ are too large to fit in memory, we can run \sequentialalg
to generate the initial dictionaries.
The dictionaries $\coldict_{\dataset_i}$ are added to a dictionary collection
$\dictpool_1$, and at each step $h \in [k]$ Alg.~\ref{alg:distributedalg}
arbitrarily chooses two dictionaries $\coldict_{\dataset}$ and $\coldict_{\dataset'}$
from $\dictpool_{h}$, and merges them.
\mergedictop first combines them into a single dictionary $\wb{\coldict}$ (the equivalent of
the \expandop phase in \sequentialalg) and then \updatedictop is run on the merged dictionaries
 to create an updated dictionary $\coldict_{\dataset\cup \dataset'}$, which
 is placed back in the dictionary collection $\dictpool$.
This sequence of merges can be represented using a binary merge tree,
as in Fig.~\ref{fig:arbitrary.merge-trees}.
Since \mergedictop only takes the two dictionaries as input and does not
require any information on the dictionaries in the rest of the tree,
separate branches can be run simultaneously on different machines, and only
the resulting (small) dictionary needs to be propagated to the parent node
for the future \mergedictop.
Unlike in \sequentialalg, \updatedictop is run on the union of two distinct
dictionaries rather than one dictionary and a new single point. As a result, we
need to derive the ``distributed'' counterparts of
Lemmas~\ref{lem:monotone-decrease-prob} and~\ref{lem:fast-rls} to analyze the
behavior of the RLSs and the quality of the estimator.
\begin{lemma}\label{lem:monotone-decrease-prob-merge}
    Given two disjoint datasets $\dataset, \dataset'$,
    for every $i \in \dataset \cup \dataset'$,
    $\tau_{i,\dataset} \geq \tau_{i,\dataset \cup \dataset'}$
    and
    \begin{align*}
        2\deff{ \gamma }_{\dataset \cup \dataset'} \geq \deff{ \gamma }_{\dataset} + \deff{ \gamma }_{\dataset'} \geq \deff{ \gamma }_{\dataset \cup \dataset'}.
    \end{align*}
\end{lemma}
\vspace{-.5\baselineskip}
\tododout{need to reflect the fact that we don't need bound on RLS decrease anymore}
While in \sequentialalg we were merging an $\varepsilon$-accurate dictionary $\coldict_t$ and a new
point, which is equivalent to a perfect, $0$-accurate dictionary,
in \parallelalg both dictionaries used in a merge are only $\varepsilon$-accurate.
To balance this change, we introduce a new
estimator,
    \begin{align}\label{eq:rls-estimator-merge}
        &\atau_{\dataset \cup \dataset',i} = \tfrac{1-\varepsilon}{\gamma}(k_{i,i}- \:k_{i}^\transp\wb{\selmatrix}(\wb{\selmatrix}^\transp\kermatrix\wb{\selmatrix} + (1+\varepsilon)\gamma\:I)^{-1}\wb{\selmatrix}^\transp\:k_{i}),
    \end{align}

\vspace{-\baselineskip}
where $\wb{\selmatrix}$ is the selection matrix associated with the temporary dictionary $\wb{\coldict} = \coldict_{\dataset} \cup \coldict_{\dataset'}$.
Eq.~\ref{eq:rls-estimator-merge} has similar guarantees as Lem.~\ref{lem:fast-rls},
with only a slightly larger $\alpha$.

\begin{figure}[t]
\begin{center}
\vspace{-0.05in}
\includegraphics[height=0.95\columnwidth]{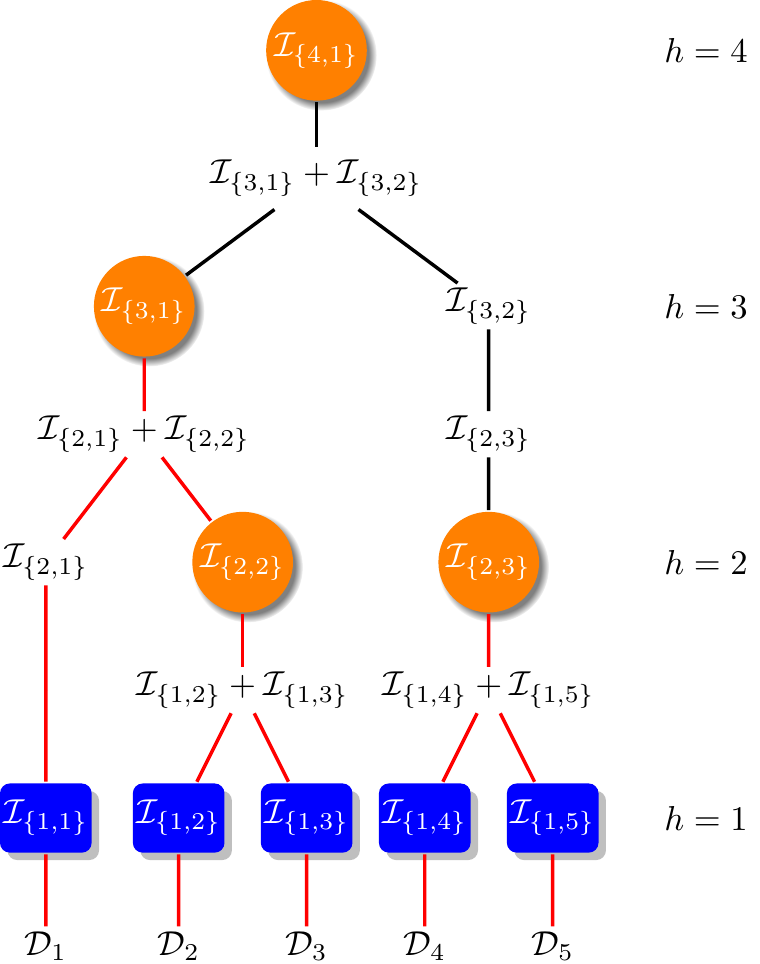}
\vspace{-0.15in}
\end{center}
\caption{Merge tree for Alg.~\ref{alg:distributedalg} with an arbitrary partitioning and merging scheme.}\label{fig:arbitrary.merge-trees}
\vspace{-0.15in}
\end{figure}
\begin{lemma}\label{lem:fast-rls-merge}
    Given two disjoint datasets $\dataset, \dataset'$,
    and two $\varepsilon$-approximate dictionaries
    $\coldict_{\dataset}$, $\coldict_{\dataset'}$,
    let $\wb{\coldict} = \coldict_{\dataset} \cup \coldict_{\dataset'}$
    and $\wb{\selmatrix}$ be the associated selection matrix.
    Let $\kermatrix$ be the kernel matrix computed on $\dataset \cup \dataset'$,
    $\:k_{i}$ its $i$-th column, and $\tau_{\dataset \cup \dataset',i}$
    the RLS of $\:k_i$.
    Then for all $i$ in $\wb{\coldict}$ such that $q_{i} \neq 0$, the estimator in Eq.~\ref{eq:rls-estimator-merge} is
    $\alpha$-accurate, i.e. it satisfies $\tau_{\dataset \cup \dataset',i}/\alpha \leq \atau_{\dataset \cup \dataset',i} \leq \tau_{\dataset \cup \dataset',i}$, with $\alpha~=~( 1-\varepsilon )/( 1+3\varepsilon ).$
                \end{lemma}
\vspace{-.5\baselineskip}
Given these guarantees, the analysis of \parallelalg follows similar steps as \sequentialalg. Given
$\varepsilon$-accurate dictionaries, we obtain $\alpha$-accurate RLS estimates
$\atau_{\dataset \cup \dataset',i}$ that can be used to resample all points in
$\wb{\coldict}$ and generate a new dictionary $\coldict_{\dataset,\dataset'}$
that is $\varepsilon$-accurate.
To formalize a result equivalent to Thm.~\ref{thm:sequential-alg-main}, we  introduce additional notation: We index each node in the merge tree by its height
$h$ and position $l$. We denote  the dictionary associated
to node $\{h,l\}$ by $\coldict_{\{h,l\}}$ and the collection of
all dictionaries available at height~$h$ of the merge tree by $\dictpool_{h} = \{\coldict_{\{h,l\}}\}$. We also use
$\kermatrix_{\{h,l\}}$ to refer to the kernel matrix constructed from the
datasets $\dataset_{\{h,l\}}$, which contains all points present in the leaves reachable from node $\{h,l\}$.
For instance in
Fig.~\ref{fig:arbitrary.merge-trees}, node $\{3,1\}$ is associated with $\coldict_{\{3,1\}}$,
which is an $\varepsilon$-approximate dictionary of the kernel matrix $\kermatrix_{\{3,1\}}$
constructed from the dataset $\dataset_{\{3,1\}}$. $\dataset_{\{3,1\}}$ contains $\dataset_1$,
$\dataset_2$, $\dataset_3$ (descendent nodes are highlighted in red) and it has dimension
$(|\dataset_1|\!+\!|\dataset_2|\!+\!|\dataset_3|)$.
Theorem~\ref{thm:parallel-alg-main} summarizes the guarantees for \parallelalg.

\begin{theorem}\label{thm:parallel-alg-main} 
Let $\varepsilon>0$ be the
accuracy parameter, $\gamma > 1$ the regularization factor, and $0 < \delta
< 1$ the prob.~of failure. Given an arbitrary dataset $\dataset$ and a merge tree structure of height $k$ as
input together with parameters $\varepsilon$, $\gamma$, and $\delta$, we run
\parallelalg with
\begin{align*}
\wb{q} =  \frac{39\alpha\log\left(2n/\delta\right)}{\varepsilon^{2}}\cdot
\end{align*}
where $\alpha = (1+3\varepsilon)/(1-\varepsilon)$. Then, w.p.\@ at least $1-\delta$, \parallelalg generates a sequence of collections of dictionaries $\{\dictpool_{h}\}_{h=1}^k$ such that each dictionary $\coldict_{\{h,l\}}$ in $\dictpool_h$ is $\varepsilon$-accurate (Eq.~\ref{cond:concentr-condition}) w.r.t.\@ to $\kermatrix_{\{h,l\}}$, and that at any node $l$ of height $h$ the size of the dictionary is
bounded as $|\coldict_{\{h,l\}}| \leq 3\wb{q}\deff{\gamma}_{\{h,l\}}.$
The cumulative (across nodes) space and time requirementsof the algorithm depend on the exact shape of the merge tree.
\end{theorem}

\vspace{-.5\baselineskip}
Theorem \ref{thm:parallel-alg-main} gives approximation and space guarantees
for \emph{every} node of the tree. In other words, it guarantees that
each intermediate dictionary processed by \parallelalg is both
an $\varepsilon$-accurate approximation of the datasets used to generate it,
and requires a small space proportional to the effective dimension
of the same dataset.
From an accuracy perspective, \parallelalg provides exactly the same
guarantees of \sequentialalg.
Analysing the complexity of \parallelalg is however more complex,
since the order and arguments of the \updatedictop operations
is determined by the merge tree.
We distinguish between the time and work complexity of a tree by defining the
time complexity as the amount of time necessary to compute the final solution,
and the work complexity as the total amount of operations carried out by all
machines in the tree in order to compute the final solution.
We consider two special cases, a fully balanced tree (all inner nodes have either
two inner nodes as children or two leaves), and a fully unbalanced tree (all
inner nodes have exactly one inner node and one leaf as children).
For both cases, we consider trees where each leaf dataset contains a single point $\dataset_i = \{\:x_i\}$.
In the fully unbalanced tree, we always merge the current dictionary with a new dataset
(a single new point) and no \mergedictop operation can be carried
out in parallel. Unsurprisingly, the sequential algorithm induced by this
merge tree is strictly equivalent to \sequentialalg.
Computing a solution in the fully unbalanced tree takes $\bigotime(n\deff{\gamma}_{\dataset}^3\wb{q}^3)$ time
with a total work that is also $\bigotime(n\deff{\gamma}_{\dataset}^3\wb{q}^3)$,
as reported in Thm.~\ref{thm:sequential-alg-main}.
On the opposite end, the fully balanced tree needs to invert
a $\deff{\gamma}_{\{h,l\}}$ dimensional matrix at each layer of the tree for a total of $\log(n)$ layers. Bounding all $\deff{\gamma}_{\{h,l\}}$
with $\deff{\gamma}_{\dataset}$, gives a complexity for computing the final solution of $\bigotime(\log(n)\wb{q}^3\deff{\gamma}_{\dataset}^3)$ time, with a huge
improvement on \sequentialalg.  Surprisingly, the total work is only
twice $\bigotime(n\wb{q}^3\deff{\gamma}_{\dataset}^3)$, since
at each layer $h$ we perform $n/2^{h}$ inversions (on $n/2^{h}$ machines),
and the sum across all layers is $\sum_{h=1}^{\log(n)} n/2^{h} \leq 2n$.
Therefore, we can compute a solution in a much shorter time than \sequentialalg,
with a comparable amount of work, but at the expense of requiring much more memory
across multiple machines, since at layer $h$, the sum $\sum_{l=1}^{|\dictpool_h|} \deff{\gamma}_{\{h,l\}}$
can be much larger than $\deff{\gamma}_{\dataset}$. Nonetheless,
this is partly alleviated by the fact that each node $\{h,l\}$
locally requires only $\deff{\gamma}_{\{h,l\}}^2 \leq \deff{\gamma}_{\dataset}^2$
memory.

\textbf{Proof sketch:}
Although \parallelalg is conceptually simple, providing guarantees on its space/time complexity and
accuracy is far from trivial.
The first step in the proof is to carefully decompose the failure event across the whole
merge tree into separate failure events for each merge node $\{h,l\}$,
and for each node construct a random process $\:Y$ that models how Alg.~\ref{alg:distributedalg}
generates the dictionary $\coldict_{\{h,l\}}$.
Notice that these processes are sequential in nature and the various steps (layers in the tree)
are not i.i.d. Furthermore, the variance of $\:Y$ is potentially large, and cannot
be bounded uniformly. Instead, we take a more refined approach,
inspired by \citet{pachocki2016analysis}, that 1) uses
Freedman's inequality to treat $\:W$, the variance of process $\:Y$, as a random object
itself, 2) applies a stochastic dominance argument to $\:W$ to reduce it to a sum of
i.i.d.\@ r.v.\@ and only then we can 3) apply i.i.d.\@ concentrations to obtain the desired
result.

 \begin{table*}[t]
\vspace{-0.15in}
\setstretch{1.3}
\centering
\begin{tabular}{|c|c|c|c|c|}
\hline
& $\wt{\bigotime}(\text{Time})$  & $\wt{\bigotime}(|\coldict_n|)$ & Increm.\\
        \hline
        \textsc{Exact} & $n^3$ & $\pdeff_n$ & -  \\
        \textsc{Uniform} (\citet{bach2013sharp}) & ${d_{\text{max},n}}$ & $d_{\text{max},n}$ & No  \\
        \citet{alaoui2014fast} & $n (|\coldict_n|)^2$ & $(\tfrac{\lambda_{\min} + n\gamma\vareps}{\lambda_{\min} - n\gamma\vareps})\pdeff_n + \Tr(\kermatrix_n)/\gamma$ &  No  \\
        \citet{calandriello2016analysis} & $n(|\coldict_n|)^3$ & $ \tfrac{\lambda_{\max}}{\gamma}  \pdeff_n$ &  Yes  \\
            \sequentialalg & $n \pdeff_n^3$ & $ \pdeff_n$  & Yes  \\
            \parallelalg & $n \pdeff_n^3/k$ & $ \pdeff_n$  & Yes  \\
        \textsc{RLS-sampling} & $n$ & $\pdeff_n$ & -  \\
            \hline
\end{tabular}
\caption{{\small Comparison of \nystrom methods. $\lambda_{\max}$ and $\lambda_{\min}$ refer to largest and smallest eigenvalues of $\protect \kermatrix_n$.}}\label{fig:table-comparison}
\vspace{-.5\baselineskip}
\end{table*}

\vspace{-0.1in}
\section{Applications}\label{sec:generalization}
\vspace{-0.1in}

In this section, we show how our approximation guarantees  translate into guarantees for typical kernel methods.
As an example, we use kernel ridge regression.
We begin by showing how to get an accurate approximation $\akermatrix_n$
from an $\varepsilon$-accurate dictionary.
\begin{lemma}\label{lem:nyst-app-guar}
    Given an $\varepsilon$-accurate dictionary $\coldict_{t}$
    of matrix $\kermatrix_{t}$, and the selection
    matrix $\:S_t$, the \emph{regularized}
Nystr\"{o}m approximation of $\kermatrix_t$ is defined as
    \begin{align}\label{eq:nystrom}
    \wt{\kermatrix}_n = \kermatrix_n \selmatrix_n(\selmatrix_n^\transp \kermatrix_n \selmatrix_n + \gamma\:I_m)^{-1}\selmatrix_n^\transp \kermatrix_n,
    \end{align}
    and satisfies
    \begin{align}\label{cond:nyst-app-guar}
       \:0 \preceq \kermatrix_t - \akermatrix_t \preceq  \frac{\gamma}{1-\varepsilon}\kermatrix_t(\kermatrix_t + \gamma\:I)^{-1} \preceq \frac{\gamma}{1-\varepsilon}\:I.
    \end{align}
\end{lemma}

\vspace{-.5\baselineskip}
This is not the only choice of an approximation from dictionary $\coldict_n$. For instance,
\citet{musco2016provably} show similar result for an unregularized \nystrom approximation
and \citet{rudi2015less} for a smaller $\kermatrix_n$, construct the estimator only for the points in
$\coldict_n$.
Let
$\:C = \kermatrix_n \selmatrix_n \in \Real^{n \times m}$, $\:W = (\selmatrix_n^\transp \kermatrix_n \selmatrix_n + \gamma\:I_m) \in \Real^{m \times m}$,
with $m = |\coldict_n|$,
and 
using the Woodbury formula define the regression weights as \begin{align}
    \wt{\:w}_n =& (\akermatrix_n + \mu \:I_{n})^{-1} \:y_n = (\:C\:W^{-1}\:C^\transp + \mu \:I_{n})^{-1} \:y_n\nonumber\\
        =&\frac{1}{\mu} \left(\:y_n - \:C\left( \:C^\transp \:C + \mu\:W\right)^{-1}\:C^\transp\:y_n\right).\label{eq:linear-system-transformed}
\end{align}

\vspace{-.5\baselineskip}
Computing $( \:C^\transp \:C + \mu\:W)^{-1}$, inverting it, and the other matrix-matrix multiplication take $\bigotime(nm^2 + m^3)$ time,
and require to store at most an $n \times m$ matrix. Therefore the final
complexity of computing $\wt{\:w}_n$ is reduced from $\bigotime(n^3)$ to $\bigotime(nm^2 + m^3)$ time,
and from $\bigotime(n^2)$ to $\bigotime(nm)$ space.
We now provide guarantees
for the empirical risk of~$\wt{\:w}_n$ in a fixed design setting.
\begin{corollary}[{\cite[Thm. 3]{alaoui2014fast}}]\label{cor:fixed-design-risk-guarantees}
    For an arbitrary dataset $\dataset$, let $\kermatrix$ be the kernel matrix
    constructed on $\dataset$.
    Run \sequentialalg or \parallelalg with regularization
    parameter $\gamma$. Then, the solution
    $\wt{\:w}$ computed using the regularized \nystrom approximation
    $\akermatrix$ satisfies
    \begin{align*}
        \mathcal{R}_{\dataset}(\wt{\:w}) \leq& \left(1 + \frac{\gamma}{\mu}\frac{1}{1-\varepsilon}\right)^{2} \mathcal{R}_{\dataset}(\wh{\:w}),
    \end{align*}
    where $\mu$ is the regularization of kernel ridge regression
    and $\mathcal{R}_{\dataset}(\wt{\:w})$ is the empirical risk on $\dataset$.
\end{corollary}
Using tools from \cite{rudi2015less},
and under some mild assumption on the kernel function $\kerfunc(\cdot,\cdot)$
and the dataset~$\dataset$,
similar results can be derived for the random design setting.

\textbf{Other applications.} 
The projection $\:P_n$ naturally appears in some form across nearly
all kernel-based methods. Therefore, in addition to KRR in the fixed~\cite{alaoui2014fast}
and random~\cite{rudi2015less} design setting, any kernel matrix approximation that provides $\varepsilon$-accuracy guarantees on $\:P_n$
can be used to provide guarantees for a variety of other kernel problems.
As an example, \citet{musco2016provably} show this is the case for
kernel PCA~\cite{scholkopf1999kernel}, kernel CCA with regularization,
and kernel $K$-means clustering. Similarly, inducing points methods
for Gaussian Processes \cite{quinonero-candela_unifying_2005, rasmussen_gaussian_2006} can benefit from fast and provably accurate
dictionary construction.

\vspace{-0.2cm}
\section{Discussion}\label{sec:conclusions}
\vspace{-0.2cm}

Tab.~\ref{fig:table-comparison} compares several kernel approximation methods w.r.t.\@ the time complexity
required to compute an $\varepsilon$-accurate dictionary, as well as the size of the final dictionary
$|\coldict_n|$. Note that for all methods the final complexity to construct
an approximate $\wt{\kermatrix}_n$ from $\coldict_n$ (e.g.\@ using Eq.~\ref{eq:nystrom}) scales as $\wt{\bigotime}(n|\coldict_n|^2 + |\coldict_n|^3)$ time
and $\wt{\bigotime}(n|\coldict_n| + |\coldict_n|^2)$ space.
For all methods, we omit $\bigotime(\log(n))$ factors.
We first report \textsc{RLS-sampling}, a fictitious algorithm that
receives the exact RLSs as input, as an ideal baseline
for all RLS sampling algorithms.
The space complexity of uniform sampling~\cite{bach2013sharp} scales with the maximal degree
of freedom~$d_{\text{max}}$. Since
$d_{\text{max}} = n \max_i \tau_{n,i} \geq \sum_{i} \tau_{n,i} = \pdeff_n$,
uniform sampling is often outperformed by RLS sampling.
Moreover, uniform sampling needs to know $d_{\text{max}}$ in advance to guarantee
$\varepsilon$-accuracy, and this quantity is expensive to compute \cite{woodruff2014sketching}.
While \citet{alaoui2014fast} also sample according to RLS, their two-pass
estimator does not preserve the same level of accuracy. In particular, the first pass requires to
sample $\bigotime\left(n\gamma\vareps/(\lambda_{\min} - n\gamma\vareps)\right)$
columns, which quickly grows above $n^2$ when $\lambda_{\min}$ becomes small.
Finally, \citet{calandriello2016analysis} require that the maximum dictionary
size is fixed in advance, which implies some information about  the effective
dimensions $\pdeff_n$, and requires estimating both $\atau_{t,i}$ and
$\adeff{\gamma}_t$. This extra estimation effort causes an
additional $\lambda_{\max}/\gamma$ factor to appear in the space complexity. This
factor cannot be easily estimated, and it leads to a space
complexity of $n^3$ in the worst case. 
Therefore, we can see from the table that \sequentialalg achieves the
same space complexity (up to constant factors) as knowing the RLS in advance and hence outperforms previous methods.
Moreover, when parallelized across $k$ machines, \parallelalg can reduce this
linear runtime by a factor of $k$ (linear scaling) with little communication cost.

A recent method by~\citet{musco2016provably} achieves comparable space and time guarantees as \sequentialalg.\footnote{The technical report of~\citet{musco2016provably} was developed independently from our work.} While they rely on a similar estimator, the two approaches are very different. Their method is batch in nature, estimating leverage scores using repeated independent sampling from the whole dataset, and it \emph{requires multiple passes} on the data. On the other hand, \sequentialalg is intrinsically sequential and it \emph{only requires one single pass} on $\dataset$, as points are ``forgotten'' once they are dropped from the dictionary. Furthermore, the different structure requires remarkably different tools for the analysis. While the method of~\cite{musco2016provably} can directly use i.i.d.\@ concentration inequalities (for the price of needing several passes), we need to rely on more sophisticated martingale arguments to consider the sequential stochastic process of \sequentialalg. Furthermore, \citet{musco2016provably} requires centralized coordination after each of the sampling passes, and cannot leverage distributed architectures to match \parallelalg's sub-linear runtime.

\textbf{Future developments}
Both \sequentialalg and \parallelalg need to know in advance the size
of the dataset $n$ to tune $\wb{q}$. An interesting question is whether it is
possible to adaptively adjust the $\wb{q}$ parameter at runtime. This would
allow us to continue updating $\coldict_n$ and
indefinitely process new data beyond the initial dataset $\dataset$.
It is also interesting to see whether \sequentialalg could be used in conjuntion with existing meta-algorithms (e.g., \cite{kumar_sampling_2012} with model averaging)
for kernel matrix approximation that can leverage an accurate sampling scheme as a black-box,
and what kind of improvements we could obtain.

 {\small
\vspace{-0.15in}
\paragraph{\small Acknowledgements}
\label{sec:Acknowledgements}
The research presented was supported by French Ministry of
Higher Education and Research, Nord-Pas-de-Calais Regional Council and French National Research Agency projects ExTra-Learn (n.ANR-14-CE24-0010-01) and BoB (n.ANR-16-CE23-0003) }

\bibliographystyle{plainnat}

\onecolumn
\appendix

\begin{table}
\begin{center}
\begin{tabular}{|m{\textwidth}|}
\hline
\textbf{Notation summary}\\
\hline
$\rhd$ Kernel matrix at time $t$ $\kermatrix_t \in \Real^{t \times t}$, $i$-th column $\:k_{t,i} \in \Real^{t}$, $(i,j)$-th entry $k_{i,j} \in \Real$\\
$\rhd$ Eigendecomposition $\kermatrix_t = \:U_t\:\Lambda_t\:U_t^\transp \in \Real^{t \times t}$, eigenvector matrix $\:U_t \in \Real^{t \times t}$ and
diagonal SDP eigenvalues matrix $\:\Lambda_t \in \Real^{t \times t}$.\\
$\rhd$ Kernel matrix at time $t$ in the RKHS, $\kermatrix_t = \featkermatrix_t^\transp\featkermatrix_t$, with $\featkermatrix_t \in \Real^{D \times t}$\\
$\rhd$ SVD decomposition $\featkermatrix_t = \:V_t\:\Sigma_t\:U_t^\transp \in \Real^{D \times t}$
with left singular vectors, $\:V_t \in \Real^{D \times D}$, right singular vector $\:U_t$ (same as $\kermatrix_t$), and singular values
$\:\Sigma_t \in \Real^{D \times t}$ ($\sigma_i $ on the main diagonal and zeros under it)\\
$\rhd$ SVD decomposition $\featkermatrix_t = \:V_t\:\Sigma_t\:U_t^\transp \in \Real^{D \times t}$
with left singular vectors, $\:V_t \in \Real^{D \times D}$, right singular vector $\:U_t$ (same as $\kermatrix_t$), and singular values
$\:\Sigma_t \in \Real^{D \times t}$ ($\sigma_i $ on the main diagonal and zeros under it)\\
$\rhd$ $\:P_t = \:\Psi_t\:\Psi_t^\transp$ with $\:\Psi_t = (\kermatrix_t + \gamma\:I_t)^{-1/2}\kermatrix_t^{1/2} \in \Real^{t \times t}$\\
$\rhd$ Column dictionary at time $t$, $\coldict_t = \{(i, \wt{p}_{t,i}, q_{t,i})\}_{i=1}^t$\\
$\rhd$ Selection matrix $\selmatrix_t \in \Real^{t \times t}$ with $\{\sqrt{\frac{q_{t,i}}{\wb{q}\wt{p}_{t,i}}}\}_{i=1}^t$ on the diagonal\\
$\rhd$ $\wt{\:P}_t = \:\Psi_t\selmatrix_t\selmatrix_t^\transp\:\Psi_t^\transp$ with $\selmatrix_t \in \Real^{t \times t}$\\
\\
\textbf{For node ${\{h,l\}}$ in the merge tree, given }$\mathbf{\nhl = |\dataset_{\{h,l\}}|}$\\
$\rhd$ Similarly defined $\kermatrix_{\{h,l\}} \in \Real^{\nhl \times \nhl}, \:P_{\{h,l\}} \in \Real^{\nhl \times \nhl}$\\
$\rhd$ Similarly defined $\coldict_{\{h,l\}}, \selmatrix_{\{h,l\}} \in \Real^{\nhl \times \nhl}, \wt{\:P}_{\{h,l\}} \in \Real^{\nhl \times \nhl}$ \\
$\rhd$ For layer $h$ in the merge tree, block diagonal matrices $\kermatrix^h, \:P^h, \wt{\:P}^h$ with $\kermatrix_{\{h,l\}}, \:P_{\{h,l\}}, \wt{\:P}_{\{h,l\}}$ on the diagonal\\
$\rhd$ $\concmat = (\kermatrix_{\{h,l\}} + \gamma\:I_{\nhl})^{-1/2}\kermatrix_{\{h,l\}}^{1/2} \in \Real^{\nhl \times \nhl}$, with $\concvec_i$ its $i$-th column\\
$\rhd$ $\wt{\:P}^{\{h,l\}}_s = \sum_{i=1}^{\nhl_{\{h,l\}}} \frac{q_{s,i}}{\wb{q}\wt{p}_{s,i}}\concvec_i\concvec_i^\transp \in \Real^{\nhl \times \nhl}$ constructed using weights $\wt{p}_{s,i}$ from step $s$ of Alg.~\ref{alg:distributedalg} and $\concvec_i$ from ${\{h,l\}}$\\
$\rhd$ $\:Y_{s} = \:P_{\{h,l\}} - \wt{\:P}_{\{h,l\}} = \:P_{\{h,l\}} - \wt{\:P}^{\{h,l\}}_h \in \Real^{\nhl \times \nhl}$ sampling process based on $\kermatrix_{\{h,l\}}$, at final step $h$\\
$\rhd$ $\:Y_{s} = \:P_{\{h,l\}} - \wt{\:P}^{\{h,l\}}_s \in \Real^{\nhl \times \nhl}$ sampling process based on $\kermatrix_{\{h,l\}}$, at intermediate step $s$\\
$\rhd$ $\wb{\:Y}_{s} \in \Real^{\nhl \times \nhl}$ sampling process based on $\kermatrix_{\{h,l\}}$, with freezing\\
$\rhd$ $\:W_{h} \in \Real^{\nhl \times \nhl}$ total variance of sampling process based on $\kermatrix_{\{h,l\}}$, with freezing\\
\hline
\end{tabular}
\end{center}
\end{table}

\section{Preliminaries}\label{sec:app.preliminaries}

In this section, we introduce standard matrix results and equivalent definitions for the kernel matrix and the the projection error which use convenient representations exploiting the feature space. 
\textbf{Matrix identity.} We often use the following identity.
\begin{proposition}\label{sec:app-linalg1}
For any symmetric matrix $\:A \in \Real^{n \times m}$ and any $\gamma > 0$
\begin{align*}
\:A(\:A^\transp\:A + \gamma\:I_m)^{-1}\:A^\transp = \:A\:A^\transp(\:A\:A^\transp + \gamma\:I_n)^{-1}.
\end{align*}
For any symmetric matrix $\:A \in \Real^{n \times n}$ and diagonal matrix $\:B \in \Real^{n \times n}$
such that $\:B$ has $n-s$ zero entries, and $s$ non-zero entries,
define $\:C \in \Real^{n \times s}$ as the matrix obtained by removing all
zero columns in $\:B$.
Then
\begin{align*}
\:A\:B(\:B\:A\:B + \gamma\:I_m)^{-1}\:B\:A
=\:A\:C(\:C^\transp\:A\:C + \gamma\:I_s)^{-1}\:C^\transp\:A.
\end{align*}
For any appropriately shaped matrix $\:A,\:B,\:C$, with $\:A$ and $\:B$ invertible,
the Woodbury matrix identity states
\begin{align*}
(\:A + \:C\:B\:C^\transp)^{-1} = \:A^{-1} - \:A^{-1}\:C\left(\:C^\transp\:A^{-1}\:C + \:B^{-1}\right)^{-1}\:C^\transp\:A^{-1}.
\end{align*}
\end{proposition}

\textbf{Kernel matrix.}
Since $\kermatrix_t$ is real symmetric matrix, we can eigendecompose it as
$$\kermatrix_t = \:U_t\:\Lambda_t\:U_t^\transp,$$
where $\:U_t \in \Real^{t \times t}$ is the eigenvector matrix and $\:\Lambda_t \in \Real^{t \times t}$ is the diagonal eigenvalue matrix, with all non-negative elements since $\kermatrix_t$ is PSD. Considering the feature mapping from $\X$ to the RKHS, we can write the kernel matrix as
$$\kermatrix_t = \featkermatrix_t^\transp\featkermatrix_t,$$
where $\featkermatrix_t \in \Real^{D \times t}$ is the feature matrix, whose SVD decomposition is
$$\featkermatrix_t = \:V_t\:\Sigma_t\:U_t^\transp,$$
where $\:V_t \in \Real^{D \times D}$ contains the left singular vectors, $\:\Sigma_t \in \Real^{D \times t}$ has the singular values $\sigma_i$ on the main diagonal (followed by zeros), and $\:U_t$ contains the right singular vector, which coincide with the eigenvectors of $\kermatrix_t$. Furthermore, we have 
$$\:\Lambda_t = \:\Sigma_t^\transp \:\Sigma_t,$$
and each eigenvalue $\lambda_i = \sigma^i$, with $i=1,\ldots,t$.

\textbf{Projection error.}
We derive a convenient lemma on the formulation of the projection error both in terms of kernels and feature space.
\begin{lemma}\label{lem:concentration-equivalence}
The following identity holds.
\begin{align*}
\norm{\:P_t - \wt{\:P}_t}{2}
&=\norm{(\kermatrix_t + \gamma\:I_t)^{-1/2}\kermatrix_t^{1/2}(\:I_t - \:S_t\:S_t^\transp)\kermatrix_t^{1/2}(\kermatrix_t + \gamma\:I_t)^{-1/2}}{2}\\
&= \norm{(\featkermatrix_t\featkermatrix_t^\transp + \gamma\:I_D)^{-1/2}\featkermatrix_t(\:I_t - \:S_t\:S_t^\transp)\featkermatrix_t^\transp(\featkermatrix_t\featkermatrix_t^\transp + \gamma\:I_D)^{-1/2}}{2}.
\end{align*}
\end{lemma}
\begin{proof}[Proof of Lemma \ref{lem:concentration-equivalence}]
Using the SVD decomposition,
\begin{align*}
&\norm{(\kermatrix_t + \gamma\:I_t)^{-1/2}\kermatrix_t^{1/2}(\:I_t - \:S_t\:S_t^\transp)\kermatrix_t^{1/2}(\kermatrix_t + \gamma\:I_t)^{-1/2}}{2}\\
&=\norm{(\:U_t\:\Lambda_t\:U^\transp + \gamma\:U_t\:I_t\:U_t^\transp)^{-1/2}\:U_t\:\Lambda_t^{1/2}\:U^\transp(\:I_t - \:S_t\:S_t^\transp)\:U_t\:\Lambda_t^{1/2}\:U^\transp(\:U_t\:\Lambda_t\:U^\transp + \gamma\:U_t\:I_t\:U_t^\transp)^{-1/2}}{2}\\
&=\norm{\:U_t(\:\Lambda_t + \gamma\:I_t)^{-1/2}\:\Lambda_t^{1/2}\:U^\transp(\:I_t - \:S_t\:S_t^\transp)\:U_t\:\Lambda_t^{1/2}(\:\Lambda_t + \gamma\:I_t)^{-1/2}\:U^\transp}{2}\\
&=\norm{\:U_t(\:\Lambda_t + \gamma\:I_t)^{-1/2}\:\Lambda_t^{1/2}\:U^\transp(\:I_t - \:S_t\:S_t^\transp)\:U_t\:\Lambda_t^{1/2}(\:\Lambda_t + \gamma\:I_t)^{-1/2}\:U^\transp}{2}\\
&=\norm{(\:\Lambda_t + \gamma\:I_t)^{-1/2}\:\Lambda_t^{1/2}\:U^\transp(\:I_t - \:S_t\:S_t^\transp)\:U_t\:\Lambda_t^{1/2}(\:\Lambda_t + \gamma\:I_t)^{-1/2}}{2}\\
&=\norm{(\:\Sigma_t\:\Sigma_t^\transp + \gamma\:I_D)^{-1/2}\:\Sigma_t\:U^\transp(\:I_t - \:S_t\:S_t^\transp)\:U_t\:\Sigma_t^\transp(\:\Sigma_t\:\Sigma_t^\transp + \gamma\:I_D)^{-1/2}}{2}\\
&=\norm{\:V_t(\:\Sigma_t\:\Sigma_t^\transp + \gamma\:I_D)^{-1/2}\:\Sigma_t\:U^\transp(\:I_t - \:S_t\:S_t^\transp)\:U_t\:\Sigma_t^\transp(\:\Sigma_t\:\Sigma_t^\transp + \gamma\:I_D)^{-1/2}\:V_t^\transp}{2}\\
&=\norm{\:V_t(\:\Sigma_t\:U_t\:U_t^\transp\:\Sigma_t^\transp + \gamma\:I_D)^{-1/2}\:V_t^\transp\:V_t\:\Sigma_t\:U^\transp(\:I_t - \:S_t\:S_t^\transp)\:U_t\:\Sigma_t^\transp\:V^\transp\:V(\:\Sigma_t\:U_t\:U_t^\transp\:\Sigma_t^\transp + \gamma\:I_D)^{-1/2}\:V_t^\transp}{2}\\
&= \norm{(\featkermatrix_t\featkermatrix_t^\transp + \gamma\:I_D)^{-1/2}\featkermatrix_t(\:I_t - \:S_t\:S_t^\transp)\featkermatrix_t^\transp(\featkermatrix_t\featkermatrix_t^\transp + \gamma\:I_D)^{-1/2}}{2}.
\end{align*}
\end{proof}

\section{Ridge Leverage Scores and Effective Dimension (Proof of Lemma~\ref{lem:monotone-decrease-prob} and \ref{lem:monotone-decrease-prob-merge})}\label{sec:app-rls-estimation-proofs}

From \citet[Lem. 1]{calandriello2016analysis} we know that $\tau_{t,i} \leq \tau_{t-1,i}$ and
$\deff{\gamma}_t \geq \deff{\gamma}_{t-1}$. We will now prove the lower bound
$\tau_{t,i} \geq \tau_{t-1,i}/(\tau_{t-1,1}+1)$.

Considering the definition of $\tau_{t,i}$ in terms of $\bphi_i$ and $\featkermatrix_t$, and applying the Sherman-Morrison formula we obtain
\begin{align*}
\tau_{t,i} &= \bphi_i^\transp(\featkermatrix_t\featkermatrix_t^\transp + \gamma\:I)^{-1}\bphi_i
= \bphi_i^\transp(\featkermatrix_{t-1}\featkermatrix_{t-1}^\transp + \bphi_t\bphi_t^\transp + \gamma\:I)^{-1}\bphi_i\\
&= \bphi_i^\transp(\featkermatrix_{t-1}\featkermatrix_{t-1}^\transp  + \gamma\:I)^{-1}\bphi_i
- \frac{\bphi_i^\transp(\featkermatrix_{t-1}\featkermatrix_{t-1}^\transp  + \gamma\:I)^{-1}\bphi_t\bphi_t^\transp(\featkermatrix_{t-1}\featkermatrix_{t-1}^\transp  + \gamma\:I)^{-1}\bphi_i}{1+\bphi_t^\transp(\featkermatrix_{t-1}\featkermatrix_{t-1}^\transp  + \gamma\:I)^{-1}\bphi_t}\\
&= \tau_{t-1,i}
- \frac{\bphi_i^\transp(\featkermatrix_{t-1}\featkermatrix_{t-1}^\transp  + \gamma\:I)^{-1}\bphi_t\bphi_t^\transp(\featkermatrix_{t-1}\featkermatrix_{t-1}^\transp  + \gamma\:I)^{-1}\bphi_i}{1+\bphi_t^\transp(\featkermatrix_{t-1}\featkermatrix_{t-1}^\transp  + \gamma\:I)^{-1}\bphi_t}\cdot
\end{align*}
Let
\begin{align*}
\:x = (\featkermatrix_{t-1}\featkermatrix_{t-1}^\transp  + \gamma\:I)^{-1/2}\bphi_i
\quad \text{and} \quad  \:y = (\featkermatrix_{t-1}\featkermatrix_{t-1}^\transp  + \gamma\:I)^{-1/2}\bphi_t.
\end{align*}
Then $\tau_{t,i}/\tau_{t-1,i}$ is equal to
\begin{align*}
\frac{\tau_{t,i}}{\tau_{t-1,i}} = 1
- \frac{(\bphi_t^\transp(\featkermatrix_{t-1}\featkermatrix_{t-1}^\transp  + \gamma\:I)^{-1}\bphi_i)^2}{(1+\bphi_t^\transp(\featkermatrix_{t-1}\featkermatrix_{t-1}^\transp  + \gamma\:I)^{-1}\bphi_t)\bphi_i^\transp(\featkermatrix_{t-1}\featkermatrix_{t-1}^\transp  + \gamma\:I)^{-1}\bphi_i}
    &= 1 -  \frac{(\:y^\transp\:x)^{2}}{ \left(1 + \:y^\transp\:y\right)\:x^\transp\:x}\cdot
\end{align*}
Defining the cosine between $\:y$ and $\:x$
as $\text{cos}(\:y,\:x) = \:y^\transp\:x/(\normsmall{\:x}\normsmall{\:y})$,
we have that
\begin{align*}
1 -  \frac{(\:y^\transp\:x)^{2}}{ \left(1 + \:y^\transp\:y\right)\:x^\transp\:x}
= 1 -  \frac{\:y^\transp\:y\:x^\transp\:x\text{cos}(\:y,\:x)^2}{ \left(1 + \:y^\transp\:y\right)\:x^\transp\:x}
= 1 - \frac{\normsmall{\:y}^2}{ 1 + \normsmall{\:y}^2}\text{cos}(\:y,\:x)^2,
\end{align*}
where $\frac{\normsmall{\:y}^2}{ 1 + \normsmall{\:y}^2}$ depends only on the
norm of $\:y$ and not its direction, and $\text{cos}(\:y,\:x)$ depends only
on the direction of $\:y$ and is
maximized when $\:y = \:x$. Therefore,
\begin{align*}
\frac{\tau_{t+1,i}}{\tau_{t,i}}
= 1 -  \frac{(\:y^\transp\:x)^{2}}{ \left(1 + \:y^\transp\:y\right)\:x^\transp\:x}
= 1 - \frac{\normsmall{\:y}^2}{ 1 + \normsmall{\:y}^2}\text{cos}(\:y,\:x)^2
\geq 1 - \frac{\normsmall{\:x}^2}{ 1 + \normsmall{\:x}^2}
= \frac{1}{ 1 + \normsmall{\:x}^2}
= \frac{1}{1+\tau_{t,i}}\CommaBin
\end{align*}
which concludes the proof of Lem.~\ref{lem:monotone-decrease-prob}.

For Lem.~\ref{lem:monotone-decrease-prob-merge}, the first point $\tau_{i,\dataset} \geq \tau_{i,\dataset \cup \dataset'}$ can be easily proven by choosing $\dataset$ to construct a kernel matrix $\kermatrix_{\dataset}$, and then invoke
Lem.~\ref{lem:monotone-decrease-prob} as we add one sample at a time
from $\dataset'$.
Also as easily for $\deff{\gamma}_{\dataset} + \deff{\gamma}_{\dataset'}$
we have
\begin{align*}
    \deff{\gamma}_{\dataset} + \deff{\gamma}_{\dataset'}
    \leq 2\max\{\deff{\gamma}_{\dataset}; \deff{\gamma}_{\dataset'}\}
    \leq 2\max\{\deff{\gamma}_{\dataset \cup \dataset'}; \deff{\gamma}_{\dataset \cup \dataset'}\}
    = 2\deff{\gamma}_{\dataset \cup \dataset'}.
\end{align*}
Finally, we prove the other side of the inequality for $\deff{\gamma}_{\dataset} + \deff{\gamma}_{\dataset}$.
Let $\featkermatrix_{\dataset}, \featkermatrix_{\dataset'}$ be the matrices
constructed using the feature vectors of the samples in $\dataset$ and
$\dataset'$ respectively. Then,
\begin{align*}
    \deff{\gamma}_{\dataset} + \deff{\gamma}_{\dataset'}
    &= \sum_{i \in \dataset} \tau_{\dataset, i} + \sum_{i \in \dataset'} \tau_{\dataset',i}
    = \sum_{i \in \dataset} \:\phi_i^\transp(\featkermatrix_{\dataset}\featkermatrix_{\dataset}^\transp + \gamma\:I_{D})^{-1}\:\phi_i + \sum_{i \in \dataset'} \:\phi_i^\transp(\featkermatrix_{\dataset'}\featkermatrix_{\dataset'}^\transp + \gamma\:I_{D})^{-1}\:\phi_i\\
    &\geq \sum_{i \in \dataset} \:\phi_i^\transp(\featkermatrix_{\dataset \cup \dataset'}\featkermatrix_{\dataset \cup \dataset'}^\transp + \gamma\:I_{D})^{-1}\:\phi_i + \sum_{i \in \dataset'} \:\phi_i^\transp(\featkermatrix_{\dataset \cup \dataset'}\featkermatrix_{\dataset \cup \dataset'}^\transp + \gamma\:I_{D})^{-1}\:\phi_i\\
    &= \sum_{i \in \dataset \cup \dataset'} \:\phi_i^\transp(\featkermatrix_{\dataset \cup \dataset'}\featkermatrix_{\dataset \cup \dataset'}^\transp + \gamma\:I_{D})^{-1}\:\phi_i
    =\sum_{i \in \dataset \cup \dataset'} \tau_{\dataset \cup \dataset',i}
    = \deff{\gamma}_{\dataset \cup \dataset'}.
    \end{align*}

\section{Ridge Leverage Scores Estimation (Proof of Lemma~\ref{lem:fast-rls} and \ref{lem:fast-rls-merge})}\label{sec:app-rls-estimation-proofs}

We begin with a convenient reformulation of the ridge leverage scores,
\begin{align*}
\tau_{t,i}
&= \:e_{t,i}^\transp\kermatrix_t(\kermatrix_t + \gamma\:I_t)^{-1}\:e_{t,i}
= \:e_{t,i}^\transp\featkermatrix_t^\transp\featkermatrix_t(\featkermatrix_t^\transp\featkermatrix_t + \gamma\:I_t)^{-1}\:e_{t,i}\\
&= \:e_{t,i}^\transp\featkermatrix_t^\transp(\featkermatrix_t\featkermatrix_t^\transp + \gamma\:I_D)^{-1}\featkermatrix_t\:e_{t,i}
= \:\phi_{i}^\transp(\featkermatrix_t\featkermatrix_t^\transp + \gamma\:I_D)^{-1}\:\phi_{i}.
\end{align*}
This formulation, combined with Def.~\ref{def:eps-acc-dict}, suggests
$\:\phi_{i}^\transp(\featkermatrix_t\:S_t\:S_t^\transp\featkermatrix_t^\transp + \gamma\:I_D)^{-1}\:\phi_{i}$
as an estimator for $\tau_{t,i}$. However, at step $t$, we only have access to an $\varepsilon$-accurate dictionary
\wrt$\featkermatrix_{t-1}$ and not \wrt$\featkermatrix_{t}$. Therefore, we augment it with $(t,1,\wb{q})$ to construct
$\wb{\coldict}_t$ and the corresponding $\wb{\selmatrix}_t$, which will have
$[\wb{\selmatrix}_t]_{t,t} = 1$.
We will now show how to implement this estimator efficiently.
From Prop.~\ref{sec:app-linalg1}, we apply Woodbury matrix identity
with $\:A = \gamma\:I$, $\:B = \:I$ and $\:C = \featkermatrix_t\wb{\:S}_t$,
\begin{align*}
\atau_{t,i}
&= (1-\varepsilon)\:\phi_{i}^\transp(\featkermatrix_t\wb{\:S}_t\wb{\:S}_t^\transp\featkermatrix_t^\transp + \gamma\:I_D)^{-1}\:\phi_{i}\\
&= (1-\varepsilon)\:\phi_{i}^\transp(\featkermatrix_t\wb{\:S}_t\:I_t\wb{\:S}_t^\transp\featkermatrix_t^\transp + \gamma\:I_D)^{-1}\:\phi_{i}\\
(Prop.~\ref{sec:app-linalg1}) &= (1-\varepsilon)\:\phi_{i}^\transp\left(\frac{1}{\gamma}\:I_{D} - \frac{1}{\gamma^2}\featkermatrix_t\wb{\:S}_t\left(\frac{1}{\gamma}\wb{\:S}_t^\transp\featkermatrix_t^\transp\featkermatrix_t\wb{\:S}_t + \:I_t\right)^{-1}\wb{\:S}_t^\transp\featkermatrix_t^\transp\right)\:\phi_{i}\\
 &= \frac{(1-\varepsilon)}{\gamma}\:\phi_{i}^\transp\left(\:I_{D} - \featkermatrix_t\wb{\:S}_t\left(\wb{\:S}_t^\transp\featkermatrix_t^\transp\featkermatrix_t\wb{\:S}_t + \gamma\:I_t\right)^{-1}\wb{\:S}_t^\transp\featkermatrix_t^\transp\right)\:\phi_{i}\\
 &= \frac{(1-\varepsilon)}{\gamma}\left(\:\phi_{i}^\transp\:\phi_{i} - \:\phi_{i}^\transp\featkermatrix_t\wb{\:S}_t\left(\wb{\:S}_t^\transp\featkermatrix_t^\transp\featkermatrix_t\wb{\:S}_t + \gamma\:I_t\right)^{-1}\wb{\:S}_t^\transp\featkermatrix_t^\transp\:\phi_{i}\right)\\
&= \frac{(1-\varepsilon)}{\gamma}\left(k_{i,i} - \:k_{t,i}\wb{\:S}_t(\wb{\:S}_t^\transp\kermatrix_t\wb{\:S}_t + \gamma\:I_t)^{-1}\wb{\:S}_t^\transp\:k_{t,i}\right),
\end{align*}
which is the estimator defined in Eq.~\ref{eq:rls-estimator}.

We can generalize this estimator to the case where instead of using
a single dictionary and fresh data to estimate $\tau_{t,i}$,
we are using $k$ $\varepsilon$-accurate dictionaries $\coldict_k$.
Given disjoint datasets $\{\dataset_i\}_{i=1}^k$ with associated
feature matrices~$\featkermatrix_i$. From each dataset, construct
an $\varepsilon$-accurate dictionary $\coldict_i$, with its associated
selection matrix $\selmatrix_i$.

To estimate the RLS $\tau_i$ of point $i$ w.r.t.\@ the whole
dataset $\dataset = \cup_{j=1}^k\dataset_j$, and corresponding feature matrix $\featkermatrix$, we set the estimator to be
\begin{align*}
\atau_{i}
&= (1-\varepsilon)\:\phi_{i}^\transp\left(\sum_{j=1}^k\featkermatrix_j\:S_j\:S_j^\transp\featkermatrix_j^\transp + (1+(k-1)\varepsilon)\gamma\:I_D\right)^{-1}\:\phi_{i}.
\end{align*}

\textbf{Part 1: accuracy of the RLS estimator $\bm{\wt{\tau}_{i}}$.} 
Since each of the dictionaries $\coldict_{i}$ used to generate $\:S_{i}$ is $\varepsilon$-accurate,
we can use the equivalence
from Lem.~\ref{lem:concentration-equivalence}, 
\begin{align*}
\norm{\:P_{i} - \wt{\:P}_{i}}{2}
&= \norm{(\featkermatrix_{i}\featkermatrix_{i}^\transp + \gamma\:I_D)^{-1/2}(\featkermatrix_{i}\featkermatrix_{i}^\transp - \featkermatrix_{i}\:S_{i}\:S_{i}^\transp\featkermatrix_{i}^\transp)(\featkermatrix_{i}\featkermatrix_{i}^\transp + \gamma\:I_D)^{-1/2}}{2} \leq \varepsilon,
\end{align*}
which implies that 
\begin{align*}
(1-\vareps)\featkermatrix_{i}\featkermatrix_{i}^\transp - \varepsilon\gamma\:I_D \preceq \featkermatrix_{i}\:S_{i}\:S_{i}^\transp\featkermatrix_{i}^\transp \preceq (1+\vareps)\featkermatrix_{i}\featkermatrix_{i}^\transp + \varepsilon\gamma\:I_D.
\end{align*}
Therefore, we have
\begin{align*}
\atau_{i}
&= (1-\varepsilon)\:\phi_{i}^\transp\left(\sum_{j=1}^k\featkermatrix_j\:S_j\:S_j^\transp\featkermatrix_j^\transp + (1+(k-1)\varepsilon)\gamma\:I_D\right)^{-1}\:\phi_{i}\\
&\leq (1-\varepsilon)\:\phi_{i}^\transp\left(\sum_{j=1}^k\left((1-\varepsilon)\featkermatrix_j\featkermatrix_j^\transp-\varepsilon\gamma\:I_D\right) + (1+(k-1)\varepsilon)\gamma\:I_D\right)^{-1}\:\phi_{i}\\
&\leq (1-\varepsilon)\:\phi_{i}^\transp\left((1-\varepsilon)\featkermatrix\featkermatrix^\transp - k\varepsilon\gamma\:I_D + (1+(k-1)\varepsilon)\gamma\:I_D\right)^{-1}\:\phi_{i}\\
&= (1-\varepsilon)\:\phi_{i}^\transp\left((1-\varepsilon)(\featkermatrix\featkermatrix^\transp + \gamma\:I_D)\right)^{-1}\:\phi_{i}
= \frac{(1-\varepsilon)}{(1-\varepsilon)}\:\phi_{i}^\transp\left(\featkermatrix\featkermatrix^\transp + \gamma\:I_D\right)^{-1}\:\phi_{i} = \tau_i,
\end{align*}
and
\begin{align*}
\atau_{i}
&= (1-\varepsilon)\:\phi_{i}^\transp\left(\sum_{j=1}^k\featkermatrix_j\:S_j\:S_j^\transp\featkermatrix_j^\transp + (1+(k-1)\varepsilon)\gamma\:I_D\right)^{-1}\:\phi_{i}\\
&\geq (1-\varepsilon)\:\phi_{i}^\transp\left(\sum_{j=1}^k\left((1+\varepsilon)\featkermatrix_j\featkermatrix_j^\transp + \varepsilon\gamma\:I_D\right) + (1+(k-1)\varepsilon)\gamma\:I_D\right)^{-1}\:\phi_{i}\\
&= (1-\varepsilon)\:\phi_{i}^\transp\left((1+\varepsilon)\featkermatrix\featkermatrix^\transp + k\varepsilon\gamma\:I_D + (1+(k-1)\varepsilon)\gamma\:I_D\right)^{-1}\:\phi_{i}\\
&= (1-\varepsilon)\:\phi_{i}^\transp\left((1+\varepsilon)\featkermatrix\featkermatrix^\transp +(1 + (2k -1)\varepsilon) \gamma\:I_D)\right)^{-1}\:\phi_{i}\\
&\geq (1-\varepsilon)\:\phi_{i}^\transp\left((1+(2k -1)\varepsilon)\featkermatrix\featkermatrix^\transp +(1 + (2k -1)\varepsilon) \gamma\:I_D)\right)^{-1}\:\phi_{i}\\
&= \frac{(1-\varepsilon)}{(1+(2k-1)\varepsilon)}\:\phi_{i}^\transp\left(\featkermatrix\featkermatrix^\transp - \gamma\:I_D\right)^{-1}\:\phi_{i} = \frac{(1-\varepsilon)}{(1+(2k-1)\varepsilon)}\tau_i.
\end{align*}
Then, we can instantiate this result with $k=1$ to prove the accuracy claim in Lem.~\ref{lem:fast-rls},
and with $k=2$ to prove the accuracy claim in Lem.~\ref{lem:fast-rls-merge}.

\textbf{Part 2: accuracy of $\bm{\min\left\{\atau_{t},\; \atau_{t-1} \right\}}$.}
To simplify the notation, for this part of the proof we indicate with
$\:\tau_{t} \leq \:\tau_{t-1}$ that for each $i \in \{1,\dots,t-1\}$ we have $\tau_{t,i} \leq \tau_{t-1,i}$.
From Lem.~\ref{lem:monotone-decrease-prob}, we know that $\tau_{t-1}  \geq \tau_{t}$.
Given $\alpha$-accurate $\atau_{t}$ and $\atau_{t-1}$ we have the upper bound
\begin{align*}
    \min\left\{\atau_{t},\; \atau_{t-1} \right\}
    \leq \min\left\{\tau_{t},\; \tau_{t-1} \right\}
    = \tau_{t},
\end{align*}
and the lower bound,
\begin{align*}
    \min\left\{\atau_{t},\; \atau_{t-1} \right\}
    \geq \frac{1}{\alpha}\min\left\{\tau_{t},\; \tau_{t-1} \right\}
     = \frac{1}{\alpha}\:\tau_{t},
\end{align*}
which combined gives us $\frac{1}{\alpha}\:\tau_{t}
    \leq \min\left\{\atau_{t},\; \atau_{t-1} \right\}
\leq \tau_{t}$ as required by the definition of $\alpha$-accuracy.

\vspace{-0.05in}
\section{Proof of Thm.~\ref{thm:sequential-alg-main} and Thm.~\ref{thm:parallel-alg-main}}
\vspace{-0.05in}
From the discussion of Thm.~\ref{thm:parallel-alg-main}, we know that
running \sequentialalg is equivalent to running \parallelalg on a specific
(fully unbalanced) merge tree. Therefore, we just prove Thm.~\ref{thm:parallel-alg-main},
and invoke it on this tree to prove Thm.~\ref{thm:sequential-alg-main}.

We begin by describing more in detail some notation introduced in the main paper
and necessary for this proof.

\begin{figure}[t]
\begin{tabular}{m{0.33\textwidth}|m{0.33\textwidth}|m{0.33\textwidth}}
\subfigure[arbitrary tree]{\includegraphics[height=0.45\textwidth]{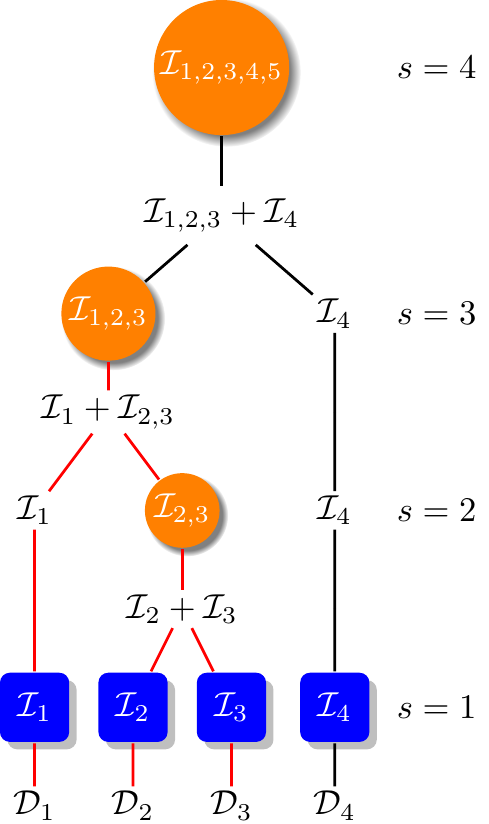} \label{fig:merge-trees-exp}}
& \subfigure[sequential tree]{\includegraphics[height=0.45\textwidth]{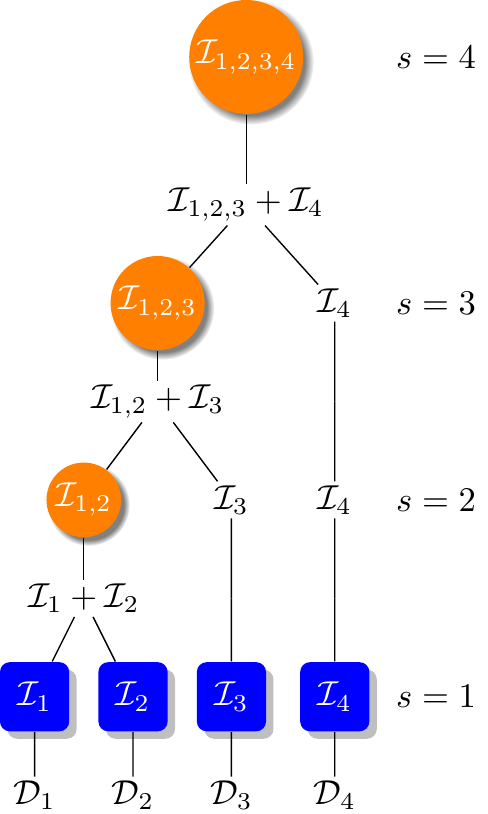}}
& \subfigure[minimum depth tree]{\includegraphics[height=0.45\textwidth]{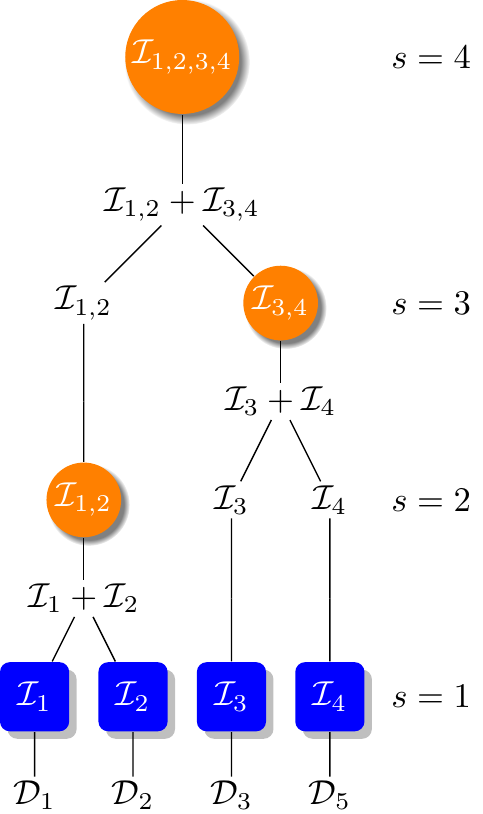}}
\end{tabular}
\caption{Merge trees for Algorithm \ref{alg:distributedalg}.}\label{fig:merge-trees}
\end{figure}

\textbf{Merge trees} We first formalize the random process induced by Alg.~\ref{alg:distributedalg}.

We partition $\dataset$ into $k$ disjoint sub-datasets $\dataset_i$
of size $n_i$, such that $\dataset = \cup_{i=1}^k\dataset_i$.
For each dataset $\dataset_i$, we construct an initial dictionary
$\coldict_{\{1,i\}} = \{(j,\wt{p}_{0,i} = 1, q_{0,i} = \wb{q}) : j \in \dataset_i\}$
by inserting all points from $\dataset_i$ into $\coldict_{\dataset_i}$ with
weight $\wt{p}_{0,i} = 1$ and number of copies $q_{0,i} = \wb{q}$.
It is easy to see that $\coldict_{\{1,i\}}$ is an $\epsilon$-accurate dictionary,
and we can split the data in small enough chunks to make sure that it can be
easily stored and manipulated in memory. Alternatively, if we want our
initial dictionaries to be small and cannot choose the size of
$\dataset_i$, we can run Alg.~\ref{alg:sequentialalg}
on $\dataset_i$ to generate $\coldict_{\{1,i\}}$, and the following proof
will remain valid.
Regardless of construction, the initial dictionaries $\coldict_{\{1,i\}}$ are included into the dictionary pool $\dictpool_{1}$.

At iteration $h$, the inner loop of Alg.~\ref{alg:distributedalg}
arbitrarily chooses  two dictionaries from
$\dictpool_h$ and merges them into a new dictionary.
Any arbitrary sequence of merges can be described by a full binary tree, i.e.,
a binary tree where each node is either a leaf or has exactly two children.
Figure \ref{fig:merge-trees} shows several different merge trees corresponding
to different choices for the order of the merges.
Note that starting from $k$ leaves, a full binary tree will always have exactly $k-1$ internal
nodes. Therefore, regardless of the structure of the merge tree, we can always
transform it into a tree of depth $k$, with all the initial dictionaries
$\coldict_{1,i}$ as leaves on its deepest layer.
After this transformation, we index the tree nodes using their height
(longest path from the node to a leaf, also defined as depth of the tree minus depth of the node),
where leaves have height 1 and the root has height $k$.
We can also see that at each layer, there is a single dictionary merge,
and the size of $\dictpool_h$ (number of dictionaries present at layer $h$)
is $|\dictpool_h| = k-h+1$.
Therefore, a node corresponding to a dictionary is uniquely identified with two indices $\{h,l\}$, where
$h$ is the height of the layer and $l\leq |\dictpool_h|$ is the index of the node in the
layer. For example, in Figure~\ref{fig:merge-trees-exp}, the node containing
$\coldict_{1,2,3}$ is indexed as $\{3,1\}$, and the highest
node containing $\coldict_4$ is indexed as $\{3,2\}$.

We also define the dataset
$\dataset_{\{h,l\}}$ as the union of all sub-datasets $\dataset_{l'}$ that are
 reachable from node $\{h,l\}$ as leaves.
For example, in Fig.~\ref{fig:merge-trees-exp}, dictionary $\coldict_{1,2,3}$
in node $\{3,1\}$ is constructed starting from all points in $\dataset_{\{3,1\}} = \dataset_1 \cup \dataset_2 \cup \dataset_3$,
where we highlight in red the descendant tree.
We now define $\kermatrix^h$ as the block diagonal kernel matrix
where each diagonal block $\kermatrix_{\{h,l\}}$ is constructed on
$\dataset_{\{h,l\}}$. Again, from Fig.~\ref{fig:merge-trees}, $\kermatrix^3$ is a $n \times n$ matrix with two blocks on the diagonal, a
first $(n_1 + n_2 + n_3) \times (n_1 + n_2 + n_3)$ block $\kermatrix_{3,1}$ constructed
on $\dataset_{\{3,1\}} = \dataset_1 \cup \dataset_2 \cup \dataset_3$, and a second $n_4 \times n_4$ block
$\kermatrix_{3,2}$ constructed on $\dataset_{\{3,2\}} = \dataset_4$. Similarly, we can 
adapt Def.~\ref{def:eps-acc-dict} to define $\:P^h$ as a block diagonal projection matrix,
where block $\:P_{\{h,l\}}$ is defined using $\kermatrix_{\{h,l\}}$,
and block diagonal $\wt{\:P}^h$,
where block $\wt{\:P}_{\{h,l\}}$ is defined using $\kermatrix_{\{h,l\}}$
and $\coldict_{\{h,l\}}$.

\textbf{The statement.}
Since $\:P^h - \wt{\:P}^h$ is block diagonal, we have that a bound on its largest eigenvalue implies an equal bound on each matrix on the diagonal, i.e.,
\begin{align*}
\normsmall{\:P^h - \wt{\:P}^h} = \max_{l} \normsmall{\:P_{\{h,l\}} - \wt{\:P}_{\{h,l\}}} \leq \varepsilon
\Rightarrow \normsmall{\:P_{\{h,l\}} - \wt{\:P}_{\{h,l\}}} \leq \varepsilon
\end{align*}
for all blocks $l$ on the diagonal, and since each block corresponds to
a dictionary $\coldict_{\{h,l\}}$, this means that if $\normsmall{\:P^h - \wt{\:P}^h} \leq \varepsilon$,
all dictionaries at layer $l$ are $\varepsilon$-accurate approximation
of their respective represented datasets.
Our goal is to show
\begin{align}\label{eq:distri-theorem-goal}
    &\probability\bigg(\exists  h\in\{1,\ldots,k\}: \normsmall{\:P^h - \wt{\:P}^h}_2 \geq \varepsilon \;\cup\; \max_{l=1,\dots,|\dictpool_h|}|\coldict_{\{h,l\}}| \geq 3\wb{q}\deff{\gamma}_{\{h,l\}}\bigg)\nonumber\\
     &= \probability\bigg(\exists  h\in\{1,\ldots,k\} : \underbrace{\left(\max_{l=1,\dots,|\dictpool_h|}\normsmall{\:P_{\{h,l\}} - \wt{\:P}_{\{h,l\}}}_2\right) \geq \varepsilon}_{A_h} \;\cup\; \underbrace{\left(\max_{l=1,\dots,|\dictpool_h|}|\coldict_{\{h,l\}}| \geq 3\wb{q}\deff{\gamma}_{\{h,l\}}\right)}_{B_h}\bigg) \leq \delta,
\end{align}
where event $A_h$ refers to the case when some dictionary $\coldict_{\{h,l\}}$ at an intermediate layer $h$ fails to accurately approximate $\kermatrix_{\{h,l\}}$ and event $B_h$ considers the case when the memory requirement is not met (i.e., too many points are kept in one of the dictionaries $\coldict_{\{h,l\}}$ at a certain layer $h$). We can conveniently decompose the previous joint (negative) event into two separate conditions as
\begin{align*}&\probability\bigg(\bigcup_{h = 1}^{k} A_h \cup B_h\bigg) =\probability\left(\left\{\bigcup_{h = 1}^{k} A_h\right\} \cup \left\{ \bigcup_{h = 1}^{k} B_h\right\}\right)
=\probability\left(\left\{\bigcup_{h = 1}^{k} A_h\right\} \right) +\probability\left(\left\{\bigcup_{h = 1}^{k} B_h\right\} \cap \left\{ \bigcup_{h = 1}^{k} A_h\right\}^{\complement}\right) \\
&=\probability\left(\left\{\bigcup_{h = 1}^{k} A_h\right\} \right) +\probability\left(\left\{\bigcup_{h = 1}^{k} B_h\right\} \cap \left\{ \bigcap_{h = 1}^{k} A_h^{\complement}\right\}\right)
=\probability\left(\left\{\bigcup_{h = 1}^{k} A_h\right\} \right) +\probability\left(\bigcup_{h = 1}^{k} \left\{B_h \cap \left\{ \bigcap_{h' = 1}^{k} A_{h'}^{\complement}\right\}\right\}\right).
\end{align*}
Applying this reformulation and a union bound we obtain
\begin{align}\label{eq:distri-theorem-goal}
    &\probability\bigg(\exists  h\in\{1,\ldots,k\}: \normsmall{\:P^h - \wt{\:P}^h}_2 \geq \varepsilon \;\cup\; \max_{l=1,\dots,|\dictpool_h|}|\coldict_{\{h,l\}}| \geq 3\wb{q}\deff{\gamma}_{\{h,l\}}\bigg)\nonumber\\
     &= \probability\left(\exists  h\in\{1,\ldots,k\} : \left(\max_{l=1,\dots,|\dictpool_h|}\normsmall{\:P_{\{h,l\}} - \wt{\:P}_{\{h,l\}}}_2\right) \geq \varepsilon\right)\nonumber\\
     &\quad\quad+ \probability\bigg(\exists  h\in\{1,\ldots,k\} : \max_{l=1,\dots,|\dictpool_h|}|\coldict_{\{h,l\}}| \geq 3\wb{q}\deff{\gamma}_{\{h,l\}}\cap \left\{\forall  h' \in \{1, \dots, h\} :  \normsmall{\:P^{h'} - \wt{\:P}^{h'}}_2 \leq \varepsilon\right\} \bigg)\nonumber\\
     &\leq \sum_{h = 1}^{k}\sum_{l=1}^{|\dictpool_h|} \probability\left(\normsmall{\:P_{\{h,l\}} - \wt{\:P}_{\{h,l\}}}_2 \geq \varepsilon \right)\nonumber\\
     &\quad\quad+\sum_{h = 1}^{k}\sum_{l=1}^{|\dictpool_h|}\probability\left(|\coldict_{\{h,l\}}| \geq 3\wb{q}\deff{\gamma}_{\{h,l\}} \cap \left\{\forall  h' \in \{1, \dots, h\} :  \normsmall{\:P^{h'} - \wt{\:P}^{h'}}_2 \leq \varepsilon\right\}\right) \leq \delta.
\end{align}
As discussed in Sect.~\ref{sec:sequential-alg}, the accuracy of the dictionary (first term in the previous bound) is guaranteed by the fact that given an $\varepsilon$-accurate dictionary we obtain RLS estimates which are at least a fraction of the true RLS, thus forcing the algorithm to sample each column \textit{enough}. On the other hand, the space complexity bound is achieved by exploiting the fact that RLS estimates are always upper-bounded by the true RLS, thus ensuring that Alg.~\ref{alg:distributedalg} does not oversample columns w.r.t.\ the sampling process following the exact RLS. 

In the reminder of the proof, we will show that both events happen with probability
smaller than $\delta/(2k^2)$. Since $|\dictpool_h| = k - h + 1$,
we have
\begin{align*}
\sum_{h = 1}^{k}\sum_{l=1}^{|\dictpool_h|}\frac{\delta}{2k^2} = \sum_{h = 1}^{k}(k - h + 1)\frac{\delta}{2k^2} = k(k+1)\frac{\delta}{4k^2} \leq k^2\frac{\delta}{2k^2} = \delta/2,
\end{align*}
and the union bound over all events is smaller than $\delta$.
The main advantage of splitting the failure probability as we did in Eq.~\ref{eq:distri-theorem-goal}
is that we can now analyze the processes that generated each $\:P_{\{h,l\}} - \wt{\:P}_{\{h,l\}}$
(and each dictionary $\coldict_{\{h,l\}}$) separately.
Focusing on a single node $\{h,l\}$ restricts our
problem on a well defined dataset $\dataset_{\{h,l\}}$,
where we can analyze the evolution of $\coldict_{\{h,l\}}$ sequentially.

\textbf{Challenges.}
Due to its sequential nature, the ``theoretical'' structure of the process generating $\:P_{\{h,l\}} -
\wt{\:P}_{\{h,l\}}$ is considerably complicated, where
each step in the sampling process (layer in the merge tree) is highly
correlated with the previous steps. \todod{fixed vs random tree} This prevents us from using concentration
inequalities for i.i.d.\ processes that are at the basis of the analysis of
uniform sampling~\cite{bach2013sharp} and the method proposed
by~\citet{alaoui2014fast}. As a result, we first show that the project error is
a martingale process. The main difficulty technical difficulty in analyzing how
the projection error evolves over iterations is that the projection matrices
change dimension every time a new point is processed. In fact, all of the matrices
$\:P_{\{h',l'\}} - \wt{\:P}_{\{h',l'\}}$ descending from $\:P_{\{h,l\}} -
\wt{\:P}_{\{h,l\}}$ have potentially a different size
since they are based on different
kernel matrices $\kermatrix_{\{h',l'\}}$. This requires a careful definition
of the martingale process to still use matrix concentration inequalities for
\textit{fixed-size} matrices (see Sect.~\ref{ss:setting.stage}). Another major
technical challenge is to control the variance of the martingale increments. In
fact, at each the projection error may increase by a quantity whose cumulative
variance can be arbitrarily large. As a result, a direct use of the Freedman
matrix inequality in Sect.~\ref{ssec:bounding_y} would not return an
accurate result. In order to provide a tighter bound on the total variance of
the martingale process of the projection error, we need to introduce an i.i.d.\
stochastically dominant process (Sect.~\ref{ssec:bounding_w} (step 2)), which
finally allows us to use an i.i.d.\ matrix concentration inequality to bound
the total variance (Sect.~\ref{ssec:bounding_w}-(step 5)). This finally leads
to the bound on the accuracy. The bound on the space complexity (Sect.~\ref{ssec:space-complexity})follows similar
(but simpler) steps.

\subsection{Bounding the projection error $\normsmall{\mathbf{P}_{\{h,l\}} - \wt{\mathbf{P}}_{\{h,l\}}}$}\label{ss:setting.stage}

\textbf{The sequential process.} Thanks to the union bound in Eq.~\ref{eq:distri-theorem-goal},
instead of having to consider the whole merge tree followed by
Alg.~\ref{alg:distributedalg}, we can focus on each individual node
$\{h,l\}$ and study the sequential process that generated its dictionary $\coldict_{\{h,l\}}$.
We will now map more clearly the actions
taken by Alg.~\ref{alg:distributedalg} to the process that
generated $\:P_{\{h,l\}} - \wt{\:P}_{\{h,l\}}$.
We begin by focusing on $\wt{\:P}_{\{h,l\}}$, which is a random matrix defined starting from
the fixed kernel matrix $\kermatrix_{\{h,l\}}$ and the random dictionary $\coldict_{\{h,l\}}$,
where the randomness influences both which points are included in $\coldict_{\{h,l\}}$,
and the weight with which they are added.
\todod{fixed vs random tree}

Note that since the merge tree is decided in advance, the dataset $\dataset_{\{h,l\}}$ is
not a random object, and is fixed for the whole process.
Consider now a point
$i \in \dataset_{\{h,l\}}$. Since the starting datasets in the leaves are disjoint,
there is a single path in the tree, with length $h$, from the leaves to $\{h,l\}$.
This means that for all $s<h$, we can properly define a unique $\wt{p}_{s,i}$
and $q_{s,i}$ associated with that point. More in detail, if at layer $s$ point $i$ is present
in $\dataset_{\{s,l'\}}$, it means that either (1) Alg.~\ref{alg:distributedalg}
used $\coldict_{\{s,l'\}}$ to compute $\wt{p}_{s,i}$, and $\wt{p}_{s,i}$ to compute $q_{s,i}$,
or (2) at layer $h$ Alg.~\ref{alg:distributedalg} did not have any merge scheduled
for point $i$, and we simply propagate $\wt{p}_{s,i} = \wt{p}_{s-1,i}$ and $q_{s,i} = q_{s-1,i}$.
Consistently with the algorithm, we initialize
$\wt{p}_{0,i} = 1$ and $q_{0,i}=\wb{q}$.

Denote $\nhl_{\{h,l\}} = |\dataset_{\{h,l\}}|$ so that we can use index $i \in [\nhl_{\{h,l\}}]$ to index all points in
$\dataset_{\{h,l\}}$. Given the symmetric matrix $\concmat = (\kermatrix_{\{h,l\}} + \gamma\:I)^{-1/2}\kermatrix_{\{h,l\}}^{1/2}$ with its $i$-th column $\concvec_{i} = (\kermatrix_{\{h,l\}} + \gamma\:I)^{-1/2}\kermatrix_{\{h,l\}}^{1/2} \:e_{\nhl_{\{h,l\}},i}$, we can rewrite the projection matrix as that
$\:P_{\{h,l\}} = \concmat\concmat^\transp = \sum_{i=1}^{\nhl_{\{h,l\}}} \concvec_i\concvec_i^\transp$.
Note that
\begin{align*}
\normsmall{\concvec_i\concvec_i^\transp}
= \concvec_i^\transp\concvec_i
= \:e_{\nhl_{\{h,l\}},i}^\transp\concmat^\transp\concmat\:e_{\nhl_{\{h,l\}},i}
= \:e_{\nhl_{\{h,l\}},i}^\transp\concmat\concmat^\transp\:e_{\nhl_{\{h,l\}},i}
= \:e_{\nhl_{\{h,l\}},i}^\transp\:P_{\{h,l\}}\:e_{\nhl_{\{h,l\}},i} = \tau_{\dataset_{\{h,l\}},i},
    \end{align*}
or, in other words, the norm $\normsmall{\concvec_i\concvec_i^\transp}$ is equal to the RLS of
the $i$-th sample w.r.t.\@ to dataset $\dataset_{\{h,l\}}$. Note that since $i$ is present only in node $l$ on layer $h$, its
RLS is uniquely defined w.r.t.\@ $\dataset_{\{h,l\}}$ and can be shortened as
$\tau_{h,i} = \tau_{\dataset_{\{h,l\}},i}$.
Using $\concvec_i$, we can also introduce the random
matrix $\wt{\:P}^{\{h,l\}}_{s}$ as
\begin{align*}
\wt{\:P}^{\{h,l\}}_{s} = \sum_{i=1}^{\nhl_{\{h,l\}}} \frac{q_{s,i}}{\wb{q}\wt{p}_{s,i}}\concvec_i\concvec_i^\transp
= \sum_{i=1}^{\nhl_{\{h,l\}}} \sum_{j=1}^{\wb{q}}\frac{z_{s,i,j}}{\wb{q}\wt{p}_{s,i}}\concvec_i\concvec_i^\transp.
\end{align*}
where $z_{s,i,j}$ are $\{0,1\}$ r.v.~such that $q_{s,i} = \sum_{j=1}^{\wb{q}}z_{s,i,j}$,
or in other words $z_{s,i,j}$ are the Bernoulli random variables
that compose the Binomial $q_{s,i}$ associated with point $i$, with $j$ indexing each individual copy
of the point.
Note that when $s = h$, we have that $\wt{\:P}^{\{h,l\}}_h = \wt{\:P}_{\{h,l\}}$ and we recover
the definition of the approximate projection matrix from Alg.~\ref{alg:distributedalg}.
But, for
a general $s \neq h$ $\wt{\:P}^{\{h,l\}}_s$ does not have a direct interpretation in the context
of Alg.~\ref{alg:distributedalg}. It combines the vectors $\concvec_i$,
which are defined using $\kermatrix_{\{h,l\}}$ at layer $h$, with the weights
$\wt{p}_{s,i}$ computed by Alg.~\ref{alg:distributedalg} across multiple nodes at layer $s$, which are potentially stored in
different machines that cannot communicate. Nonetheless, $\wt{\:P}^{\{h,l\}}_s$ is a useful tool
to analyze Alg.~\ref{alg:distributedalg}.

\todod{fixed vs random merge tree}
Taking into account that we are now considering a specific node $\{h,l\}$,
we can drop the index from the
dataset $\dataset_{\{h,l\}} = \dataset$, RLS $\tau_{\dataset_{\{h,l\}},i} = \tau_{h,i}$,
and size $\nhl_{\{h,l\}} = \nhl$.
Using this shorter notation, we can reformulate our objective as bounding $\normsmall{\:P_{\{h,l\}} - \wt{\:P}_{\{h,l\}}}_2 =\normsmall{\:P_{\{h,l\}} - \wt{\:P}^{\{h,l\}}_{h}}_2$, and
reformulate the process as a sequence of matrices $\{\:Y_s\}_{s=1}^h$ defined as
\begin{align*}
\:Y_{s} = \:P_{\{h,l\}} - \wt{\:P}^{\{h,l\}}_s = \frac{1}{\wb{q}}\sum_{i=1}^{\nhl} \sum_{j=1}^{\wb{q}}\left(1 - \frac{z_{s,i,j}}{\wt{p}_{s,i}}\right)\concvec_i\concvec_i^\transp,
\end{align*}
where $\:Y_{h} = \:P_{\{h,l\}} - \wt{\:P}^{\{h,l\}}_{h} = \:P_{\{h,l\}} - \wt{\:P}_{\{h,l\}}$,
and $\:Y_{1} = \:P_{\{h,l\}} - \wt{\:P}^{\{h,l\}}_{0} = \:0$ since $\wt{p}_{0,i} = 1$ and $q_{0,i} = \wb{q}$.

\subsection{Bounding $\mathbf{Y}_h$}\label{ssec:bounding_y}
We transformed the problem of bounding $\normsmall{\:P_{\{h,l\}} - \wt{\:P}_{\{h,l\}}}$
into the problem of bounding $\:Y_h$, which we modeled as a random matrix process,
connected to Alg.~\ref{alg:distributedalg} by the fact that both algorithm
and random process $\:Y_h$ make use of the same weight $\wt{p}_{s,i}$
and multiplicities $q_{s,i}$.
There are two main issues in analyzing the process $\{\:Y_s\}_{s=1}^h$:
\todod{rewrite item (1)}
\begin{itemize}
\item[(1)] On the one hand, the overall algorithm may fail in generating an accurate
dictionary at some intermediate iteration, and yet return an accurate dictionary at
the end. On the other hand, whenever one of the intermediate $\normsmall{\:Y_{s}}$ is larger than~$\varepsilon$ (inaccurate intermediate dictionary) we lose our guarantees for $\wt{p}_{s,i}$ for the whole process. This is because an inaccurate
$\wt{p}_{s,i}$ that underestimates too much $p_{s.i}$ will influence all successive $\wt{p}_{s',i}$ through the minimum $\wt{p}_{s',i} = \min\{\atau_{s',i}\wt{p}_{s'-1,i}\}$.
To solve this, we consider an alternative (more pessimistic) process which is
``frozen'' as soon as it constructs an inaccurate dictionary.
Freezing the probabilities at the first error gives us a process
that fails more often, but that provides strong guarantees up to the moment
of failure.
\item[(2)] While $\normsmall{\:Y_{h}} = \normsmall{\:P_{\{h,l\}} - \wt{\:P}^{\{h,l\}}_{h}} \leq \varepsilon$
guarantees that $\coldict_{\{h,l\}}$ is an $\varepsilon$-accurate dictionary of
$\kermatrix_{\{h,l\}}$, knowing $\normsmall{\:Y_{s}} = \normsmall{\:P_{\{h,l\}} - \wt{\:P}^{\{h,l\}}_{s}} \leq \varepsilon$
for $s < h$ does not guarantee that all descendant $\coldict_{\{s,l'\}}$ are $\epsilon$-accurate \wrt their $\kermatrix_{\{s,l'\}}$.
Nonetheless, we will show that $\normsmall{\:Y_{s}} \leq \varepsilon$
is enough to guarantee that the intermediate estimate $\wt{p}_{s,i}$
computed in Alg.~\ref{alg:distributedalg}, and used as weights in $\wt{\:P}^{\{h,l\}}_{s}$, are
never too small.
\end{itemize}

\textbf{The frozen process.}
We will now replace the process $\:Y_s$ with an alternative process $\wb{\:Y}_{s}$
defined as
\begin{align*}
    \wb{\:Y}_s = 
    \:Y_{s-1} \indfunc\left\{\normsmall{\wb{\:Y}_{s-1}} \leq \varepsilon\right\}
    +\wb{\:Y}_{s-1}\indfunc\left\{\normsmall{\wb{\:Y}_{s-1}} \geq \varepsilon\right\}.
\end{align*}
This process starts from $\wb{\:Y}_0 = \:Y_0 = \:0$, and
is identical to $\:Y_s$ until a step $\wb{s}$ where for the first
time $\normsmall{\:Y_{\wb{s}}} \leq \varepsilon$ and
$\normsmall{\:Y_{\wb{s}+1}} \geq \varepsilon$. After this failure happen
the process $\wb{\:Y}_s$ is ``frozen'' at $\wb{s}$ and $\wb{\:Y}_{s} = \:Y_{\wb{s}+1}$ for all $\wb{s} +1\leq s \leq h$.
Consequently, if any of the intermediate elements
of the sequence violates the condition $\normsmall{\:Y_s}~\leq~\varepsilon$,
the last element will violate it too. For the rest, $\wb{\:Y}_s$ behaves
exactly like $\:Y_s$.
Therefore,
\begin{align*}
&\probability\left( \normsmall{\:Y_h} \geq \varepsilon\right)
\leq \probability\Big( \normsmall{\wb{\:Y}_h} \geq \varepsilon\Big),
\end{align*}
and if we can bound $\probability\Big( \normsmall{\wb{\:Y}_h} \geq \varepsilon\Big)$
we will have a bound for the failure probability of Alg.~\ref{alg:distributedalg},
even though after ``freezing'' the process $\wb{\:Y}_h$ does not make the same choices as
the algorithm.

We will see now how to construct the process $\wb{\:Y}_s$ starting from
$z_{s,i,j}$ and $\wt{p}_{s,i,j}$.
We recursively define the indicator ($\{0,1\}$) random variable~$\wb{z}_{s,i,j}$ as
\begin{align*}
    \wb{z}_{s,i,j} = \indfunc\left\{ u_{s,i,j} \leq \frac{\wb{p}_{s,i,j}}{\wb{p}_{s-1,i,j}}\right\} \wb{z}_{s-1,i,j},
\end{align*}
where $u_{s,i,j} \sim \mathcal{U}(0,1)$ is a $[0,1]$ uniform random variable and $\wb{p}_{s,i,j}$ is defined as
\begin{align*}
\wb{p}_{s,i,j} =\wt{p}_{s,i}
    \indfunc\left\{\normsmall{\wb{\:Y}_{s-1}} \leq \varepsilon \cap z_{s-1,i,j} = 1\right\}
    +\wb{p}_{s-1,i,j}\indfunc\left\{\normsmall{\wb{\:Y}_{s-1}} \geq \varepsilon \cup z_{s-1,i,j} = 0\right\}.
\end{align*}

This definition of the process satisfies the freezing condition, since if
$\normsmall{\:Y_{\wb{s}+1}} \geq \varepsilon$ (we have a failure at step
$\wb{s}$), for all $s' \geq \wb{s}+1$ we have
$\wb{z}_{s',i,j} = \wb{z}_{\wb{s}+1,i,j}$ with probability 1
($\wb{p}_{\wb{s}+1,i,j}/\wb{p}_{\wb{s},i,j} = \wb{p}_{\wb{s},i,j}/\wb{p}_{\wb{s},i,j} = 1$),
and the weights $1/(\wb{q}\wb{p}_{\wb{s}+1,i,j}) = 1/(\wb{q}\wb{p}_{\wb{s},i,j})$
never change.

Introducing a per-copy weight $\wb{p}_{s,i,j}$ and enforcing that $\wb{p}_{s+1,i,j} = \wb{p}_{s,i,j}$
when $z_{s,i,j} = 0$ avoids subtle inconsistencies in the formulation.
In particular, not doing so would semantically correspond to reweighting
dropped copies. Although this does not directly affect $\:Y_s$ (since the ratio $z_{s,i,j}/\wt{p}_{s,i}$
is zero for dropped copies), and therefore
the relationship $\probability\left( \normsmall{\:Y_h} \geq \varepsilon\right)
\leq \probability\Big( \normsmall{\wb{\:Y}_h} \geq \varepsilon\Big)$
still holds. We will see later how maintaining consistency helps us bound the second moment of our process.

We can now arrange the indices $s$, $i$, and $j$ into a linear index $r=s$ in the
range $[1,\dots,\nhl^2\wb{q}]$, obtained as $r~=~\{s,i,j\}~=~(s-1)\nhl\wb{q}~+~(i-1)\wb{q}~+~j$.
We also define the difference matrix as 
\begin{align*}
\wb{\:X}_{\{s,i,j\}} = \frac{1}{\wb{q}}\left(\frac{z_{s-1,i,j}}{\wb{p}_{s-1,i,j}} - \frac{z_{s,i,j}}{\wb{p}_{s,i,j}}\right)\concvec_i\concvec_i^\transp,
    \end{align*}
which allows us to write the cumulative matrix as
$\wb{\:Y}_{\{s,i,j\}}~=~\sum_{r=1}^{\{s,i,j\}}~\wb{\:X}_{\{s,i,j\}}$
where the checkpoints $\{s,\nhl,\wb{q}\}$ correspond to $\wb{\:Y}_s$,
\begin{align*}
\wb{\:Y}_{\{s,\nhl,\wb{q}\}} = \wb{\:Y}_s = 
    \frac{1}{\wb{q}}\sum_{i=1}^{\nhl}\sum_{j=1}^{\wb{q}}\left(1-\frac{z_{s,i,j}}{\wb{p}_{s,i,j}}\right)\concvec_i\concvec_i^\transp.
\end{align*}
Let $\F_s$ be the filtration containing all the
realizations of the uniform random variables $u_{s,i,j}$ up to the step~$s$, that is $\F_s = \{ u_{s',i',j'}, \forall\{s',i',j'\}
\leq s\}$. Again, we notice that $\F_s$ defines the state of
the algorithm after completing iteration $s$ because, unless a ``freezing'' happened,
Alg.~\ref{alg:distributedalg} and $\wb{\:Y}_s$ flip coins with the same probability,
and generate the same dictionaries.
Since $\atau_{s,i}$ and
$\wb{p}_{s,i,j}$ are computed at the beginning of iteration $s$ using the
dictionary $\coldict_{\{s,l'\}}$ (for some $l'$ unique at layer $s$), they are fully determined by $\F_{s-1}$.
Furthermore, since $\F_{s-1}$ also defines the values of all indicator
variables~$\wb{z}_{s',i,j}$ up to $\wb{z}_{s-1,i,j}$ for any $i$ and $j$, we have that
all the Bernoulli variables $\wb{z}_{s,i,j}$ at iteration $s$ are conditionally
independent given $\F_{s-1}$. In other words, we have that for any
$i'$, and $j'$ such that $\{s,1,1\} \leq \{s,i',j'\} <s$ the following
random variables are equal in distribution,
\begin{align}\label{eq:distro.z}
\wb{z}_{s,i,j} \big| \F_{\{s,i',j'\}} = \wb{z}_{s,i,j} \big| \F_{\{s-1,\nhl,\wb{q}\}} \sim \mathcal{B}\Big( \frac{\wb{p}_{s,i,j}}{\wb{p}_{s-1,i,j}} \Big)\CommaBin
\end{align}
and for any $i'$, and $j'$ such that $\{s,1,1\} \leq \{s,i',j'\} \leq \{s,\nhl,\wb{q}\} $ and $s \neq \{s,i',j'\}$ we have the independence
\begin{align}\label{eq:distro.z.indep}
\wb{z}_{s,i,j} \big| \F_{\{s-1,\nhl,\wb{q}\}} \perp \wb{z}_{s,i',j'} \big| \F_{\{s-1,\nhl,\wb{q}\}}.
\end{align}
While knowing that $\normsmall{\:Y_s} \leq \varepsilon$ is not sufficient to provide guarantees for the approximate probabilities
$\wt{p}_{s,i}$, we can show that it is enough to prove that the frozen probabilities $\wb{p}_{s,i,j}$ are
never too small.
\begin{lemma}\label{lem:wbp-always-good}
    Let $\alpha = (1+3\varepsilon)/(1-\epsilon)$ and $\wb{p}_{s,i,j}$ be the sequence of probabilities generated
    by the freezing process. Then for any $s,i,$ and $j$, we
    have $\wb{p}_{s,i,j} \geq p_{h,i}/\alpha = \tau_{h,i}/\alpha$.
\end{lemma}
\begin{proof}[Proof of Lemma \ref{lem:wbp-always-good}]
Let $\wb{s}$ be the step where the process freezes ($\wb{s} = h$ if it does not
freeze), or, in other words, $\normsmall{\:Y_{\wb{s}}} < \varepsilon$
and $\normsmall{\:Y_{\wb{s}+1}} \geq \varepsilon$.
From the definition of $\wb{p}_{s,i,j}$, we have
that 
\begin{align*}
&\wb{p}_{s,i,j} \geq \wb{p}_{\wb{s},i} = \wt{p}_{\wb{s},i}
= \max\left\{\min\left\{\atau_{\wb{s},i},\; \wt{p}_{\wb{s}-1,i} \right\},\; \wt{p}_{\wb{s}-1,i}/2\right\}\\
&\geq \min\left\{\atau_{\wb{s},i},\; \wt{p}_{\wb{s}-1,i} \right\}
=\min\left\{\atau_{\wb{s},i},\; \wt{p}_{\wb{s}-2,i}\right\}
=\min\left\{\atau_{\wb{s},i},\; \wt{p}_{\wb{s}-3,i}\right\} \ldots
=\min\left\{\atau_{\wb{s},i},\; \wt{p}_{0,i}\right\}
=\atau_{\wb{s},i},
\end{align*}
and therefore $\wb{p}_{s,i,j} \geq \atau_{\wb{s},i}$.
Now let $\{\wb{s},l'\}$ be the node where $\atau_{\wb{s},i}$ was computed.
We will again drop the $\{h,l\}$ index from $\dataset_{\{h,l\}}$, and simply
refer to it as $\dataset$. Similarly, we will refer with
$\dataset_i$ to $\dataset_{\{\wb{s},l'\}}$ (as in, the dataset used to
compute $\atau_{\wb{s},i}$), and with $\wb{\dataset}_{i}$ to the samples
in $\dataset$ not contained in $\dataset_{i}$ (complement of $\dataset_i$).
\todod{here we use a square $\bm{S}$, have to recheck if it is introduced in main text or appendix}
Define $\:A$ as the $|\dataset| \times |\dataset_i|$ matrix that contains
the columns of $\selmatrix_{\wb{s}}$ related to points in $\dataset_i$,
and similarly define $\:B$ as the $|\dataset| \times |\wb{\dataset}_i|$ matrix that contains
the columns of $\selmatrix_{\wb{s}}$ related to points in $\wb{\dataset}_i$,
where $\selmatrix_{\wb{s}}$ can be reconstructed by interleaving columns of $\:A$
and $\:B$.
From its definition in Eq.\ref{eq:rls-estimator-merge},
we know that $\atau_{s,i}$ is computed by Alg.~\ref{alg:distributedalg} as
\begin{align*}
\atau_{s,i} = (1-\varepsilon)\:\phi_i^\transp\left(\featkermatrix_{\dataset}\:A\:A^\transp\featkermatrix_{\dataset}^\transp + (1+\varepsilon)\gamma \:I_D\right)^{-1}\:\phi_{i},
\end{align*}
using only the points in $\:A$ that are available at node $\{\wb{s},l'\}$.
From Lem.\ref{lem:concentration-equivalence}
we know that 
\begin{align*}
\normsmall{\:Y_{\wb{s}}}
=\norm{\:P_{\{h,l\}} - \wt{\:P}^{\{h,l\}}_{\wb{s}}}{2}
&= \norm{(\featkermatrix_{\dataset}\featkermatrix_{\dataset}^\transp + \gamma\:I_D)^{-1/2}(\featkermatrix_{\dataset}\featkermatrix_{\dataset}^\transp - \featkermatrix_{\dataset}\:S_{\wb{s}}\:S_{\wb{s}}^\transp\featkermatrix_{\dataset}^\transp)(\featkermatrix_{\dataset}\featkermatrix_{\dataset}^\transp + \gamma\:I_D)^{-1/2}}{2}\\
&= \norm{(\featkermatrix_{\dataset}\featkermatrix_{\dataset}^\transp + \gamma\:I_D)^{-1/2}(\featkermatrix_{\dataset}\featkermatrix_{\dataset}^\transp - \featkermatrix_{\dataset}\:A\:A^\transp\featkermatrix_{\dataset}^\transp - \featkermatrix_{\dataset}\:B\:B^\transp\featkermatrix_{\dataset}^\transp)(\featkermatrix_{\dataset}\featkermatrix_{\dataset}^\transp + \gamma\:I_D)^{-1/2}}{2} \leq \varepsilon
\end{align*}
and we know that this implies
\begin{align*}
    \featkermatrix_{\dataset}\:A\:A^\transp\featkermatrix_{\dataset}^\transp \preceq \featkermatrix_{\dataset}\featkermatrix_{\dataset}^\transp + \varepsilon(\featkermatrix_{\dataset}\featkermatrix_{\dataset}^\transp + \gamma\:I_D) - \featkermatrix_{\dataset}\:B\:B^\transp\featkermatrix_{\dataset}^\transp \preceq \featkermatrix_{\dataset}\featkermatrix_{\dataset}^\transp + \varepsilon(\featkermatrix_{\dataset}\featkermatrix_{\dataset}^\transp + \gamma\:I_D). 
\end{align*}
Plugging it in the initial definition,
\begin{align*}
\atau_{s,i} &= (1-\varepsilon)\:\phi_i^\transp\left(\featkermatrix_{\dataset}\:A\:A^\transp\featkermatrix_{\dataset}^\transp + (1+\varepsilon)\gamma \:I_D\right)^{-1}\:\phi_{i}\\
&\geq (1 - \vareps)\:\phi_i^\transp(\featkermatrix_{\dataset}\featkermatrix_{\dataset}^\transp + \varepsilon(\featkermatrix_{\dataset}\featkermatrix_{\dataset}^\transp + \gamma\:I_D) + (1+\varepsilon)\gamma \:I_D)^{-1}\:\phi_i\\
&= (1 - \vareps)\frac{1}{1+2\varepsilon}\:\phi_i^\transp(\featkermatrix_{\dataset}\featkermatrix_{\dataset}^\transp + \gamma \:I_D)^{-1}\:\phi_i
\geq \frac{1-\varepsilon}{1+2\varepsilon}\tau_{h,i}
\geq \tau_{h,i}/\alpha.
\end{align*}
\end{proof}
This result is weaker then Lemmas~\ref{lem:fast-rls} and \ref{lem:fast-rls-merge},
since we do not
provide an upper bound, and only show that $\wt{p}_{s,i} \geq p_{h,i}/\alpha$
and not $\wt{p}_{s,i} \geq p_{s,i}/\alpha$, but it guarantees that the
probabilities used at any intermediate layer $s$ are bigger than a fraction
$1/\alpha$ of the exact probabilities that we would use at layer $h$, which
will suffice for our purpose.

We now proceed by studying the process
$\{\wb{\:Y}_s\}_{s=1}^h$ and showing that it is a bounded martingale.
In order to show that $\wb{\:Y}_s$ is a martingale, it is sufficient to verify the following (equivalent) conditions
\begin{align*}
    \expectedvalue\left[\wb{\:Y}_s \condbar \F_{s-1}\right]
    =\wb{\:Y}_{s-1} \enspace \Leftrightarrow \enspace
    \expectedvalue\left[\wb{\:X}_{\{s,i,j\}} \condbar \F_{s-1}\right] = \:0.
\end{align*}
We begin by inspecting the conditional random variable $\wb{\:X}_{\{s,i,j\}} | \F_{s-1}$. Given the definition of $\wb{\:X}_{\{s,i,j\}}$, the conditioning on $\F_{s-1}$ determines the values of $\wb{z}_{s-1,i,j}$ and the approximate probabilities $\wb{p}_{s-1,i,j}$ and $\wb{p}_{s,i,j}$. In fact, remember that these quantities are fully determined by the realizations in $\F_{s-1}$ which are contained in $\F_{s-1}$. As a result, the only stochastic quantity in $\wb{\:X}_{\{s,i,j\}}$ is the variable $\wb{z}_{s,i,j}$. Specifically, if $\normsmall{\wb{\:Y}_{s-1}} \geq \varepsilon$,
then we have $\wb{p}_{s,i,j} = \wb{p}_{s-1,i,j}$ and $\wb{z}_{s,i,j} = \wb{z}_{s-1,i,j}$
(the process is stopped), and the martingale requirement
$ \expectedvalue\left[\wb{\:X}_{\{s,i,j\}} \condbar \F_{s-1}\right] = \:0$
is trivially satisfied.
On the other hand, if $\normsmall{\wb{\:Y}_{s-1}} \leq \varepsilon$ we have
\begin{align*}
\expectedvalue_{u_{s,i,j}}&\left[\frac{1}{\wb{q}} \left(\frac{\wb{z}_{s-1,i,j}}{\wb{p}_{s-1,i,j}} - \frac{\wb{z}_{s,i,j}}{\wb{p}_{s,i,j}}\right)\concvec_i\concvec_i^\transp \condbar \F_{s-1}\right]\\
 &= \frac{1}{\wb{q}} \left( \frac{\wb{z}_{s-1,i,j}}{\wb{p}_{s-1,i,j}} -\frac{\wb{z}_{s-1,i,j}}{\wb{p}_{s,i,j}}\expectedvalue\left[\indfunc\left\{ u_{s,i,j} \leq \frac{\wb{p}_{s,i,j}}{\wb{p}_{s-1,i,j}}\right\}\condbar\F_{s-1}\right]\right)\concvec_i\concvec_i^\transp\\
&= \frac{1}{\wb{q}}
    \left(\frac{\wb{z}_{s-1,i,j}}{\wb{p}_{s-1,i,j}} - \frac{\wb{z}_{s-1,i,j}}{\wb{p}_{s,i,j}}\frac{\wb{p}_{s,i,j}}{\wb{p}_{s-1,i,j}}
\right)\concvec_i\concvec_i^\transp
= \:0,
\end{align*}
where we use the recursive definition of $\wb{z}_{s,i,j}$ and the fact that $u_{s,i,j}$ is a uniform random variable in $[0,1]$. This proves that $\wb{\:Y}_s$ is indeed a martingale.
We now compute an upper-bound $R$ on the norm of the values of the difference process as
\begin{align*}
&\normsmall{\wb{\:X}_{\{s,i,j\}}}
= \frac{1}{\wb{q}} \left|\left(\frac{\wb{z}_{s-1,i,j}}{\wb{p}_{s-1,i,j}} - \frac{\wb{z}_{s,i,j}}{\wb{p}_{s,i,j}}\right)\right|\normsmall{\concvec_i\concvec_i^\transp}
\leq \frac{1}{\wb{q}} \frac{1}{\wb{p}_{s,i,j}} \normsmall{\concvec_i\concvec_i^\transp}
= \frac{1}{\wb{q}} \frac{1}{\wb{p}_{s,i,j}} \tau_{h,i}
\leq \frac{1}{\wb{q}} \frac{\alpha}{\tau_{h,i}}\tau_{h,i}
=\frac{\alpha}{\wb{q}} \eqdef R,
\end{align*}
where we used Lemma~\ref{lem:wbp-always-good} to bound
$\wb{p}_{s,i,j} \leq \tau_{h,i}/\alpha$.
If instead, $\normsmall{\wb{\:Y}_{s-1}} \geq \varepsilon$,
 the process is stopped and 
$\normsmall{\wb{\:X}_s} = \normsmall{\:0} = 0 \leq R$.

We are now ready to use a Freedman matrix inequality from \cite{tropp2011freedman} to bound the norm of $\wb{\:Y}$.

\begin{proposition}[\citet{tropp2011freedman},~Theorem~1.2]\label{prop:matrix-freedman}
Consider a matrix martingale $\{ \:Y_k : k = 0, 1, 2, \dots \}$ whose values are self-adjoint matrices with dimension $d$, and let $\{ \:X_k : k = 1, 2, 3, \dots \}$ be the difference sequence.  Assume that the difference sequence is uniformly bounded in the sense that
\begin{align*}
 \normsmall{\:X_k}_2  \leq R
\quad\text{almost surely}
\quad\text{for $k = 1, 2, 3, \dots$}.
\end{align*}
Define the predictable quadratic variation process of the martingale as
\begin{align*}
\:{W}_k \eqdef \sum_{j=1}^k \expectedvalue \left[ \:X_j^2 \condbar \{\:X_{s}\}_{s=0}^{j-1} \right],
\quad\text{for $k = 1, 2, 3, \dots$}.
\end{align*}
Then, for all $\varepsilon \geq 0$ and $\sigma^2 > 0$,
\begin{align*}
\probability\left( \exists k \geq 0 : \normsmall{\:Y_k}_2 \geq \varepsilon \ \cap\ 
        \normsmall{ \:W_{k} } \leq \sigma^2 \right)
	\leq 2d \cdot \exp \left\{ - \frac{ \varepsilon^2/2 }{\sigma^2 + R\varepsilon/3} \right\}\cdot
\end{align*}
\end{proposition}

In order to use the previous inequality, we develop the probability of error for any fixed $h$ as
\begin{align*}
\probability\left( \normsmall{\:Y_h} \geq \varepsilon\right)
\leq \probability\left( \normsmall{\wb{\:Y}_h} \geq \varepsilon\right)
&= \probability\left( \normsmall{\wb{\:Y}_h} \geq \varepsilon \cap \normsmall{\:W_h} \leq \sigma^2\right)
+ \probability\left( \normsmall{\wb{\:Y}_h} \geq \varepsilon \cap \normsmall{\:W_h} \geq \sigma^2\right)\\
&\leq \underbrace{\probability\left( \normsmall{\wb{\:Y}_h} \geq \varepsilon \cap \normsmall{\:W_h} \leq \sigma^2\right)}_{\mbox{(a)}}
    + \underbrace{\probability\left( \normsmall{\:W_h} \geq \sigma^2\right)}_{\mbox{(b)}}.
\end{align*}
Using the bound on  $\normsmall{\wb{\:X}_{\{s,i,j\}}}_2$, we can directly apply Proposition~\ref{prop:matrix-freedman} to bound $\mbox{(a)}$ for any fixed $\sigma^2$.
To bound the part $\mbox{(b)}$, we use the following lemma, proved later in Sec.~\ref{ssec:bounding_w}.

\begin{lemma}[Low probability of the large norm of the predictable quadratic variation process]\label{lem:prob-dominance}
    \begin{align*}
        \probability\left( \normsmall{\:W_h} \geq \frac{6\alpha}{\wb{q}}\right)
        \leq n \cdot \exp \left\{ - 2\frac{\wb{q}}{\alpha} \right\}
    \end{align*}
\end{lemma}

Combining Prop.~\ref{prop:matrix-freedman} with $\sigma^2 = 6\alpha/\wb{q}$, Lem~\ref{lem:prob-dominance}, the fact that $2\varepsilon/3 \leq 1$ and the value used by Alg.~\ref{alg:distributedalg} $\wb{q} = 39\alpha\log(2n/\delta)/\varepsilon^2$ we obtain
\begin{align*}
\probability\left( \normsmall{\:P_{\{h,l\}} - \wt{\:P}_{\{h,l\}}}_2 \geq \varepsilon\right)
& = \probability\left( \normsmall{\:Y_h} \geq \varepsilon\right)
\leq \probability\left( \normsmall{\wb{\:Y}_h} \geq \varepsilon \cap \normsmall{\:W_h} \leq \sigma^2\right)
+ \probability\left( \normsmall{\:W_h} \geq \sigma^2\right)\\
&\leq 2\nu \cdot \exp\left\{-\frac{\varepsilon^2\wb{q}}{\alpha}\left(\frac{1}{12+2\varepsilon/3}\right)\right\}
+ n \cdot \exp\left\{- 2\frac{\wb{q}}{\alpha}\right\}\\
&\leq 3n \cdot \exp \left\{ - \frac{\varepsilon^2}{13\alpha}\wb{q}\right\}
= 3n \cdot \exp \left\{ - 3\log\left(\frac{2n}{\delta}\right)\right\}\\
&= 3n \cdot \exp \left\{ - \log\left(\left(\frac{2n}{\delta}\right)^3\right)\right\}
= 3n \frac{\delta^3}{8n^3} \leq \frac{\delta}{2n^2}\cdot
\end{align*}
This, combined with the fact that $k \leq n$ since at most we can split our
dataset in $n$ parts, concludes this part of the proof.

\vspace{-0.05in}
\subsection{Proof of Lemma \ref{lem:prob-dominance} (bound on predictable quadratic variation)}\label{ssec:bounding_w}
\vspace{-0.05in}

\newcounter{cnt-lem-quad-variation}
\setcounter{cnt-lem-quad-variation}{1}

\textbf{Step \arabic{cnt-lem-quad-variation}\stepcounter{cnt-lem-quad-variation} (a preliminary bound).}
We start by writing out $\:W_{r}$ for the process $\wb{\:Y}_s$,
\begin{align*}
\:W_r =\frac{1}{\wb{q}^2}\sum_{\{s,i,j\}\leq r} \expectedvalue \left[\left(\frac{\wb{z}_{s-1,i,j}}{\wb{p}_{s-1,i,j}} - \frac{\wb{z}_{s,i,j}}{\wb{p}_{s,i,j}}\right)^{2} \condbar \F_{\{s,i,j\}-1} \right]\concvec_i\concvec_i^\transp\concvec_i\concvec_i^\transp.
\end{align*}
We rewrite the expectation terms in the equation above as
\begin{align*}
\expectedvalue &\left[ 
\left(\frac{\wb{z}_{s-1,i,j}}{\wb{p}_{s-1,i,j}} - \frac{\wb{z}_{s,i,j}}{\wb{p}_{s,i,j}}\right)^{2} \condbar \F_{\{s,i,j\}-1} \right]\\
&= \expectedvalue \left[\frac{\wb{z}_{s-1,i,j}^2}{\wb{p}_{s-1,i,j}^2} -2 \frac{\wb{z}_{s-1,i,j}}{\wb{p}_{s-1,i,j}}\frac{\wb{z}_{s,i,j}}{\wb{p}_{s,i,j}} +\frac{\wb{z}_{s,i,j}^2}{\wb{p}_{s,i,j}^2} \condbar \F_{\{s,i,j\}-1} \right]\\
&\stackrel{(a)}{=} \expectedvalue \left[\frac{\wb{z}_{s-1,i,j}^2}{\wb{p}_{s-1,i,j}^2} -2 \frac{\wb{z}_{s-1,i,j}}{\wb{p}_{s-1,i,j}}\frac{\wb{z}_{s,i,j}}{\wb{p}_{s,i,j}} +\frac{\wb{z}_{s,i,j}^2}{\wb{p}_{s,i,j}^2} \condbar \F_{s-1} \right]\\
&= \frac{\wb{z}_{s-1,i,j}^2}{\wb{p}_{s-1,i,j}^2} -2 \frac{\wb{z}_{s-1,i,j}}{\wb{p}_{s-1,i,j}}\frac{1}{\wb{p}_{s,i,j}}\expectedvalue \left[\wb{z}_{s,i,j}\condbar \F_{s-1} \right] +\frac{1}{\wb{p}_{s,i,j}^2}\expectedvalue \left[\wb{z}_{s,i,j}^2 \condbar \F_{s-1} \right]\\
&\stackrel{(b)}{=} \frac{\wb{z}_{s-1,i,j}}{\wb{p}_{s-1,i,j}^2}
-2 \frac{\wb{z}_{s-1,i,j}}{\wb{p}_{s-1,i,j}}\frac{\wb{z}_{s-1,i,j}}{\wb{p}_{s-1,i,j}}
+\frac{1}{\wb{p}_{s,i,j}^2}\expectedvalue \left[\wb{z}_{s,i,j}\condbar \F_{s-1} \right]\\
&=\frac{1}{\wb{p}_{s,i,j}^2}\expectedvalue \left[\wb{z}_{s,i,j} \condbar \F_{s-1} \right] - \frac{\wb{z}_{s-1,i,j}}{\wb{p}_{s-1,i,j}^2}\\
&\stackrel{(c)}{=}\frac{1}{\wb{p}_{s,i,j}}\frac{\wb{z}_{s-1,i,j}}{\wb{p}_{s-1,i,j}} - \frac{\wb{z}_{s-1,i,j}}{\wb{p}_{s-1,i,j}^2}
=\frac{\wb{z}_{s-1,i,j}}{\wb{p}_{s-1,i,j}}\left(\frac{1}{\wb{p}_{s,i,j}} - \frac{1}{\wb{p}_{s-1,i,j}}\right)\CommaBin
\end{align*}
where in $(a)$ we use the fact that the approximate probabilities $\wb{p}_{s-1,i,j}$ and $\wb{p}_{s,i,j}$ and $\wb{z}_{s-1,i,j}$ are fixed at the end of the previous iteration, while in $(b)$ and $(c)$ we use the fact that $\wb{z}_{s,i,j}$ is a Bernoulli of parameter $\wb{p}_{s,i,j}/\wb{p}_{s-1,i,j}$ (whenever $\wb{z}_{s-1,i,j}$ is equal to 1).
Therefore, we can write $\:W_r$ at the end of the process as
\begin{align*}
\:W_{h} = \:W_{\{h,m,\wb{q}\}} &= \frac{1}{\wb{q}^2}\sum_{j=1}^{\wb{q}} \sum_{i=1}^{\nhl} \sum_{s=1}^h \frac{\wb{z}_{s-1,i,j}}{\wb{p}_{s-1,i,j}}\left(\frac{1}{\wb{p}_{s,i,j}} - \frac{1}{\wb{p}_{s-1,i,j}}\right)\concvec_i\concvec_i^\transp\concvec_i\concvec_i^\transp.
\end{align*}

We can now upper-bound $\:W_h$ as
\begin{align*}
\:W_{h} &\preceq \frac{1}{\wb{q}^2}\sum_{j=1}^{\wb{q}} \sum_{i=1}^{\nhl} \sum_{s=1}^h \frac{\wb{z}_{s-1,i,j}}{\wb{p}_{s-1,i,j}}\left(\frac{1}{\wb{p}_{s,i,j}} - \frac{1}{\wb{p}_{s-1,i,j}}\right)\concvec_i\concvec_i^\transp\concvec_i\concvec_i^\transp\\
&= \frac{1}{\wb{q}^2}\sum_{j=1}^{\wb{q}} \sum_{i=1}^{\nhl} \left(\frac{\wb{z}_{h,i,j}}{\wb{p}_{h,i,j}^2} - \frac{\wb{z}_{h,i,j}}{\wb{p}_{h,i,j}^2} + \sum_{s=1}^h \frac{\wb{z}_{s-1,i,j}}{\wb{p}_{s-1,i,j}}\left(\frac{1}{\wb{p}_{s,i,j}} - \frac{1}{\wb{p}_{s-1,i,j}}\right)\right)\concvec_i\concvec_i^\transp\concvec_i\concvec_i^\transp\\
&= \frac{1}{\wb{q}^2}\sum_{j=1}^{\wb{q}} \sum_{i=1}^{\nhl} \left(\frac{\wb{z}_{h,i,j}}{\wb{p}_{h,i,j}^2} + \left(\sum_{s=1}^{h} -\frac{\wb{z}_{s,i,j}}{\wb{p}_{s,i,j}^2} + \frac{\wb{z}_{s-1,i,j}}{\wb{p}_{s,i,j}\wb{p}_{s-1,i,j}}\right) - \frac{\wb{z}_{0,i,j}}{\wb{p}_{0,i,j}^2}\right)\concvec_i\concvec_i^\transp\concvec_i\concvec_i^\transp\\
&\preceq \frac{1}{\wb{q}^2}\sum_{j=1}^{\wb{q}} \sum_{i=1}^{\nhl} \left(\frac{\wb{z}_{h,i,j}}{\wb{p}_{h,i,j}^2} + \left(\sum_{s=1}^{h}  \frac{\wb{z}_{s-1,i,j}}{\wb{p}_{s,i,j}\wb{p}_{s-1,i,j}} - \frac{\wb{z}_{s,i,j}}{\wb{p}_{s,i,j}\wb{p}_{s-1,i,j}}\right)\right)\concvec_i\concvec_i^\transp\concvec_i\concvec_i^\transp\\
&= \frac{1}{\wb{q}^2}\sum_{j=1}^{\wb{q}} \sum_{i=1}^{\nhl} \left(\frac{\wb{z}_{h,i,j}}{\wb{p}_{h,i,j}^2} + \sum_{s=1}^{h}  \frac{\wb{z}_{s-1,i,j}(1 - \wb{z}_{s,i,j})}{\wb{p}_{s,i,j}\wb{p}_{s-1,i,j}} \right)\concvec_i\concvec_i^\transp\concvec_i\concvec_i^\transp,
\end{align*}
where in the inequality we use the fact $\wb{p}_{s,i,j} \leq \wb{p}_{s-1,i,j}$. From the definition
of $\wb{p}_{s,i,j}$, we know that when $\wb{z}_{s,i,j} = 0$, $\wb{p}_{s,i,j} = \wb{p}_{s-1,i,j}$.
Therefore $\frac{\wb{z}_{s-1,i,j}(1 - \wb{z}_{s,i,j})}{\wb{p}_{s,i,j}\wb{p}_{s-1,i,j}} = \frac{\wb{z}_{s-1,i,j}(1 - \wb{z}_{s,i,j})}{\wb{p}_{s-1,i,j}^2}$, since the term is non-zero only when $\wb{z}_{s,i,j} = 0$.
Finally, we see that only one of the $\wb{z}_{s-1,i,j}(1-\wb{z}_{s,i,j})$ terms can be active for $s \in [h]$ and thus
\begin{align}
\:W_{h}
& \preceq \frac{1}{\wb{q}^2}\sum_{j=1}^{\wb{q}} \sum_{i=1}^{\nhl} \left(\frac{\wb{z}_{h,i,j}}{\wb{p}_{h,i,j}^2} + \sum_{s=1}^{h}  \frac{\wb{z}_{s-1,i,j}(1 - \wb{z}_{s,i,j})}{\wb{p}_{s-1,i,j}^2} \right)\concvec_i\concvec_i^\transp\concvec_i\concvec_i^\transp\nonumber\\
&=\frac{1}{\wb{q}^2}\sum_{j=1}^{\wb{q}} \sum_{i=1}^{\nhl} \left(\max\left\{ \max_{s=1,\dots,h}  \left\{\frac{\wb{z}_{s-1,i,j}(1 - \wb{z}_{s,i,j})}{\wb{p}_{s-1,i,j}^2}\right\} ; \frac{\wb{z}_{h,i,j}}{\wb{p}_{h,i,j}^2}\right\} \right)\concvec_i\concvec_i^\transp\concvec_i\concvec_i^\transp\nonumber\\
&= \frac{1}{\wb{q}^2}\sum_{j=1}^{\wb{q}}\sum_{i=1}^{\nhl}\concvec_i\concvec_i^\transp\concvec_i\concvec_i^\transp\left(\max_{s=0,\dots,h}\left\{\frac{\wb{z}_{s,i,j}}{\wb{p}_{s,i,j}^2}\right\}\right)\cdot\label{eq:dominance-W}
\end{align}

\textbf{Step \arabic{cnt-lem-quad-variation}\stepcounter{cnt-lem-quad-variation} (introduction of a stochastically dominant process).}
We want to study $\max_{s=0,\dots,h}\left\{\frac{\wb{z}_{s,i,j}}{\wb{p}_{s,i,j}^2}\right\}$.
To simplify notation, we will consider $\max_{s=0,\dots,h}\left\{\frac{\wb{z}_{s,i,j}}{\wb{p}_{s,i,j}}\right\}$,
where we removed the square, which will be re-added in the end.
We know trivially that this quantity is larger or equal than, 1 because $\wb{z}_{0,i,j}/\wb{p}_{0,i,j} = 1$,
but upper-bounding this quantity is not trivial as the evolution
of the various $\wb{p}_{s,i,j}$ depends in a complex way on the interaction
between the random variables $\wb{z}_{s,i,j}$.
Nonetheless, whenever $\wb{p}_{s,i,j}$ is significantly
smaller than $\wb{p}_{s-1,i,j}$, the probability of keeping a copy of point $i$ at
iteration $s$ (i.e., $\wb{z}_{s,i,j}=1$) is also very small. As a result, we expect
the ratio $\frac{\wb{z}_{s,i,j}}{\wb{p}_{s,i,j}}$ to be still small with
high probability.

Unfortunately, due to the dependence between different copies
of the point at different iterations, it seems difficult to exploit this intuition directly
to provide an overall high-probability bound on $\:W_{h}$. For this
reason, we simplify the analysis by replacing each of the (potentially
dependent) chains $\{\wb{z}_{s,i,j}/\wb{p}_{s,i,j}\}_{s=0}^{h}$ with a set of
(independent) random variables $w_{0,i,j}$ that will stochastically dominate
them.

We define the random variable $w_{s,i,j}$ using the following conditional
distribution,\footnote{
Notice that unlike $\wb{z}_{s,i,j}$, $w_{s,i,j}$ is no longer $\F_{s}$-measurable but it is $\F'_{s}$-measurable, where 
\begin{align*}
\F'_{\{s,i,j\}} = \left\{ u_{s',i',j'},\; \forall\{s',i',j'\} \leq \{s,i,j\} \right\} \cup \left\{  w_{s,i,j} \right\}
= \F_{\{s,i,j\}} \cup \left\{  w_{s,i,j} \right\}.
\end{align*}
} 
\begin{align*}
\probability\left(\frac{1}{w_{s,i,j}} \leq a \condbar \F_{s}\right)
 = \begin{cases}
0 &\text{ for }\quad a < 1/\wb{p}_{s,i,j}\\
1-\frac{1}{\wb{p}_{s,i,j}a} &\text{ for }\quad 1/\wb{p}_{s,i,j} \leq a < \alpha/p_{h,i}\\
1 &\text{ for }\quad \alpha/p_{h,i} \leq a
\end{cases}.
\end{align*}
To show that this distribution is well defined, we use Lem.~\ref{lem:wbp-always-good}
to guarantee that $1/\wb{p}_{s,i,j} \leq a < \alpha/p_{h,i}$.
Note that the distribution of $\frac{1}{w_{s,i,j}}$ conditioned on $\F_{s}$
is determined by only $\wb{p}_{s,i,j}$, $p_{h,i}$, and~$\alpha$, where~$p_{h,i}$ and~$\alpha$
are fixed. Remembering that $\wb{p}_{s,i,j}$ is a function of
$\F_{s-1}$ (computed using the previous iteration),
we have that  
\begin{align*}
\probability\left(\frac{1}{w_{s,i,j}} \leq a \condbar \F_{s}\right)
= \probability\left(\frac{1}{w_{s,i,j}} \leq a \condbar \F_{s-1}\right).
    \end{align*}
Notice that  in the definition of $w_{s,i,j}$, none of the other $w_{s',i',j'}$
(for any different $s'$, $i'$, or $j'$) appears
and $\wb{p}_{s,i,j}$ is a function of
$\F_{s-1}$. It follows that given  $\F_{s-1}$, $w_{s,i,j}$ is independent from all other $w_{s',i',j'}$
(for any different $s'$, $i'$, or $j'$).
\begin{figure}[t]\label{fig:rand-var-dep-graph}
\begin{center}
\includegraphics[width=0.8\textwidth]{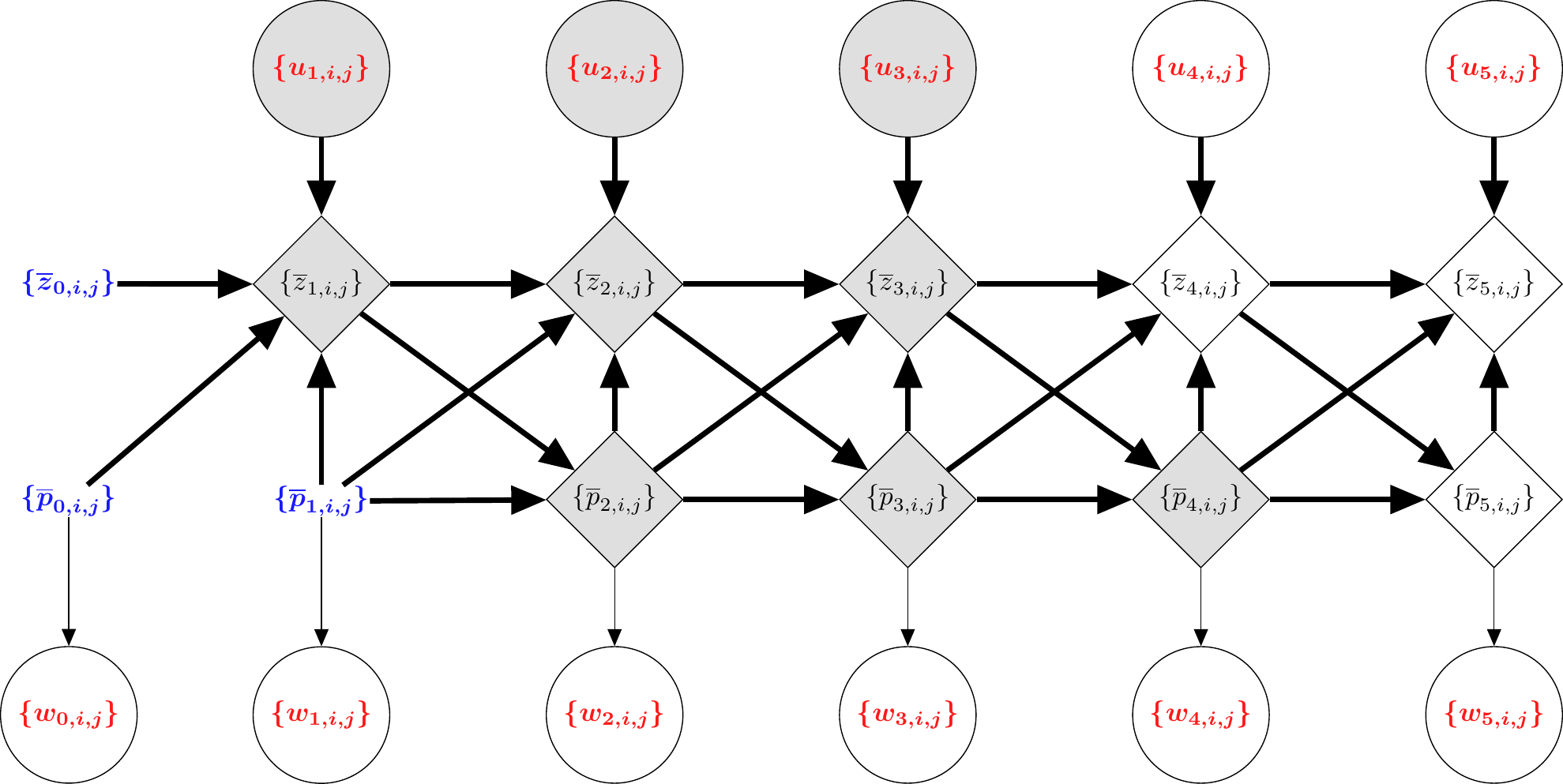}
\end{center}
\caption{The dependence graph of the considered variables. \rcolb{Red} variables are \rcolb{random}. Black variables are deterministically computed using their input (a function of their input), with bold lines indicating the deterministic (functional) relation. \bcolb{Blue} variables are \bcolb{constants}. A \colorbox{gray!30}{grey filling} indicates that a random variable is \colorbox{gray!30}{observed} or a function of observed variables.}
\label{fig:rand-var-dep-graph}
\end{figure}
This is easier to see in the probabilistic graphical model reported in
Fig.~\ref{fig:rand-var-dep-graph}, which illustrates the dependence between
the various variables.

Finally for  the special case $w_{0,i,j}$ the definition above reduces to
\begin{align}\label{eq:distro.w}
\probability\left(\frac{1}{w_{0,i,j}} \leq a\right)
 = \begin{cases}
0 &\text{ for }\quad a < 1\\
1-\frac{1}{a} &\text{ for }\quad 1 \leq a < \alpha/p_{h,i}\\
1 &\text{ for }\quad \alpha/p_{h,i} \leq a
\end{cases},
\end{align}
since $\wb{p}_{0,i,j}=1$ by definition.
From this definition, $w_{0,i,j}$ and $w_{0,i',j'}$ are all
independent, and this will allow us to use stronger concentration inequalities
for independent random variables.

\textbf{Step \arabic{cnt-lem-quad-variation}\stepcounter{cnt-lem-quad-variation} {Proving the dominance}.}
We remind the reader that a random variable $A$ stochastically dominates random
variable $B$, if for all values $a$ the two equivalent conditions are verified,
\begin{align*}
\probability(A \geq a) \geq \probability(B \geq a) \Leftrightarrow \probability(A \leq a) \leq \probability(B \leq a).
\end{align*}
As a consequence, if $A$ dominates $B$, the following implication holds,
\begin{align*}
\probability(A \geq a) \geq \probability(B \geq a) \implies \expectedvalue[A] \geq \expectedvalue[B],
\end{align*}
while the reverse ($A$ dominates $B$, if $\expectedvalue[A] \geq \expectedvalue[B]$) is not
true in general.
Following this definition of  stochastic dominance, our goal is to prove
\begin{align*}
\probability\left(\max_{s=0}^h\frac{\wb{z}_{s,i,j}}{\wb{p}_{s,i,j}} \leq a\right)
\geq 
\probability\left(\frac{1}{w_{0,i,j}} \leq a \right).
\end{align*}
We prove this inequality by proceeding backwards with a sequence of conditional probabilities.
We first study the distribution of the maximum conditional to the state of the algorithm at the end of iteration $h$, i.e., $\Fp_{h}$. From the definition of $w_{h,i,j}$, we know that,
w.p.~1, $1/\wb{p}_{h,i} \leq 1/w_{h,i,j}$.
Therefore,
\begin{align*}
&\probability\left(\max_{s=0,\dots,h}\frac{\wb{z}_{s,i,j}}{\wb{p}_{s,i,j}} \leq a\right)
\geq \probability\left(\max\left\{ \max_{s=0,\dots,h-1}\frac{\wb{z}_{s,i,j}}{\wb{p}_{s,i,j}} ; \frac{\wb{z}_{h,i,j}}{w_{h,i,j}}\right\} \leq a \right).
\end{align*}Now focus on an arbitrary intermediate step $1\leq k\leq h$, where we fix $\Fp_{k-1}$. Since~$u_{k,i,j}$ and
$w_{k,i,j}$ are independent given $\Fp_{k-1}$, we have
\begin{align}
\probability\left( \frac{\wb{z}_{k,i,j}}{w_{k,i,j}} \leq a \condbar \Fp_{k-1}\right)
&=\probability\left( \indfunc\left\{u_{k,i,j} \leq \frac{\wb{p}_{k,i,j}}{\wb{p}_{k-1,i,j}}\right\}\frac{1}{w_{k,i,j}} \leq a \condbar \Fp_{k-1}\right)\nonumber\\
&=
\begin{cases}
0 &\text{ for }\quad a \leq 0 \nonumber\\
1 - \frac{\wb{p}_{k,i,j}}{\wb{p}_{k-1,i,j}} &\text{ for }\quad  0 \leq a < 1/\wb{p}_{k,i,j}\\
1-\frac{\wb{p}_{k,i,j}}{\wb{p}_{k-1,i,j}} +\frac{\wb{p}_{k,i,j}}{\wb{p}_{k-1,i,j}}\left(1-\frac{1}{\wb{p}_{k,i,j}a}\right) = 1-\frac{1}{\wb{p}_{k-1,i,j}a}  &\text{ for }\quad 1/\wb{p}_{k,i,j} \leq a < \alpha/p_{h,i} \\ 
1 &\text{ for }\quad \alpha/p_{h,i} \leq a
\end{cases}\nonumber\\
&\geq 
\begin{cases}
0 &\text{ for }\quad a < 1/\wb{p}_{k-1,i,j} \\
1-\frac{1}{\wb{p}_{k-1,i,j}a} &\text{ for }\quad  1/\wb{p}_{k-1,i,j} \leq a < 1/\wb{p}_{k,i,j} \\
1-\frac{1}{\wb{p}_{k-1,i,j}a}  &\text{ for }\quad 1/\wb{p}_{k,i,j} \leq a < \alpha/p_{h,i}  \\
1 &\text{ for }\quad \alpha/p_{h,i}  \leq a 
\end{cases}\label{w.stoch123}\\
&=\probability\left(\frac{1}{w_{k-1,i,j}} \leq a \condbar \Fp_{k-2}\right)
=\probability\left(\frac{1}{w_{k-1,i,j}} \leq a \condbar \Fp_{k-1}\right) \nonumber,
\end{align}
where the inequality is also represented in Fig.~\ref{fig:dominance}.
\begin{figure}[t]
\centering
\includegraphics[width=0.6\textwidth]{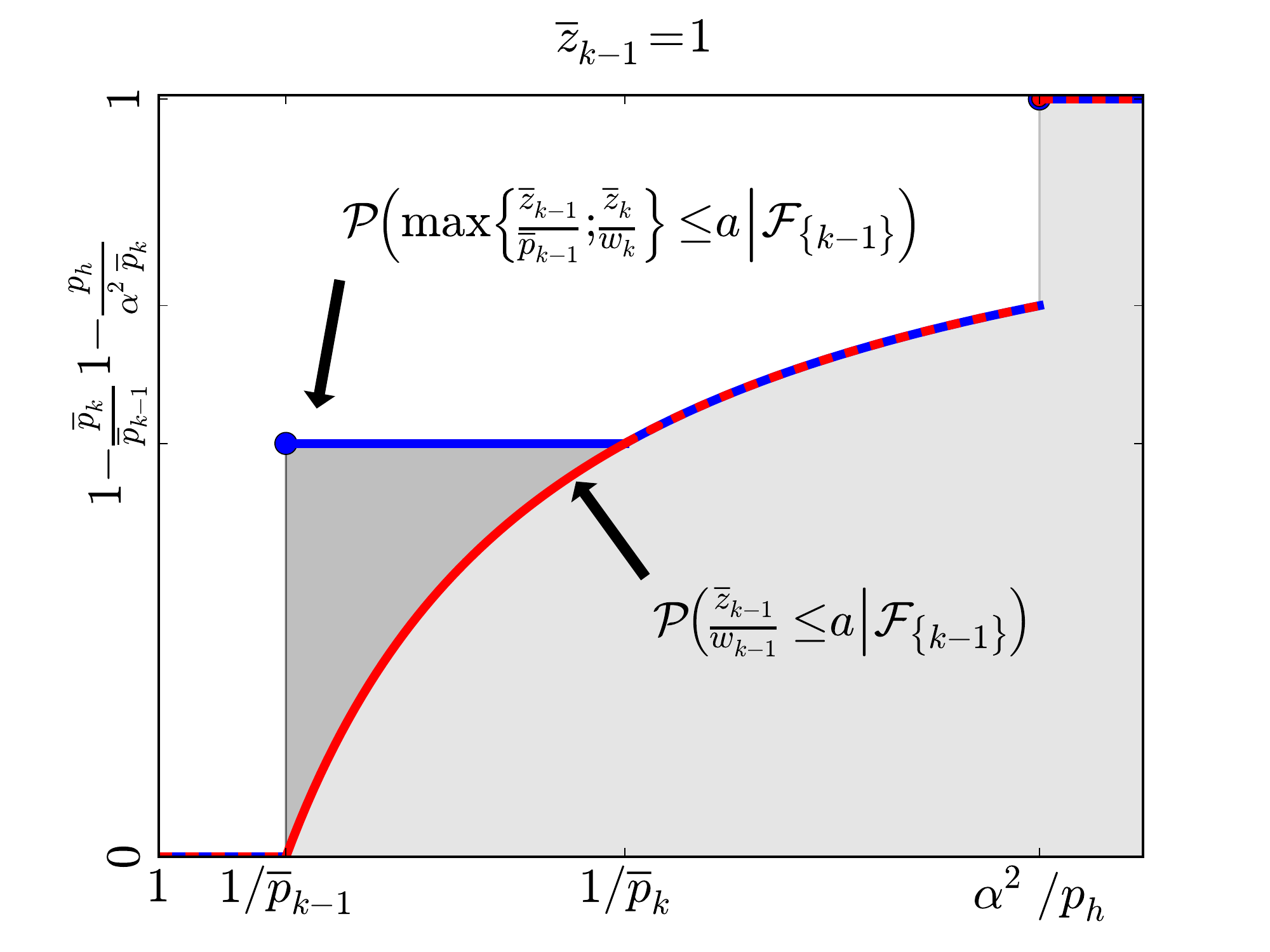}
\caption{C.d.f.\,of $\max\left\{ \wb{z}_{k-1,i,j}/\wb{p}_{t,i,j} ; \wb{z}_{k,i,j}/\wb{p}_{k,i,j}\right\}$ and
$\wb{z}_{k-1,i,j}/w_{k-1,i,j}$ conditioned on $\mathcal{F}_{\{k-1\}}$. For conciseness, we omit the $i,j$ indices.}\label{fig:dominance}
\end{figure}
We now proceed by peeling off layers from the end of the chain one by one, taking advantage of the dominance we just proved. Fig.~\ref{fig:dominance} visualizes one step of the peeling when $\wb{z}_{k-1,i,j} = 1$ (note that the peeling is trivially true when $\wb{z}_{k-1,i,j} = 0$ since the whole chain terminated at step $\wb{z}_{k-1,i,j}$). We show how to move from an iteration $k\leq h$ to $k-1$.
\begin{align*}\label{eq:generic.case}
& \probability\left(\max\left\{ \max_{s=0 \dots k-1}\frac{\wb{z}_{s,i,j}}{\wb{p}_{s,i,j}} ; \frac{\wb{z}_{k,i,j}}{w_{k,i,j}}\right\} \leq a \right)
=\expectedvalue_{\Fp_{k-1}}\left[\probability\left(\max\left\{ \max_{s=0 \dots k-1}\frac{\wb{z}_{s,i,j}}{\wb{p}_{s,i,j}} ; \frac{\wb{z}_{k,i,j}}{w_{k,i,j}}\right\} \leq a \condbar \Fp_{k-1} \right)\right]\nonumber\\
&\stackrel{(a)}{\geq}\expectedvalue_{\Fp_{k-1}}\left[\probability\left(\max\left\{ \max_{s=0 \dots k-1}\frac{\wb{z}_{s,i,j}}{\wb{p}_{s,i,j}} ; \frac{\wb{z}_{k-1,i,j}}{w_{k-1,i,j}}\right\} \leq a \condbar \Fp_{k-1} \right)\right]\nonumber\\
&=\expectedvalue_{\Fp_{k-1}}\left[\probability\left(\max\left\{ \max_{s=0 \dots k-2}\frac{\wb{z}_{s,i,j}}{\wb{p}_{s,i,j}} ;\frac{\wb{z}_{k-1,i,j}}{\wb{p}_{k-1,i,j}}; \frac{\wb{z}_{k-1,i,j}}{w_{k-1,i,j}}\right\} \leq a \condbar \Fp_{k-1} \right)\right]\nonumber\\
&=\expectedvalue_{\Fp_{k-1}}\left[\probability\left(\max\left\{ \max_{s=0 \dots k-2}\frac{\wb{z}_{s,i,j}}{\wb{p}_{s,i,j}} ;\wb{z}_{k-1,i,j}\max\left\{\frac{1}{\wb{p}_{k-1,i,j}}; \frac{1}{w_{k-1,i,j}}\right\}\right\} \leq a \condbar \Fp_{k-1} \right)\right]\nonumber\\
&\stackrel{(b)}{=}\expectedvalue_{\Fp_{k-1}}\left[\probability\left(\max\left\{ \max_{s=0 \dots k-2}\frac{\wb{z}_{s,i,j}}{\wb{p}_{s,i,j}} ;\frac{\wb{z}_{k-1,i,j}}{w_{k-1,i,j}}\right\} \leq a \condbar \Fp_{k-1} \right)\right]
= \probability\left(\max\left\{ \max_{s=0 \dots k-2}\frac{\wb{z}_{s,i,j}}{\wb{p}_{s,i,j}} ; \frac{\wb{z}_{k-1,i,j}}{w_{k-1,i,j}}\right\} \leq a \right),\nonumber
\end{align*}where in $(a)$, given $\F_{k-1}$, everything is fixed except $u_{k,i,j}$ and
$w_{k,i,j}$ and we can use the stochastic dominance in~\eqref{w.stoch123}, and
in $(b)$ we use the fact that the inner maximum is always attained by
$1/w_{k,i,j}$ since by definition $1/w_{k-1,i,j}$ is lower-bounded by
$1/\wb{p}_{k-1,i,j}$.
Applying the inequality recursively from $k=h$ to $k=1$ 
removes all $\wb{z}_{s,i,j}$ from the maximum and we
are finally left with only $w_{0,i,j}$ as we wanted,
\begin{align*}
\probability\left(\max_{s=0,\dots,h}\frac{\wb{z}_{s,i,j}}{\wb{p}_{s,i,j}} \leq a\right)
\geq 
\probability\left(\max\left\{\frac{\wb{z}_{0,i,j}}{\wb{p}_{0,i,j}} ; \frac{\wb{z}_{0,i,j}}{w_{0,i,j}}\right\} \leq a \right)
\geq 
\probability\left(\frac{1}{w_{0,i,j}} \leq a \right),
\end{align*}
where in the last inequality we used that $\wb{z}_{0,i,j} = 1$ from the definition of the algorithm 
and $\wb{p}_{0,i,j} = 1$ while $w_{0,i,j} \leq 1$ by~\eqref{eq:distro.w}.

\textbf{Step \arabic{cnt-lem-quad-variation}\stepcounter{cnt-lem-quad-variation} (stochastic dominance on $\:W_{h}$).}
Now that we proved the stochastic dominance of $1/w_{0,i,j}$, we plug this
result in the definition of $\:W_{h}$. For the sake of notation, we introduce
the term $\wb{p}_{h',i,j}^{\max}$
to indicate the maximum over the first $h'$ step of copy $i,j$ such that
\begin{align*}
\max_{s=0,\dots,h'}\frac{\wb{z}_{s,i,j}}{\wb{p}_{s,i,j}} = \frac{1}{\wb{p}_{h',i,j}^{\max}}\cdot
\end{align*}
We first notice that while $\wb{\:Y}_{h}$ is not
necessarily PSD, $\:W_{h}$ is a sum of PSD matrices.
Introducing the function $\Lambda(\{1/\wb{p}_{h,i,j}^{\max}\}_{i,j})$ we can restate Eq.~\ref{eq:dominance-W} as
\begin{align*}
\normsmall{\:W_{h}}
=\lambda_{\max}(\:W_{h})
\leq \Lambda(\{1/\wb{p}_{h,i,j}^{\max}\}_{i,j}) \eqdef
\lambda_{\max}\left(\frac{1}{\wb{q}^2}\sum_{j=1}^{\wb{q}}\sum_{i=1}^{\nhl}\left(\frac{1}{\wb{p}_{h,i,j}^{\max}}\right)^2\concvec_i\concvec_i^\transp\concvec_i\concvec_i^\transp\right).
\end{align*}
In Step 4, we showed that $1/\wb{p}_{h,i,j}^{\max}$ is stochastically
dominated by $1/w_{0,i,j}$ for every $i$ and~$j$. In order to bound
$\Lambda(\{1/\wb{p}_{h,i,j}^{\max}\}_{i,j})$,
we need to show that this dominance also applies to the summation
over all columns inside the matrix norm. We
can reformulate $\Lambda(\{1/\wb{p}_{h,i,j}^{\max}\}_{i,j})$ as
\begin{align*}
&\lambda_{\max}\left(\frac{1}{\wb{q}^2}\sum_{j=1}^{\wb{q}}\sum_{i=1}^{\nhl}\left(\frac{1}{\wb{p}_{h,i,j}^{\max}}\right)^2\concvec_i\concvec_i^\transp\concvec_i\concvec_i^\transp\right)
= \max_{\:x : \normsmall{\:x} = 1} \:x^\transp\left(\frac{1}{\wb{q}^2}\sum_{j=1}^{\wb{q}}\sum_{i=1}^{\nhl}\left(\frac{1}{\wb{p}_{h,i,j}^{\max}}\right)^2\concvec_i\concvec_i^\transp\concvec_i\concvec_i^\transp\right)\:x\\
&= \max_{\:x : \normsmall{\:x} = 1}\frac{1}{\wb{q}^2}\sum_{j=1}^{\wb{q}}\sum_{i=1}^{\nhl}\left(\frac{1}{\wb{p}_{h,i,j}^{\max}}\right)^2\normsmall{\concvec_i}_{2}^2\:x^\transp\concvec_i\concvec_i^\transp\:x
= \max_{\:x : \normsmall{\:x} = 1}\frac{1}{\wb{q}^2}\sum_{j=1}^{\wb{q}}\sum_{i=1}^{\nhl}\left(\frac{1}{\wb{p}_{h,i,j}^{\max}}\right)^2\left(\normsmall{\concvec_i}_{2}\concvec_i^\transp\:x\right)^2.
\end{align*}
From this reformulation, it is easy to see that, because $1/\wb{p}_{h,i,j}^{\max}$
is strictly positive,
the function $\Lambda(\{1/\wb{p}_{h,i,j}^{\max}\}_{i,j})$ is
monotonically increasing w.r.t.\,the individual $1/\wb{p}_{h,i,j}^{\max}$, or
in other words that increasing an $1/\wb{p}_{h,i,j}^{\max}$ without decreasing
the others can only increase the maximum.
Introducing $\Lambda(\{1/w_{0,i,j}\}_{i,j})$ as
\begin{align*}
\Lambda(\{1/w_{0,i,j}\}_{i,j}) \eqdef \max_{\:x : \normsmall{\:x} = 1}\frac{1}{\wb{q}^2}\sum_{j=1}^{\wb{q}}\sum_{i=1}^{\nhl}\left(\frac{1}{w_{0,i,j}}\right)^2\left(\normsmall{\concvec_i}_{2}\concvec_i^\transp\:x\right)^2,
\end{align*}
we now need to prove the stochastic dominance of $\Lambda(\{1/w_{0,i,j}\}_{i,j})$ over
$\Lambda(\{1/\wb{p}_{h,i,j}^{\max}\}_{i,j})$.
Using the definition of $1/\wb{p}_{h,i,j}^{\max}$, $w_{h,i,j}$, and the
monotonicity of $\Lambda$ we have
\begin{align*}
\probability\left(\Lambda\left(\left\{\frac{1}{\wb{p}_{h,i,j}^{\max}}\right\}_{i,j}\right) \leq a\right)
&=\probability\left(\Lambda\left(\left\{\max\left\{ \max_{s=0,\dots,h-1}\frac{\wb{z}_{s,i,j}}{\wb{p}_{s,i,j}} ; \frac{\wb{z}_{h,i,j}}{\wb{p}_{h,i,j}}\right\}\right\}_{i,j}\right) \leq a \right)\\
&\geq \probability\left(\Lambda\left(\left\{\max\left\{ \max_{s=0,\dots,h-1}\frac{\wb{z}_{s,i,j}}{\wb{p}_{s,i,j}} ; \frac{\wb{z}_{h,i,j}}{w_{h,i,j}}\right\}\right\}_{i,j}\right) \leq a \right).
\end{align*}
Now pick $1\leq k \leq h$, for a fixed $\Fp_{k-1}$, $\frac{1}{\wb{p}_{k-1,i,j}^{\max}}$ is a constant
and $\max\left\{\frac{1}{\wb{p}_{k,i,j}^{\max}}; x\right\}$ is a monotonically
increasing function in $x$, making $\Lambda\left(\max\left\{\frac{1}{\wb{p}_{k,i,j}^{\max}}; x\right\}\right)$ also an increasing function. Therefore, we have
\begin{align*}
&\probability\left(\Lambda\left(\left\{\max\left\{ \frac{1}{\wb{p}_{k-1,i,j}^{\max}} ; \frac{\wb{z}_{k,e,j}}{w_{k,i,j}}\right\} \right\}_{i,j}\right) \leq a \right)
=\expectedvalue_{\Fp_{k-1}}\left[\probability\left(\Lambda\left(\left\{\max\left\{ \frac{1}{\wb{p}_{k-1,i,j}^{\max}} ; \frac{\wb{z}_{k,e,j}}{w_{k,i,j}}\right\}\right\}_{i,j}\right) \leq a \condbar \Fp_{k-1}\right)\right]\\
&\stackrel{(a)}{\geq}\expectedvalue_{\Fp_{k-1}}\left[\probability\left(\Lambda\left(\left\{\max\left\{ \frac{1}{\wb{p}_{k-1,i,j}^{\max}} ; \frac{\wb{z}_{k-1,i,j}}{w_{k-1,i,j}}\right\}\right\}_{i,j}\right) \leq a \condbar \Fp_{k-1}\right)\right]\\
&\stackrel{(b)}{=}\expectedvalue_{\Fp_{k-1}}\left[\probability\left(\Lambda\left(\left\{\max\left\{ \frac{1}{\wb{p}_{k-2,i,j}^{\max}} ; \frac{\wb{z}_{k-1,i,j}}{w_{k-1,i,j}}\right\}\right\}_{i,j}\right) \leq a \condbar \Fp_{k-1}\right)\right],
\end{align*}
where inequality (a) follows from the fact that stochastic dominance is preserved
by monotonically increasing functions \cite{levy2015stochastic}, such as
$\Lambda$, combined with the fact that for
a fixed $\Fp_{k-1}$ the variables $\wb{z}_{k,i,j}$ and $w_{k,i,j}$ are all
independent and (b) from the definition of $1/\wb{p}_{k-1,i,j}^{\max}$ and the
fact that by definition $1/w_{k-1,i,j}$ is lower-bounded by $1/\wb{p}_{k-1,i,j}$.
We can iterate this inequality to obtain the desired result
\begin{align*}
\probability(\normsmall{\:W_{h}} \geq \sigma^2)
&\leq \probability\left(\Lambda\left(\left\{\frac{1}{\wb{p}_{h,i,j}^{\max}}\right\}_{i,j}\right) \geq \sigma^2\right)
\leq \probability\left(\lambda_{\max}\left(\frac{1}{\wb{q}^2}\sum_{j=1}^{\wb{q}}\sum_{i=1}^{\nhl}\left(\frac{1}{w_{0,i,j}}\right)^2\concvec_i\concvec_i^\transp\concvec_i\concvec_i^\transp\right) \geq \sigma^2\right).
\end{align*}

\textbf{Step \arabic{cnt-lem-quad-variation}\stepcounter{cnt-lem-quad-variation} (concentration inequality).}
Since all $w_{0,i,j}$ are (unconditionally) independent from each other, we can apply the following theorem.
\begin{proposition}[\citet{tropp2015an-introduction}, Theorem~5.1.1]\label{prop:matrix-chernoff}
Consider a finite sequence $\{ \:X_k : k = 1, 2, 3, \dots \}$ whose values are independent, random, PSD Hermitian matrices with dimension $d$.  Assume that each term in the sequence is uniformly bounded in the sense that
\begin{align*}
\lambda_{\max}( \:X_k ) \leq L
\quad\text{almost surely}
\quad\text{for $k = 1, 2, 3, \dots$}.
\end{align*}
Introduce the random matrix
$\:V \eqdef \sum_{k}  \:X_k$, and
the maximum eigenvalue of its expectation
\begin{align*}
\mu_{\max} \eqdef \lambda_{\max}(\expectedvalue\left[\:V\right]) = \lambda_{\max}\left(\sum_{k} \expectedvalue\left[\:X_k\right]\right).
\end{align*}
Then, for all $h \geq 0$,
\begin{align*}
\probability\left(  \lambda_{\max}(\:V) \geq (1+h)\mu_{\max}  \right)
	&\leq d \cdot \left[\frac{e^{h}}{(1+h)^{1+h}}\right]^{\frac{\mu_{\max}}{L}}\\
	&\leq d \cdot \exp \left\{ - \frac{\mu_{\max}}{L}((h+1)\log(h+1) - h) \right\}\cdot
\end{align*}
\end{proposition}

In our case, we have
\begin{align*}
\:X_{\{i,j\}}
= \frac{1}{\wb{q}^2}\frac{1}{w_{0,i,j}}\concvec_i\concvec_i^\transp\concvec_i\concvec_i^\transp
\preceq \frac{1}{\wb{q}^2}\frac{\alpha^2}{p_{h,i}^2}\concvec_i\concvec_i^\transp\concvec_i\concvec_i^\transp
\preceq \frac{1}{\wb{q}^2}\frac{\alpha^2}{p_{h,i}^2}\normsmall{\concvec_i\concvec_i^\transp}^2\:I
\preceq \frac{\alpha^2}{\wb{q}^2}\:I,
\end{align*}
where the first inequality follows from the definition of $w_{0,i,j}$ in Eq.~\ref{eq:distro.w}, the second
from the PSD ordering, and the third from the definition of $\normsmall{\concvec_i\concvec_i^\transp}$.

Therefore, we can use $L \eqdef \alpha^2/\wb{q}^2$ for the purpose of Prop.~\ref{prop:matrix-chernoff}. We need now to compute $\expectedvalue\left[\:X_k\right]$,
that we can use in turn to compute $\mu_{\max}$. We begin by computing the expected value of $1/w_{0,i,j}$.
Let us denote the c.d.f.\@ of $1/w_{0,i,j}^2$ as
\begin{align*}
F_{1/w_{0,i,j}^2}(a) = \probability\left(\frac{1}{w_{0,i,j}^2} \leq a\right) = \probability\left(\frac{1}{w_{0,i,j}} \leq \sqrt{a}\right)
 = \begin{cases}
0 &\text{ for }\quad a < 1\\
1-\frac{1}{\sqrt{a}} &\text{ for }\quad 1 \leq a < \alpha^2/p_{h,i}^2\\
1 &\text{ for }\quad \alpha^2/p_{h,i}^2 \leq a
\end{cases}.
\end{align*}
Since $\probability\left(1/w_{0,i,j}^2 \ge 0\right) = 1$, we have that 
\newcommand*\diff{\mathop{}\!\mathrm{d}}
\begin{align*}
&\expectedvalue\left[\frac{1}{w_{0,i,j}}\right]
= \int_{a=0}^\infty \left[1-F_{1/w_{0,i,j}}(a)\right]\diff a\\
&= \int_{a=0}^1 \left(1 - F_{1/w_{0,i,j}}(a)\right)\diff a + \int_{a=1}^{\alpha^2/p_{h,i}^2} \left(1 - F_{1/w_{0,i,j}}(a)\right)\diff a + \int_{a=\alpha^2/p_{h,i}^2}^{\infty} \left(1 - F_{1/w_{0,i,j}}(a)\right)\diff a\\
&= \int_{a=0}^1 \left(1 - 0\right)\diff a + \int_{a=1}^{\alpha^2/p_{h,i}^2} \left(1 - \left(1 - \frac{1}{\sqrt{a}}\right) \right)\diff a +\int_{a=\alpha^2/p_{h,i}^2}^{\infty}(1-1)\diff a \\
&= \int_{a=0}^1 \diff a + \int_{a=1}^{\alpha^2/p_{h,i}^2} \frac{1}{\sqrt{a}}\diff a
= 1+[2\sqrt{a}]_{1}^{\alpha^2/p_{h,i}^2}
= 2\alpha/p_{h,i} - 1.
\end{align*}
Therefore,
\begin{align*}
\mu_{\max} &= \lambda_{\max}(\expectedvalue\left[\:V\right]) = \lambda_{\max}\Big(\sum_{\{i,j\}} \expectedvalue\left[\:X_{\{i,j\}}\right]\Big)
= \lambda_{\max}\left(\frac{1}{\wb{q}^2}\sum_{j=1}^{\wb{q}}\sum_{i=1}^{\nhl}\expectedvalue\left[\frac{1}{w_{0,i,j}^2}\right]\concvec_i\concvec_i^\transp\concvec_i\concvec_i^\transp\right)\\
&= \lambda_{\max}\left(\frac{1}{\wb{q}}\sum_{i=1}^{\nhl}\left(\frac{2\alpha}{p_{h,i}} - 1\right)p_{h,i}\concvec_i\concvec_i^\transp\right)
\leq \lambda_{\max}\left(\frac{2\alpha}{\wb{q}}\sum_{i=1}^{\nhl}\concvec_i\concvec_i^\transp\right)
= \frac{2\alpha}{\wb{q}}\lambda_{\max}\left(\:P\right)
\leq \frac{2\alpha}{\wb{q}} \eqdef L.
\end{align*}
Therefore, selecting $h = 2$, $\sigma^2=6\alpha/\wb{q}$ and applying Prop.~\ref{prop:matrix-chernoff} we have
\begin{align*}
\probability\left(\normsmall{\:W_h} \geq \sigma^2\right)
&\leq \probability\left(\lambda_{\max}\left(\frac{1}{\wb{q}^2}\sum_{j=1}^{\wb{q}}\sum_{i=1}^{\nhl}\frac{1}{w_{0,i,j}^2}\concvec_i\concvec_i^\transp\concvec_i\concvec_i^\transp\right) \geq (1+2)\frac{2\alpha}{\wb{q}}\right)\\
&\leq \nhl \cdot \exp \left\{ - \frac{2\alpha}{\wb{q}}\frac{\wb{q}^2}{\alpha^2}(3\log(3) - 2) \right\}
\leq n \cdot \exp \left\{ - \frac{2\wb{q}}{\alpha} \right\}\cdot
\end{align*}

\vspace{-0.05in}
\subsection{Space complexity bound}\label{ssec:space-complexity}
\vspace{-0.05in}
Denote with $A$ the event $A = \left\{\forall  h' \in \{1, \dots, h\} :  \normsmall{\:P^{h'} - \wt{\:P}^{h'}}_2 \leq \varepsilon\right\}$, and again  $\nhl = |\dataset_{\{h,l\}}|$. Letting $q~=~|\coldict_{\{h,l\}}| = \sum_{i=1}^{\nhl} q_{h,i} = \sum_{j=1}^{\wb{q}}\sum_{i=1}^{\nhl}  z_{h,i,j}$ be the random number
of points in $\coldict_{\{h,l\}}$, we reformulate
\begin{align*}
    &\probability\left(|\coldict_{\{h,l\}}| \geq 3\wb{q}\deff{\gamma}_{\{h,l\}} \cap \left\{\forall  h' \in \{1, \dots, h\} :  \left(\normsmall{\:P^{h'} - \wt{\:P}^{h'}}_2 \leq \varepsilon\right) \leq \varepsilon\right\}\right)\\
    &= \probability\left(|\coldict_{\{h,l\}}| \geq 3\wb{q}\deff{\gamma}_{\{h,l\}} \cap A\right)
= \probability\left(\sum_{j=1}^{\wb{q}}\sum_{i=1}^{\nhl}  z_{h,i,j} \geq 3\wb{q}\deff{\gamma}_{\{h,l\}} \cap A\right)\\
&=\probability\left(  \sum_{j=1}^{\wb{q}}\sum_{i=1}^{\nhl} z_{h,i,j} \geq 3\wb{q}\deff{\gamma}_{\{h,l\}} \condbar A \right)
\probability\left(A \right).
\end{align*}
While we do know that the $z_{h,i,j}$ are Bernoulli random variables (since
they are either 0 or 1), it is not easy to compute the success probability
of each $z_{h,i,j}$, and in addition there could be dependencies between
$z_{h,i,j}$ and $z_{h,i',j'}$. Similarly to Lem.~\ref{lem:prob-dominance},
we are going to find a stochastic variable to dominate $z_{h,i,j}$.
Denoting with $u'_{s,i,j} \sim \mathcal{U}(0,1)$ a uniform random variable,
we will define $w'_{s,i,j}$ as
\begin{align*}
w'_{s,i,j} \vert \F_{\{s,i',j'\}} = w'_{s,i,j} \vert   \F_{s-2}  \eqdef \indfunc\left\{u'_{s,i,j} \leq \frac{p_{h,i}}{\wt{p}_{s-1,i}}\right\} \sim \mathcal{B}\left(\frac{p_{h,i}}{\wt{p}_{s-1,i}}\right)
\end{align*}
for any
$i'$ and $j'$ such that $\{s,1,1\} \leq \{s,i',j'\} <\{s,i,j\}$.
Note that $w'_{s,i,j}$, unlike $z_{s,i,j}$, does not have a recursive
definition, and its only dependence on any other variable comes from
$\wt{p}_{s-1,i}$.
First, we peel off the last step
\begin{align*}
&\probability\left( \sum_{j=1}^{\wb{q}} \sum_{i=1}^{\nhl} z_{h,i,j} \geq g \condbar A\right)
= \expectedvalue_{\Fp_{t-1}|A}\left[\probability\left( \sum_{j=1}^{\wb{q}} \sum_{i=1}^{\nhl} \indfunc\left\{u_{h,i,j} \leq \frac{\wt{p}_{h,i}}{\wt{p}_{t-1,i}}\right\}z_{h-1,i,j} \geq g \condbar \Fp_{t-1} \cap A\right)\right]\\
&\leq \expectedvalue_{\Fp_{t-1}|A}\left[\probability\left( \sum_{j=1}^{\wb{q}} \sum_{i=1}^{\nhl} \indfunc\left\{u'_{h,i,j} \leq \frac{p_{h,i}}{\wt{p}_{t-1,i}}\right\}z_{h-1,i,j} \geq g \condbar \Fp_{t-1} \cap A\right)\right]
= \probability\left( \sum_{j=1}^{\wb{q}} \sum_{i=1}^{\nhl} w'_{h,i,j}z_{h-1,i,j} \geq g \condbar A\right),
\end{align*}
where we used the fact that conditioned on $A$, $\coldict_{\{h,l\}}$ is accurate \wrt
$\kermatrix_{\{h,l\}}$, which guarantees that $\wt{p}_{h,i} \leq p_{h,i}$.
Plugging this in the previous bound,
\begin{align*}
\probability&\left( \sum_{j=1}^{\wb{q}} \sum_{i=1}^{\nhl} z_{h,i,j} \geq g \condbar A \right)
\probability\left(A \right)
\leq\probability\left( \sum_{j=1}^{\wb{q}} \sum_{i=1}^{\nhl} w'_{h,i,j}z_{h-1,i,j} \geq g \cap A\right)
\leq\probability\left( \sum_{j=1}^{\wb{q}} \sum_{i=1}^{\nhl} w'_{h,i,j}z_{h-1,i,j} \geq g \right).
\end{align*}
We now proceed by peeling off layers from the end of the chain one by one. We show how to move from an iteration $s\leq h$ to $s-1$.
\begin{align*}
&\probability\left( \sum_{j=1}^{\wb{q}} \sum_{i=1}^{\nhl} w'_{s,i,j}z_{s-1,i,j} \geq g \right)
= \expectedvalue_{\Fp_{s-2}}\left[\probability\left( \sum_{j=1}^{\wb{q}} \sum_{i=1}^{\nhl} \indfunc\left\{u'_{s,i,j} \leq \frac{p_{h,i}}{\wt{p}_{s-1,i}}\right\}z_{s-1,i,j} \geq g \condbar \Fp_{s-2} \right)\right]\\
&= \expectedvalue_{\Fp_{s-2}}\left[\probability\left( \sum_{j=1}^{\wb{q}} \sum_{i=1}^{\nhl} \indfunc\left\{u'_{s,i,j} \leq \frac{p_{h,i}}{\wt{p}_{s-1,i}}\right\}\indfunc\left\{u_{s-1,i,j} \leq \frac{\wt{p}_{s-1,i}}{\wt{p}_{s-2,i}}\right\}z_{s-2,i,j} \geq g \condbar \Fp_{s-2} \right)\right]\\
&= \expectedvalue_{\Fp_{s-2}}\left[\probability\left( \sum_{j=1}^{\wb{q}} \sum_{i=1}^{\nhl} \indfunc\left\{u'_{s-1,i,j} \leq \frac{p_{h,i}}{\wt{p}_{s-2,i}}\right\}z_{s-2,i,j} \geq g \condbar \Fp_{s-2} \right)\right]
= \probability\left( \sum_{j=1}^{\wb{q}} \sum_{i=1}^{\nhl} w'_{s-1,i,j}z_{s-2,i,j} \geq g \right)
\end{align*}
Applying this repeatedly from $s=h$ to $s=2$ we have,
\begin{align*}
&\probability\left( \sum_{j=1}^{\wb{q}} \sum_{i=1}^{\nhl} w'_{h,i,j}z_{h-1,i,j} \geq g \right)
=\probability\left( \sum_{j=1}^{\wb{q}} \sum_{i=1}^{\nhl} w'_{1,i,j}z_{0,i,j} \geq g \right)
=\probability\left( \sum_{j=1}^{\wb{q}} \sum_{i=1}^{\nhl} w'_{1,i,j} \geq g \right). 
\end{align*}
Now, all the $w'_{1,i,j}$ are independent Bernoulli random variables,
and we can bound their sum with a Hoeffding-like bound using Markov inequality,
\begin{align*}
\probability&\left( \sum_{j=1}^{\wb{q}} \sum_{i=1}^{\nhl} w'_{1,i,j} \geq g \right) = \inf_{\theta > 0}\probability\left( e^{\sum_{j=1}^{\wb{q}} \sum_{i=1}^{\nhl} \theta w'_{1,i,j}} \geq e^{\theta g} \right)\\
&\leq \inf_{\theta > 0} \frac{\expectedvalue \left[ e^{\sum_{j=1}^{\wb{q}} \sum_{i=1}^{\nhl} \theta w'_{1,i,j}}\right]}{e^{\theta g}}
= \inf_{\theta > 0} \frac{\expectedvalue \left[ \prod_{j=1}^{\wb{q}} \prod_{i=1}^{\nhl} e^{ \theta w'_{1,i,j}}\right]}{e^{\theta g}}
= \inf_{\theta > 0} \frac{\prod_{j=1}^{\wb{q}} \prod_{i=1}^{\nhl} \expectedvalue \left[ e^{ \theta w'_{1,i,j}}\right]}{e^{\theta g}}\\
&=\inf_{\theta > 0} \frac{\prod_{j=1}^{\wb{q}} \prod_{i=1}^{\nhl} (p_{h,i} e^\theta + (1-p_{h,i}))}{e^{\theta g}}
= \inf_{\theta > 0} \frac{\prod_{j=1}^{\wb{q}} \prod_{i=1}^{\nhl} (1+p_{h,i}( e^\theta -1))}{e^{\theta g}}\\
&\leq \inf_{\theta > 0} \frac{\prod_{j=1}^{\wb{q}} \prod_{i=1}^{\nhl} e^{p_{h,i}( e^\theta -1)}}{e^{\theta g}}
\leq \inf_{\theta > 0} \frac{e^{\wb{q}( e^\theta -1)\sum_{i=1}^{\nhl} p_{h,i}}}{e^{\theta g}}
= \inf_{\theta > 0} e^{(\deff{\gamma}_{\{h,l\}}\wb{q}(e^\theta - 1) - \theta g)}
\leq \inf_{\theta > 0} e^{(\deff{\gamma}_{\{h,l\}}\wb{q}(e^\theta - 1) - \theta g)},
\end{align*}
where we use the fact that $1 + x \leq e^x$, $w'_{1,i,j} \sim \mathcal{B}(p_{h,i})$ and by Def.~\ref{def:exact-lev-scores}, $\sum_{i=1}^{\nhl} p_{h,i} = \sum_{i=1}^{\nhl} \tau_{h,i} = \deff{\gamma}_{\{h,l\}}$.
The choice of $\theta$ minimizing the previous expression is obtained as
\begin{align*}
\frac{d}{d\theta}e^{\left(\wb{q}\deff{\gamma}_{\{h,l\}}(e^\theta - 1) - \theta g \right)}
= e^{\left(\wb{q}\deff{\gamma}_{\{h,l\}}(e^\theta - 1) - \theta g \right)}\left(\wb{q}\deff{\gamma}_{\{h,l\}}e^\theta - g\right) = 0,
\end{align*}
and thus $\theta = \log (g/(\wb{q}\deff{\gamma}_{\{h,l\}}))$. Plugging this in the previous bound,
\begin{align*}
\inf_\theta \exp\left\{ \wb{q}\deff{\gamma}_{\{h,l\}}(e^\theta - 1) - \theta g)\right\}
&= \exp\left\{ g  - \wb{q}\deff{\gamma}_{\{h,l\}} -  g \log \left(\frac{g}{\wb{q}\deff{\gamma}_{\{h,l\}}}\right)\right\}\\
&= \exp\left\{ -g\left(\log \left(\frac{g}{\wb{q}\deff{\gamma}_{\{h,l\}}}\right) - 1\right)\right\} e^{- \wb{q}\deff{\gamma}_{\{h,l\}}},
\end{align*}
and choosing $g = 3\wb{q}\deff{\gamma}_{\{h,l\}}$, we conclude our proof.
 \section{Applications}\label{sec:app-generalization}
\begin{proof}[Proof of Lemma \ref{lem:nyst-app-guar}]
    \begin{align*}
    &\wt{\kermatrix}_n = \kermatrix_n \selmatrix_n(\selmatrix_n^\transp \kermatrix_n \selmatrix_n + \gamma\:I_n)^{-1}\selmatrix_n^\transp \kermatrix_n\\
    &= \featkermatrix_n^\transp\featkermatrix_n \selmatrix_n(\selmatrix_n^\transp \featkermatrix_n^\transp\featkermatrix_n \selmatrix_n + \gamma\:I_n)^{-1}\selmatrix_n^\transp \featkermatrix_n^\transp\featkermatrix_n\\
    &= \featkermatrix_n^\transp\featkermatrix_n \selmatrix_n\selmatrix_n^\transp \featkermatrix_n^\transp(\featkermatrix_n \selmatrix_n\selmatrix_n^\transp \featkermatrix_n^\transp + \gamma\:I_D)^{-1}\featkermatrix_n\\
    &= \featkermatrix_n^\transp(\featkermatrix_n \selmatrix_n\selmatrix_n^\transp \featkermatrix_n^\transp + \gamma\:I_D - \gamma\:I_D)(\featkermatrix_n \selmatrix_n\selmatrix_n^\transp \featkermatrix_n^\transp + \gamma\:I_D)^{-1}\featkermatrix_n\\
    &= \featkermatrix_n^\transp(\:I_D - \gamma(\featkermatrix_n \selmatrix_n\selmatrix_n^\transp \featkermatrix_n^\transp + \gamma\:I_D)^{-1})\featkermatrix_n\\
    &= \featkermatrix_n^\transp\:I_D\featkermatrix_n - \gamma\featkermatrix_n^\transp(\featkermatrix_n \selmatrix_n\selmatrix_n^\transp \featkermatrix_n^\transp + \gamma\:I_D)^{-1}\featkermatrix_n\\
    &= \kermatrix_n - \gamma\featkermatrix_n^\transp(\featkermatrix_n \selmatrix_n\selmatrix_n^\transp \featkermatrix_n^\transp + \gamma\:I_D)^{-1}\featkermatrix_n
    \end{align*}
    From Lem.~\ref{lem:concentration-equivalence}, and the fact that $\coldict_n$ is $\varepsilon$-approximate we have that
    \begin{align*}
    \norm{(\featkermatrix_t\featkermatrix_t^\transp + \gamma\:I_D)^{-1/2}\featkermatrix_t(\:I_t - \:S_s\:S_s^\transp)\featkermatrix_t^\transp(\featkermatrix_t\featkermatrix_t^\transp + \gamma\:I_D)^{-1/2}}{2} \leq \varepsilon,
    \end{align*}
    which implies
    \begin{align*}
    \featkermatrix_t\featkermatrix_t^\transp - \featkermatrix_t\:S_s\:S_s^\transp\featkermatrix_t^\transp \preceq \varepsilon(\featkermatrix_t\featkermatrix_t^\transp + \gamma\:I_D)
    \end{align*}
    and
    \begin{align*}
    \featkermatrix_t\featkermatrix_t^\transp - \varepsilon(\featkermatrix_t\featkermatrix_t^\transp + \gamma\:I_D) \preceq \featkermatrix_t\:S_s\:S_s^\transp\featkermatrix_t^.\transp
    \end{align*}
    Therefore,
    \begin{align*}
    \kermatrix_n - \wt{\kermatrix}_n &= \gamma\featkermatrix_n^\transp(\featkermatrix_n \selmatrix_n\selmatrix_n^\transp \featkermatrix_n^\transp + \gamma\:I_D)^{-1}\featkermatrix_n
     \preceq \gamma\featkermatrix_n^\transp(\featkermatrix_t\featkermatrix_t^\transp - \varepsilon(\featkermatrix_t\featkermatrix_t^\transp + \gamma\:I_D) + \gamma\:I_D)^{-1}\featkermatrix_n\\
     &= \frac{\gamma}{1-\varepsilon}\featkermatrix_n^\transp(\featkermatrix_t\featkermatrix_t^\transp + \gamma\:I_D)^{-1}\featkermatrix_n
     = \frac{\gamma}{1-\varepsilon}\kermatrix_t(\kermatrix_t + \gamma\:I)^{-1} \preceq \frac{\gamma}{1-\varepsilon}\:I.
    \end{align*}
\end{proof}
 
\end{document}